\documentclass{article}

\PassOptionsToPackage{numbers,compress,sort}{natbib}

\usepackage[main, final]{neurips_2026}

\usepackage[utf8]{inputenc}
\usepackage[T1]{fontenc}
\usepackage{url}
\usepackage{booktabs}
\usepackage{amsfonts}
\usepackage{nicefrac}
\usepackage{microtype}
\usepackage{xcolor}

\usepackage{amsmath}
\usepackage{amssymb}
\usepackage{amsfonts}
\usepackage{amsthm}
\usepackage{bm}

\newcommand{\R}{\mathbb{R}}

\newcommand{\V}{\mathcal{V}}
\newcommand{\Mcal}{\mathcal{M}}

\newcommand{\nips}[1]{\textcolor{black}{#1}}
\newcommand{\jonas}[1]{\textcolor{black}{#1}}
\newcommand{\dionisia}[1]{\textcolor{black}{#1}}

\theoremstyle{plain}
\newtheorem{theorem}{Theorem}[section]
\newtheorem{proposition}[theorem]{Proposition}
\newtheorem{lemma}[theorem]{Lemma}

\theoremstyle{definition}
\newtheorem{definition}[theorem]{Definition}

\theoremstyle{remark}

\usepackage{mathtools}
\usepackage{wrapfig}
\usepackage{multirow}
\usepackage[most]{tcolorbox}
\usepackage{mdframed}
\usepackage{enumitem}

\usepackage{graphicx}
\usepackage{subcaption}
\definecolor{hyboviolet}{RGB}{135, 95, 175}
\newtcolorbox{takeawaybox}{
  enhanced,
  colback=hyboviolet!10,
  colframe=hyboviolet!60,
  boxrule=0.6pt,
  arc=2pt,
  left=6pt,right=6pt,top=6pt,bottom=6pt,
  fonttitle=\normalfont\itshape, 
  title=Key Takeaway
}

\usepackage{hyperref}
\usepackage[capitalize,noabbrev]{cleveref}

\title{Hyperbolic Graph Neural Networks Under the Microscope: The Role of Geometry--Task Alignment}

\author{%
  Dionisia Naddeo\thanks{Equal contribution.} \\
  University of Rome Tor Vergata
  \And
  Jonas Linkerh\"agner\textsuperscript{*} \\
  University of Basel
  \And
  Nicola Toschi \\
  University of Rome Tor Vergata\\
  Harvard Medical School
  \And
  Geri Skenderi \thanks{Main theory contributor.} \\
  HASH \\
  \And
  Veronica Lachi \\
  UiT The Arctic University of Norway
}

\begin{document}

\maketitle

\begin{abstract}
Many complex networks exhibit hierarchical, tree-like structures, making hyperbolic space a natural candidate wherein to learn representations of them. Based on this observation, Hyperbolic Graph Neural Networks (HGNNs) have been widely adopted as a principled choice for representation learning on tree-like graphs. In this work, we question this paradigm by proposing the additional condition of geometry--task alignment, i.e., whether the metric structure of the target follows that of the input graph. We start by empirically showing that HGNNs can recover low-distortion representations when supervision explicitly rewards metric preservation and proceed by theoretically analyzing a controlled regression setting in which a linear readout benefits from representations whose geometry is aligned with the target metric. Finally, by relying on both predictive performance and embedding distortion, we further show that HGNNs gain an advantage on link prediction, a naturally geometry-aligned task, whereas this advantage largely disappears on standard node classification benchmarks. Overall, our findings shift the focus from only asking \textit{Is the graph hyperbolic?} to also questioning \textit{Is the task aligned with hyperbolic geometry?}, showing that HGNNs consistently outperform Euclidean models under such alignment, while their advantage vanishes otherwise.
\end{abstract}

\section{Introduction}
\label{sec:intro}

Graph representation learning has gathered increasing attention over the last decade, due to the importance of solving downstream tasks such as node classification or link prediction \cite{khoshraftar2024survey}. Early embedding approaches such as \textsc{DeepWalk}~\cite{perozzi2014deepwalk} and \textsc{node2vec} ~\cite{grover2016node2vec} focus on learning low-dimensional node representations that reflect graph proximity in Euclidean space by relying on random-walk co-occurrence statistics. Therefore, the core premise behind the embedding paradigm is \emph{geometric faithfulness}: a good graph embedding should preserve the relevant structure of the input graph.

This naturally raises the question of which geometry is appropriate to maintain this faithfulness. Many real-world networks exhibit strong hierarchical organization and tree-like growth patterns, including biological taxonomies~\cite{matsumoto2021novel}, knowledge graphs~\cite{kolyvakis2019hyperkg}, and social hierarchies~\cite{yang2023hyperbolic}.
From a metric viewpoint, tree-like graphs are closely connected to negatively curved spaces.
Crucially, \citet{sarkar2011low} provides a guarantee: for a fixed low dimensionality, trees can be embedded in hyperbolic space with arbitrarily small distortion, while this cannot be done in Euclidean space. Motivated by this theory, hyperbolic representation learning has emerged as a powerful tool to encode latent hierarchies \cite{nickel2018learning, kitsak2020link, wang2019hyperbolic,chami2020low,pan2021hyperbolic}. In particular, \citet{nickel2017poincare} introduced Poincar\'e embeddings, demonstrating that hyperbolic geometry yields low-distortion representations of hierarchical symbolic data and relational structures.

\nips{With the advent of Graph Neural Networks (GNNs) \cite{scarselli2008graph,kipf2017semi,hamilton2017graphsage,velickovic2018gat}, graph representation learning shifted to end-to-end task-driven optimization via message passing. Building on the success of hyperbolic embeddings, Hyperbolic GNNs (HGNNs) extend this paradigm to hyperbolic spaces \cite{liu2019hyperbolic,chami2019hyperbolic,zhang2021lorentzian,chen2022fully}.
These works often motivate the use of hyperbolic architectures by linking their expected benefits to the hyperbolicity of the input graph, commonly measured through Gromov's $\delta$-hyperbolicity \cite{gromov1987hyperbolic} .
For instance, \citet{chami2019hyperbolic} explicitly conjecture that "HGCN works better on graphs with small $\delta$-hyperbolicity", a line of reasoning that has been adopted, explicitly or implicitly, in subsequent work~\citep{zhang2021lorentzian,yang2024hypformer,yang2022hyperbolic}. This has fostered the belief that HGNNs should be preferred whenever the input graph is highly hyperbolic.}

 \begin{wrapfigure}{r}{0.55\columnwidth}
 \vspace{-12pt}
   \centering
   \includegraphics[width=\linewidth]{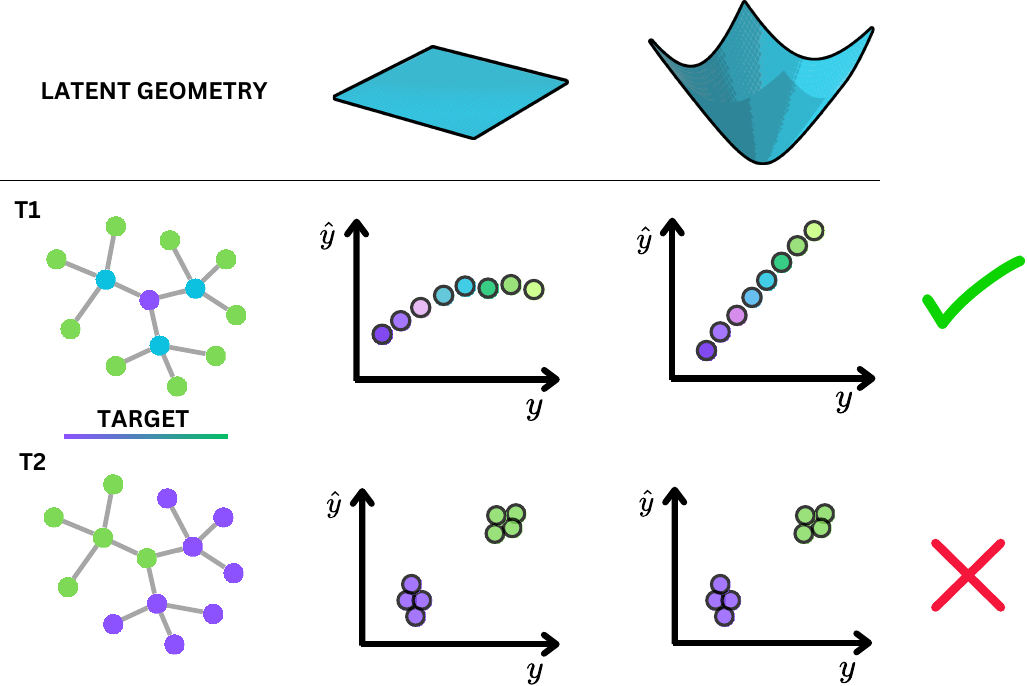}
   \caption{\nips{\textbf{Geometry--Task Alignment.}
 Different tasks induce different levels of alignment between the latent geometries and the predicted outcomes. The \jonas{columns indicate} latent spaces with increasing curvature magnitude. \jonas{The rows plot labels $y$ against predictions $\hat{y}$ for targets \textbf{T1}\&\textbf{T2}}. Increasing curvature improves performance for \textbf{T1}, whose targets align with the underlying input geometry, but not \textbf{T2}, despite identical inputs.}}
   \label{fig:task-geom}
   \vspace{-8pt}
 \end{wrapfigure}
Despite their popularity in applications like brain network analysis, recommender systems, and knowledge graphs \cite{baker2024hyperbolic, yang2023hyperbolic,li2024hyperbolic}, the benefits of HGNNs remain contested. Recent work shows that their empirical advantages are difficult to assess on standard benchmarks, where tasks can often be solved from node features alone \cite{katsman2025shedding}. \citet{naddeo2025we} further demonstrate that, on graph-level tasks, HGNNs do not consistently outperform Euclidean models, even on strongly hyperbolic graphs. These findings highlight a conceptual gap in the current narrative (see Appendix~\ref{app:relwork}). 

The original motivation for hyperbolic \emph{embeddings} was fundamentally to minimize distortion \citep{chami2019hyperbolic}. In contrast, representations learned via supervision are task-dependent and do not necessarily optimize for metric faithfulness. Consequently, even if hyperbolic space is the right geometry for \emph{the input}, it does not follow that hyperbolic message passing is beneficial for \emph{the task at hand}. Specifically, while trees admit near-isometric hyperbolic embeddings \cite{sarkar2011low}, it has never been established that hyperbolic message passing under supervision actually recovers low-distortion node geometries, nor that such low distortion is systematically useful across different graph learning tasks.

In this work, we revisit the question of when \emph{message-passing HGNNs} are \emph{truly} warranted through the lens of \emph{geometric alignment between input structure and task supervision} (Fig.~\ref{fig:task-geom}).
We argue that the "tree-likeness" of the input graph is not, by itself, a sufficient criterion, but we must consider an additional aspect, i.e., whether solving the downstream task requires to reflect (parts of) the hyperbolic input geometry. Indeed, even a maximally hyperbolic input, such as a tree, does not guarantee an advantage for HGNNs: the same graph can favor different representation geometries depending on the supervision signal. Our analysis therefore distinguishes the ability to learn metric-faithful embeddings from their usefulness for downstream prediction. 
Through controlled experiments on synthetic graphs, evaluations on real-world datasets, and theoretical analyses, we disentangle the benefits of learning hyperbolic representations in GNNs and when they translate into measurable performance gains. 

In summary, our main contributions are: 
\begin{itemize}[leftmargin=25pt]
    \item[(i)] We empirically verify that supervised HGNNs can recover low-distortion representations much better than their Euclidean counterparts when the task explicitly rewards metric preservation. Motivated by established theoretical results, we verify that this gap can be reduced as hidden dimensionality increases;
    \item[(ii)] We introduce geometry--task alignment as a lens for understanding when hyperbolic inductive biases help. By formalizing this principle in a controlled node-regression setting and conducting a comprehensive re-evaluation of HGNNs under a fair benchmarking protocol, we show that HGNN performance is governed not only by graph hyperbolicity, but also by the alignment between supervision and input geometry;
    
    \item[(iii)] We derive actionable guidance for practitioners: HGNNs are well suited for link prediction on hyperbolic graphs, but they do not typically provide advantages for node classification. For the latter, we recommend careful benchmarking against Euclidean counterparts. Initial results indicate that HGNNs can benefit node regression when the target is geometry-aligned, albeit the lack of real-world hyperbolic benchmarks for this task remains an open research gap.
\end{itemize}

The code repository is available on \href{https://github.com/jlinki/Hyperbolic-Graph-Neural-Networks-Under-the-Microscope.git}{GitHub}.

\section{Background}
In this section, we introduce the key concepts from
hyperbolic geometry and hyperbolic graph learning. Additional details can be found in Appendices~\ref{app:relwork} and~\ref{app:extended background}.

\paragraph{Riemannian Manifolds and Hyperbolic Space.}
A $d$-dimensional smooth manifold $\Mcal$ is a topological space that is locally homeomorphic to $\mathbb{R}^d$, such that all transition maps between the local coordinates are differentiable. For $\bm{x}\in \Mcal$, the tangent space at $\bm{x}$, $T_{\bm{x}}\Mcal \subseteq \mathbb{R}^d$, is the space of spanned by velocity vectors of all curves passing through $\bm{x}$. A Riemannian manifold $(\mathcal{M}, g)$ is a smooth manifold equipped with a metric 
$g_x : T_x\mathcal{M} \times T_x\mathcal{M} \to \mathbb{R}$ at each point $x \in \mathcal{M}$,
where $T_x\mathcal{M}$ denotes the tangent space at $x$. The metric $g$ defines a (local) inner product thereby leading to notions of
length and curvature, enabling the extension of Euclidean geometry to non-flat spaces.
We consider hyperbolic spaces modeled as Riemannian manifolds with constant negative sectional curvature. This property leads to exponential volume growth (Appendix~\ref{appendix:geometry}) and makes it suitable for modeling tree-like data, where the number of nodes grows exponentially with depth. 

\paragraph{Models of Hyperbolic Space.}
\nips{Hyperbolic space admits several equivalent representations, the most common being the Poincar\'e ball and the Lorentz (hyperboloid) model, which are isometric to one another. 
The Poincar\'e ball $\mathcal{B}^d_c$ models hyperbolic space as an open ball in $\R^d$:
\begin{equation}
\mathcal{B}^d_c = \left\{ \bm{x}\in\R^d \;:\; \lVert \bm{x}\rVert^2 < 1/c \right\},
\end{equation}
equipped with a conformal metric with factor $\lambda_c(\bm{x}) = \frac{2}{1 - c\|\bm{x}\|^2}$. Although points lie inside a bounded Euclidean region, hyperbolic distances grow rapidly near the boundary.
The Lorentz model $\mathcal{L}^d_c$ represents hyperbolic space as a $d$-dimensional hyperboloid embedded in $\R^{d+1}$:
\begin{equation}
\mathcal{L}^d_c = \left\{ \bm{x}\in\R^{d+1} \;:\; \langle \bm{x},\bm{x}\rangle_{\mathcal{L}} = -1/c,\; x_0>0 \right\},
\end{equation}
where $\langle \bm{x},\bm{y}\rangle_{\mathcal{L}} = -x_0y_0 + \sum_{i=1}^{d} x_i y_i$, and the metric is induced from the ambient Minkowski space.}

\paragraph{Geodesics, exponential map, and logarithmic map.} In curved spaces, the notion of straight line is replaced by \emph{geodesics}, i.e., locally shortest curves under the manifold metric. The \emph{exponential map} $\exp_{\bm{x}}:T_{\bm{x}}\Mcal\to \Mcal$ maps a tangent vector to the manifold by following the geodesic starting at $\bm{x}$ with that initial velocity. Conversely, the \emph{logarithmic map} $\log_{\bm{x}}:\Mcal\to T_{\bm{x}}\Mcal$ locally inverts the above map and sends a point on the manifold back to a tangent vector. These operations enable neural architectures to apply Euclidean computations by mapping back and forth into the manifold \cite{ganea2018hyperbolic, chami2019hyperbolic}.

\paragraph{HGNNs.}
HGNNs learn node representations in a negatively curved space and require aggregation and update operators that are consistent with hyperbolic geometry. Existing approaches can be grouped into two design families. \emph{Tangent-space HGNNs}, such as HGCN~\cite{chami2019hyperbolic}, perform message passing in a Euclidean tangent space (typically at the origin) and then map the updated features back to the manifold via logarithmic and exponential maps. In contrast, \emph{fully hyperbolic GNNs}, such as HyboNet~\cite{chen2022fully}, operate directly in hyperbolic space by aggregating neighborhoods with geometry-aware operators (Appendix~\ref{app:hyperbolic_graph_learning}). 
In our experiments, we use HGCN in the Poincaré ball model with learnable curvature, HyboNet in the Lorentz model, and GCN \cite{kipf2017semi} and GAT \cite{velickovic2018gat} as Euclidean baselines. Note that HyboNet is designed to work only on the Lorentz hyperboloid.Transformer-based approaches, such as SGFormer \citep{wu2023sgformer} and HypFormer~\citep{yang2024hypformer}, fall outside the scope of this work, which focuses on message-passing GNNs.

\section{Can HGNNs learn low-distortion embeddings of trees?}
\label{sec:pairs_distortion_pdp}
HGNNs typically learn representations through supervision rather than by directly optimizing a metric embedding objective. 
This raises a basic question: \emph{can a HGNN actually learn low-distortion representations under supervision?} To answer it, we first formally define distortion, and then relate it to a controlled synthetic task which we can learn to solve, thereby directly answering the question. It will be useful for our subsequent analyses to view a GNN as a composition of functions $f_{\theta} = g_{\psi} \circ h_{\phi}$, where $h_{\phi}$ is an \emph{embedding} function and $g_{\psi}$ is a task-dependent \emph{solver}.
Concretely, $h_{\phi}: \mathcal{G} \to \mathcal{M} \subseteq \mathbb{R}^d$ maps nodes (and their graph context) into a (Riemannian) representation manifold $\mathcal{M}$, while the solver $g_{\psi} : \mathcal{M} \to \mathcal{Y}$ maps latent representations to an output space $\mathcal{Y}$ determined by the task. For example, $\mathcal{Y}=\mathbb{R}$ in node regression or $\mathcal{Y}=\{1,\ldots, C\}$ in node classification. For simplicity, we will drop the weight subscripts when referring to the model functions. The distortion of the embedding map $h$ is meant to capture the metric faithfulness of an embedding by considering it as a map between metric spaces. Let the distance measure on the input graph $\mathcal{G}=(\mathcal{V},\mathcal{E})$ be the shortest-path distance $d_{\mathcal{G}}(\cdot,\cdot)$ and $d_{\mathcal{M}}(\cdot,\cdot)$ be geodesic distance on the manifold $\mathcal{M}$.

\begin{definition}[Embedding distortion \cite{sarkar2011low}] 
\label{def:distortion}
Let the embedding $h: (\mathcal{G}, d_{\mathcal{G}}) \to (\mathcal{M}, d_{\mathcal{M}})$ be a map between metric spaces that is injective onto its image. Define the contraction and expansion factors:
\begin{equation}\label{eq:dist_factors}
    \delta_c(h) = \max_{\substack{i,j \in \mathcal{V} \\  i \neq j}} \frac{d_{\mathcal{G}}(i,j)}{d_{\mathcal{M}}(h(\bm{x}_i),h(\bm{x}_j))} \quad \text{and} \ \ \delta_e(h) = \max_{\substack{i,j \in \mathcal{V} \\  i \neq j}} \frac{d_{\mathcal{M}}(h(\bm{x}_i),h(\bm{x}_j))}{d_{\mathcal{G}}(i,j)}.
\end{equation}
The distortion of $h$ is defined as $\delta(h) = \delta_c(h)\delta_e(h)$.
\end{definition}

It is immediate to see that $\delta(h)\ge 1$. Moreover, we have that for all $i\neq j: \:
\delta_c(h)^{-1} d_{\mathcal{G}}(i,j) \le d_{\mathcal{M}}(h(\bm{x}_i),h(\bm{x}_j)) \le \delta_e(h)\, d_{\mathcal{G}}(i,j),$
so $h$ is \emph{bi-Lipschitz onto its image} with constants given by its distortion. When $\delta(h)=1$ we can write
$\delta_c(h)=\delta_e(h)^{-1}$ and therefore $d_{\mathcal{M}}(h(\bm{x}_i),h(\bm{x}_j)) = \tau\, d_{\mathcal{G}}(i,j) \quad \text{for some }\tau>0$. An \emph{embedding without distortion is thereby a global similarity}, i.e., it globally scales pairwise distances in $\mathcal{G}$ by a factor $\tau$. The stricter case of a (global) isometry is obtained by requiring $\delta_e(h)=\delta_c(h)=1$. We are now ready to ask our first key question regarding hyperbolic message passing over trees: \emph{even if low-distortion hyperbolic embeddings exist \cite{sarkar2011low}, can a HGNN actually learn them under supervision?}

\subsection{Pairwise Distance Prediction: training explicitly for metric preservation}
\label{sec:pdp_task}

To address this question, we introduce a controlled synthetic task named the Pairwise Distance Prediction (PDP) task.
Given a finite, connected graph $\mathcal{G}=(\mathcal{V},\mathcal{E})$ we consider all node pairs $\mathcal{P}\subseteq \{(i,j)\in\V\times\V: i\neq j\}$, divide them in training, validation and test sets and define regression targets $D_{ij} := d_{\mathcal{G}}(i,j), \: (i,j)\in \mathcal{P}$. We train the encoder $h_\phi$ to map nodes to a latent manifold $\bm{z}_i = h_\phi(\bm{x}_i)\in\mathcal{M}$, where $\bm{x}_i$ are the features of node $i$ and $\mathcal{M}$ is either $\R^d$ or $\mathcal{B}^d_{c}/\mathcal{L}^d_{c}$. For each pair we compute the corresponding latent distance $\hat{d}_{ij} = d_{\mathcal{M}}(\bm{z}_i,\bm{z}_j)$ and use it to predict $D_{ij}$ through a scalar decoder $\hat{D}_{ij} = g_\psi(\hat{d}_{ij}) = a\, \hat{d}_{ij} + b, \: (i,j)\in \mathcal{P}$, where $\psi=(a,b)$. This isolates the role of the \emph{geometry of the embedding space}: the model succeeds only if shortest-path distances can be expressed up to affine rescaling as geodesic distances in $\mathcal{M}$.

\paragraph{Stress loss.}
We optimize the parameters $\phi, \psi$ by minimizing a distortion-aware objective inspired by metric multidimensional scaling \cite{canzar2024metric}:
\begin{equation}
\mathcal{L}_{\mathrm{stress}}(\phi,\psi)
=
\frac{1}{|\mathcal{P}|}
\sum_{(i,j)\in\mathcal{P}}
\frac{\bigl(\hat{D}_{ij} - D_{ij}\bigr)^{2}}
     {D_{ij}^{2} + \varepsilon},
\end{equation}
where $\varepsilon>0$ is a small constant for numerical stability. 
The normalization by $D_{ij}^2$ prevents optimization from being dominated by large-distance pairs. Importantly, this loss provides a straightforward empirical proxy to operationalize distance preservation by measuring whether graph distances are easily recoverable from latent distances. If this loss is further normalized by the stress obtained from a feature-only baseline that calculates distances directly in the input feature space, we refer to it as \emph{Normalized Stress Loss}.

\paragraph{Datasets.}

To evaluate the PDP task, we construct two synthetic graphs with comparable size but different intrinsic geometry: a \emph{tree} and a \emph{grid}.
These are canonical settings where hyperbolic and Euclidean latent spaces are respectively well matched.
The \emph{tree} is a balanced $b$-ary tree with branching factor $b=5$ and depth $l=4$ ($|\V|=781$), whereas the \emph{grid} is a $28\times 28$ $2$D lattice ($|\V|=784$).
Node features are sparse ($10\%$) Gaussian vectors in $\R^{100}$.

\begin{figure}[t]
    \centering
    \begin{subfigure}{0.38\linewidth}
        \centering
        \includegraphics[width=\linewidth]{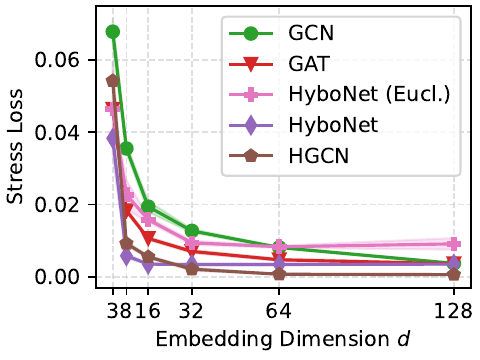}
        \subcaption{Tree}
        \label{fig:stress_loss_tree}
    \end{subfigure}
    \begin{subfigure}{0.38\linewidth}
        \centering
        \includegraphics[width=\linewidth]{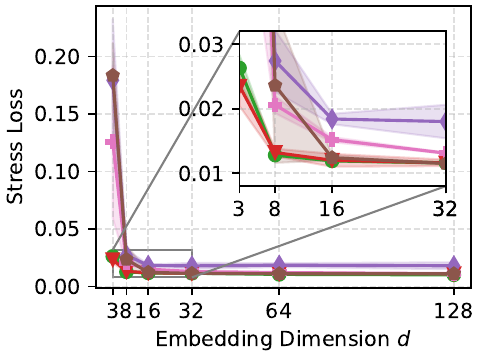}
        \subcaption{Grid}
        \label{fig:stress_loss_grid}
    \end{subfigure}
    \caption{Test Stress Loss ($\downarrow$) vs. embedding dimension for the PDP task. Markers show mean over $10$ seeds; shaded regions indicate standard deviation.}
    \label{fig:stress_vs_embdim}
    \vspace{-1em}
\end{figure}

\paragraph{Results.}

Averaged results over 10 random seeds on the \textit{tree} and \textit{grid} graphs are reported in Fig.~\ref{fig:stress_vs_embdim}, with detailed values provided in Appendix~\ref{appendix:pairs_task}. \nips{For details on tuning and evaluation, see Appendix~\ref{app:tuning}.} Inspired by Nash's embedding theorems \cite{nash1956imbedding}, we vary the embedding dimension $d$ to assess whether Euclidean models can increasingly approximate a low-distortion embedding of the input graph. To isolate the role of geometry, we further compare HyboNet with a Euclidean counterpart that shares the same message-passing architecture and model capacity, differing only in the embedding space. On the \textit{tree}, hyperbolic models achieve substantially lower stress than Euclidean baselines at low dimensions, confirming that negative curvature provides an effective inductive bias. HyboNet performs best at $d=3$, while HGCN becomes competitive at higher dimensions. As $d$ increases, Euclidean models progressively close the gap, and by $d=128$ all methods achieve comparable stress. On the \textit{grid}, the trend reverses: Euclidean models consistently achieve lower stress across all dimensions, indicating that flat geometry better matches grid distances. Notably, HGCN outperforms HyboNet in this setting, suggesting that its learnable curvature allows it to adapt toward an effectively Euclidean representation regime. \dionisia{This interpretation is supported by the learned final-layer curvature magnitudes, which are close to zero on the grid but remain substantially nonzero on the tree (Appendix~\ref{appendix:pairs_task}, Table~\ref{tab:pdp_learned_curvature}).}

We complement the stress-based evaluation with a qualitative analysis that directly compares pairwise distances in data and embedding space. Specifically, we plot average latent distances conditioned on true graph distances, fit a linear model, and report the resulting $R^2$ score. \nips{Full results are provided in Appendix~\ref{appendix:pairs_task}, including a target-permutation control where labels are randomly shuffled across nodes, preserving their marginal distribution while breaking their relation to graph distances.}

\dionisia{To test whether the observed advantage depends on geometry--task alignment, we repeat the PDP experiments after randomly permuting target distances across node pairs \jonas{\cite{zhang2017understanding}}. This preserves their marginal distribution while breaking their correspondence
with the graph metric. On the tree, all models obtain similar test stress, approximately $0.15$--$0.16$, at both $d=3$ and $d=128$. On the grid, test stress is also broadly similar across models, ranging from approximately $0.53$ to $0.58$. Increasing dimension does not consistently improve performance under permuted supervision. In particular, the low-dimensional
hyperbolic advantage on the tree disappears despite its unchanged $0$-hyperbolicity, supporting the interpretation that the benefit depends on geometry-aligned supervision rather than input hyperbolicity alone. Full results are reported in Appendix~\ref{appendix:pairs_task}, Table~\ref{table:null_tree_grid}.}

\dionisia{Additional PDP experiments (Appendix~\ref{app:pdp_tree_structure}) at fixed hidden dimension $d=8$ vary tree depth ($\ell=2$--$6$, with $b=3$) and branching factor ($b=2$--$6$, with $\ell=4$). A hyperbolic model achieves the lowest mean test stress in every configuration, showing that the main findings generalize across balanced trees with different depths and branching factors.}

\begin{takeawaybox}
When supervision demands metric faithfulness, message passing can exploit the geometric inductive bias of the latent space. Hyperbolic architectures can minimize distortion more effectively on hyperbolic inputs, especially in lower dimensions where model capacity is limited.
\end{takeawaybox}

\section{From Metric Preservation to Task-Aligned Geometry}
\label{sec:rq2_task_geometry}

Section~\ref{sec:pairs_distortion_pdp} established an important result: \emph{when supervision explicitly demands metric faithfulness}, hyperbolic message passing can recover optimal representations of tree-like graphs in terms of distortion. This dependence on supervision is clear from the control experiment in Appendix~\ref{appendix:pairs_task}: keeping a \emph{tree} as input but randomly shuffling the targets breaks the link between labels and geometry, and hyperbolic models no longer outperform Euclidean ones. This leads us to the key claim of our study: \emph{contrary to popular belief, input hyperbolicity alone is not sufficient to justify HGNNs}. Instead, a hyperbolic inductive bias is beneficial precisely when the task at hand is aligned in some way with the geometric structure of the input. This raises the central question of this section: \emph{when does a task require preserving (parts of) the input geometry?} To provide a first step toward answering this question, we turn to a simple formal model of geometry--task alignment. Consider the problem of univariate \emph{node regression} (NR), where the label space has metric function $d_{\mathcal{Y}}(y_i,y_j)=|y_i-y_j|$. The goal is to formalize the intuition that the labels inherit the input geometry if they share similar metric characteristics to the input graph. More specifically, we propose the following characterizations, in decreasing order of strength: 

Suppose that there exist constants $\alpha,\beta>0$ such that for all $i,j \in \mathcal{V}, \:i\neq j: \:$ \emph{(i) [Global]} $\alpha\, d_{\mathcal G}(i,j)\ \leq d_{\mathcal Y}(y_i, y_j) \ \leq \beta\, d_{\mathcal G}(i,j); $ \emph{(ii) [Monotone Global]} there exists a strictly increasing function $\omega:\mathbb R\to\mathbb R$ for which $ \alpha\, d_{\mathcal G}(i,j)\ \leq d_{\mathcal Y}(\omega(y_i), \omega(y_j)) \leq \beta\, d_{\mathcal G}(i,j); $ \emph{(iii) [Local]} nodes in an $r$-ball $d_{\mathcal{G}}(i,j) \leq r$ respect $ \alpha\, d_{\mathcal G}(i,j)\ \leq d_{\mathcal Y}(y_i, y_j) \leq \beta\, d_{\mathcal G}(i,j)\leq \beta\,r .$
 
Consider now a supervised learning problem with $n_t$ training labels $\mathbf{y} = \{y_i\}^{n_t}_{i=1} \in \mathbb{R}^{{n_t}}$ and features $\mathbf{X} = \{\mathbf{x}_i\}^{n}_{i=1} \in \mathbb{R}^{{n} \times d'}$, which live upon a simple, finite, and connected graph $\mathcal{G}=(\mathcal{V}, \mathcal{E}), \: |\mathcal{V}| = n$. 
The problem consists in learning a model $f_{\theta}(\mathbf{x_i}, \mathcal{G}) = \hat{y}_i \in \mathbb{R}$ such that the risk $\frac{1}{2}(y_i - \hat{y_i})^2$ is minimal over training samples. Unless noted otherwise, the representation manifolds are assumed to be Riemannian submanifolds of a Euclidean ambient space s.t. $\langle \cdot, \cdot \rangle$ and $\|\cdot\|$ correspond to the dot-product and $l^2$ norm, respectively. The proofs of all statements are deferred to Appendix \ref{app:proofs}. 

\begin{theorem} 
\label{prop:model_lipschitz}
Let $\mathcal{G} \coloneqq (\mathcal{V}, \mathcal{E})$ be the input graph, $f = g \circ h$ be the model, with $h$ an embedding as in Def. \ref{def:distortion}, and $g(\bm{z}_i = h(\bm{x}_i)| \bm{w}, b) : \mathcal{M} \to \mathbb{R} \coloneqq \langle \bm{w}, \bm{z}_i \rangle + b = \hat{y}_i$ a linear solver. Define:
\begin{gather}
    \xi(h) = \min_{\substack{i,j \in \mathcal{V} \\  i \neq j}} \frac{\|\mathbf{z}_i - \mathbf{z}_j \|}{d_{\mathcal{M}}(\mathbf{z}_i, \mathbf{z}_j)}, \quad\:
    \zeta(h) = \max_{\substack{i,j \in \mathcal{V} \\  i \neq j}} \frac{\|\mathbf{z}_i - \mathbf{z}_j \|}{d_{\mathcal{M}}(\mathbf{z}_i, \mathbf{z}_j)}, \quad\:
    \rho(h, \bm{w}) = \min_{\substack{i,j \in \mathcal{V} \\  i \neq j}} \frac{|\langle w, \mathbf{z}_i - \mathbf{z}_j\rangle|}{\|\bm{w}\| \|\mathbf{z}_i - \mathbf{z}_j\|}.
\end{gather}
Then, if $\rho(h, \bm{w}) > 0$, $f$ is bi-Lipschitz onto its image with upper constant $\beta = \|\bm{w}\| \zeta(h)\delta_e(h)$ and lower constant $\alpha = \frac{ \|\bm{w}\| \xi(h) \rho(h,\bm{w})}{\delta_c(h)}$, such that $\alpha d_{\mathcal{G}}(i,j) \leq |\hat{y_i} - \hat{y_j}| \leq \beta d_{\mathcal{G}}(i,j)$.
\end{theorem}
The above result should not be simply interpreted as saying that that a model class consisting of a linear solver is able to effectively learn geometry aligned labels controlled by distortion. The relevant
question is whether the resulting geometry agrees with the ground-truth labels. Indeed, the theorem controls distances between predictions, therefore if the prediction error is uniformly bounded as \(|y_i-\hat y_i|\leq \varepsilon\), then for all \(i\neq j\) we trivially have that $\alpha d_{\mathcal G}(i,j)-2\varepsilon \leq |y_i-y_j| \leq \beta d_{\mathcal G}(i,j)+2\varepsilon .$
Thus, \emph{a low-error solution whose embedding has low distortion implies that the labels are approximately aligned with the input geometry}.
Conversely, if the labels violate such metric relations, as in the
previous permutation experiments, preserving graph distances is not meaningful for the supervised task and the model must either incur
prediction error or learn a distorted representation. By analyzing this simple regression setting with a linear readout we can observe a concrete case of alignment, since the label space carries a well-defined distance function. On the other hand, Thm \ref{prop:model_lipschitz} does not directly apply to Lorentz-model architectures such as HyboNet, whose ambient space is Minkowski rather than Euclidean. We therefore will simply treat HyboNet as an empirical extension for detailed comparison. Having clarified the presence of geometry--task alignment, we now turn our attention to model behavior using another concept of distortion:
\begin{figure}[t]
  \centering
  \captionsetup[subfigure]{aboveskip=2pt,belowskip=3pt}

  \begin{subfigure}{0.32\columnwidth}
      \centering
      \includegraphics[width=\linewidth]{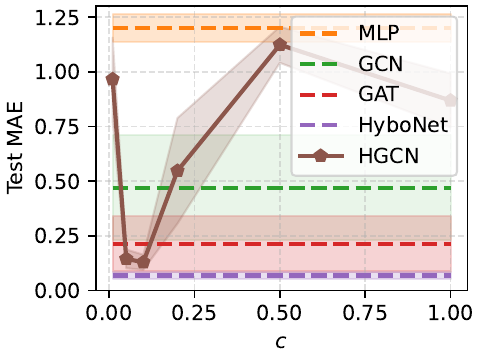}
      \caption{HGCN curvature }
      \label{fig:hgcn_nr_curvature128}
  \end{subfigure}
  \hfill
  \begin{subfigure}{0.32\columnwidth}
      \centering
      \includegraphics[width=\linewidth]{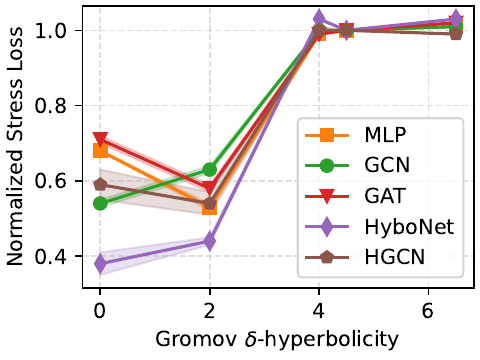}
      \caption{LP stress}
      \label{fig:lp_real_stress_loss_norm}
  \end{subfigure}
  \hfill
  \begin{subfigure}{0.32\columnwidth}
      \centering
      \includegraphics[width=\linewidth]{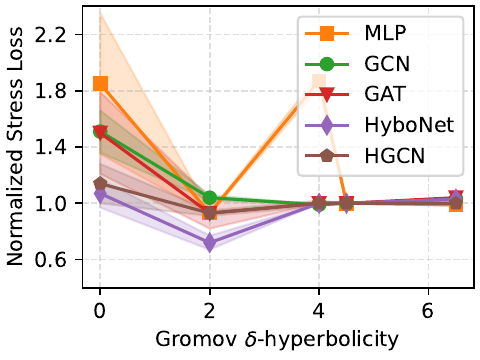}
      \caption{NC stress}
      \label{fig:nc_real_stress_loss_norm}
  \end{subfigure}

  \caption{
  Left: NR MAE ($\downarrow$) of HGCN with fixed curvature $c$ on the tree for $d=128$, showing a curvature--performance trade-off.
  Middle/right: Normalized Stress loss ($\downarrow$) for predicting shortest-path distances using embeddings learned for LP and NC.
  Hyperbolic models help for LP at small $\delta$, while gains for NC are marginal.
  }
  \label{fig:curvature_and_real_stress}
  \vspace{-0.2em}
\end{figure}
\begin{definition}[Model distortion]
\label{def:model_distortion}
Let $\mathcal{G}$ be the input graph, $f$ be the model, and $\beta, \alpha$ the upper and lower constants following Thm. \ref{prop:model_lipschitz}. The model distortion is defined as the condition number $k(f) = \frac{\beta}{\alpha} = \frac{\zeta(h) \delta(h)}{\xi(h) \rho(h,\bm{w})}$.
\end{definition}
Intuitively, model distortion is a task-level analogue of embedding distortion: it measures how the graph metric is transformed by the complete learning model. We can then derive the following results:

\begin{proposition}[Euclidean embeddings]
\label{prop:Euclidean_model_distortion}
Let $\mathcal{G}$ be the input graph, $f$ be the model, $h$ the embedding function, and $g$ the solver function, all as defined in Thm. \ref{prop:model_lipschitz}. Furthermore, let $h$ be an embedding of the graph into Euclidean space $\mathbb{R}^d$. The model distortion of $f$ is $k(f) = \frac{\delta(h)}{\rho(h,\bm{w})}.$
\end{proposition}
In the most practically employed case of learning Euclidean embeddings, the model distortion and thus the geometry--task relationship is directly proportional to the distortion of the embedding function. Assuming non-orthogonality between any chord in representation space and the weight vector of $g$, the model distortion is in fact no less than the embedding distortion, since $0 < \rho(h,\bm{w}) \leq 1$, therefore the alignment tightness is directly controlled by $\delta(h)$. This fact is precisely the intuition behind most papers that propose the use of HGNNs for tree-like graphs \cite{chami2019hyperbolic}. The following propositions show that despite being intuitive, this motivation does not necessarily hold in hyperbolic space.

\begin{proposition}[Similarities and isometries]
\label{prop:Isometries_distortion}
Let $\mathcal{G}$ be the input graph, $f$ be the model, $h$ the embedding function, and $g$ the solver function, all as defined in Thm.~\ref{prop:model_lipschitz}. If $\delta(h)=1$, the model distortion of $f$ is $ k(f) = \frac{\zeta(h)}{\xi(h)\rho(h,\bm{w})}.$
\end{proposition}

\begin{proposition}[Poincaré ball embeddings] 
\label{prop:poincare_model_distortion}
Let $\mathcal{G}$ be the input graph, $f$ be the model, $h$ the embedding function, and $g$ the solver function, all as defined in Thm. \ref{prop:model_lipschitz}. Let $h$ be an embedding function that embeds into the Poincaré ball $\mathcal{B}^d_{1}$ and define $\eta := \max_{i \in \mathcal{V}} \|h(\mathbf{x}_i)\|$. Then, the model distortion of $f$ is upper bounded by $k(f) \leq \frac{\delta(h)}{(1-\eta^2)\rho(h,\bm{w})}.$
\end{proposition}

Prop.~\ref{prop:poincare_model_distortion} applies specifically to solvers that operate on ambient Poincaré coordinates. It captures a boundary-compression effect where as embeddings approach the boundary, hyperbolic distances grow much faster than the Euclidean coordinates seen by the solver. Practical tangent-space readout usually use $\log_0$ as the coordinate map, therefore the above result should not be interpreted as a direct characterization of the full readout. Nevertheless, it serves in highlighting the general principle that the choice of solver coordinates affects model distortion and alignment, ultimately revealing the notion of geometry-task alignment. Figure~\ref{fig:hgcn_nr_curvature128} empirically illustrates a related curvature--performance trade-off for HGCN. Note that similar phenomena have been observed empirically for hyperbolic generative models \cite{wen2024hyperbolic}, where typically the curvature of the latent space is set close to $0$. 

\subsection{Experiments on Node Regression}
\label{sec:node_regression}
To approximate a setting where we can verify our theoretical results, we propose a simple task. Extract an anchor node $a \in \mathcal{V}$ and define a label for each node $i \in \mathcal{V}  \setminus \{a\}$ as $y_i = \beta^*d_\mathcal{G}(i, a)$, for $\beta^* > 0$. Using the reverse triangle inequality and excluding $a$ from the labeled set, we obtain a problem where $0 \leq |y_i - y_j| = \beta^* \lvert d_\mathcal{G}(i, a) - d_\mathcal{G}(a, j) \rvert\leq \beta^* d_\mathcal{G}(i,j),$ which provides a local geometry-alignment probe: it tests whether the model can recover distance-to-anchor structure. Concretely, we use a tree with depth $l=4$ and branching factors $b=(100,2,2,2)$ with sparse ($10\%$) Gaussian features in $\R^{100}$. The anchor is the root node, therefore the regression error reflects geometric alignment by testing whether the model captures the exponential expansion of the underlying space. As a sanity check, we also report results for a target-permutation baseline under the control column. To assess alignment on a non-hyperbolic structure, we consider the same grid graph as in Section~\ref{sec:pdp_task}, using a central node as anchor. We evaluate the same models as in Section~\ref{sec:pairs_distortion_pdp}, along with an MLP. Models are tested in two configurations: a constrained setting with dimension $3$, and the best configuration among $\{16, 32, 64, 128\}$ (denoted by $\leq 128$). Performance is measured using Mean Average Error (MAE) over $10$ splits.
\begin{table}[t]
\caption{\textbf{NR} MAE ($\downarrow$) on the synthetic tree and grid dataset for dimension $3$ and $\leq128$, and control on the tree.}
\centering
\begin{sc}
\resizebox{0.7\linewidth}{!}{%
\begin{tabular}{l c c c c c}
\hline
& \multicolumn{3}{c}{\textbf{Tree}} & \multicolumn{2}{c}{\textbf{Grid}} \\
\cmidrule(lr){2-4} \cmidrule(lr){5-6}
\textbf{Model} 
& \textbf{3} 
& \textbf{$\leq$ 128} 
& \textbf{Control} 
& \textbf{3} 
& \textbf{$\leq$128} \\
\hline

MLP     
& 1.24$\pm$\scriptsize0.06 
& 1.20$\pm$\scriptsize0.06 
& 1.20$\pm$\scriptsize0.07
& 8.36$\pm$\scriptsize1.67
& 9.40$\pm$\scriptsize2.67 \\

GCN     
& 0.51$\pm$\scriptsize0.24 
& 0.47$\pm$\scriptsize0.24 
& 1.19$\pm$\scriptsize0.07
& 0.45$\pm$\scriptsize0.10
& 0.46$\pm$\scriptsize0.04 \\

GAT     
& 0.63$\pm$\scriptsize0.39 
& 0.21$\pm$\scriptsize0.13 
& 1.21$\pm$\scriptsize0.05
& \textbf{0.41$\pm$\scriptsize0.14}
& 0.29$\pm$\scriptsize0.06 \\

HyboNet 
& \textbf{0.17$\pm$\scriptsize0.02} 
& \textbf{0.07$\pm$\scriptsize0.02}
& 1.23$\pm$\scriptsize0.06
& 1.85$\pm$\scriptsize0.55
& 0.45$\pm$\scriptsize0.08 \\

HGCN    
& 0.25$\pm$\scriptsize0.05 
& 0.08$\pm$\scriptsize0.02
& \textbf{1.16$\pm$\scriptsize0.06}
& 0.44$\pm$\scriptsize0.07
& \textbf{0.27$\pm$\scriptsize0.06} \\

\hline
\end{tabular}}
\end{sc}
\label{tab:node_regression_mae}
\end{table}

\paragraph{Results.}
\nips{The results in Table~\ref{tab:node_regression_mae} show that all GNNs outperform the MLP on both datasets, confirming the importance of graph structure. On the tree, the fully hyperbolic HyboNet achieves the best performance, likely due to operating entirely in hyperbolic space. HGCN performs slightly worse, which we attribute to a learned curvature parameter that puts the linear decoder in a tough situation, in similar spirit to Prop.~\ref{prop:poincare_model_distortion}. This trade-off is illustrated in Figure~\ref{fig:hgcn_nr_curvature128}: at fixed and relatively small curvature ($c=0.1$), HGCN outperforms all methods except HyboNet. In accordance with the broader message of the theory, we see that both representation geometry and solver coordinates affect alignment.
On the other hand, Euclidean models work better on the grid. HGCN is able to keep up due to its learnable curvature parameter, while HyboNet performs worst among all GNN methods at small capacity. Taking into account the control experiment, where all methods perform the same, these results confirm that it is not sufficient for general tasks to consider $\delta$-hyperbolicity, but also the preservation of geometry that the tasks requires. \dionisia{In Appendix~\ref{app:node-regression}, we explore a preliminary diagnostic of target smoothness relative to graph distance. It distinguishes observed from randomized labels more clearly on Tree and Grid than on Airport
(Table~\ref{tab:alignment_diagnostic}).} Finally, our results suggest that node regression is a natural setting where hyperbolic advantages \emph{should} emerge, yet it remains largely unexplored due to the lack of suitable benchmarks, which we highlight as an important direction for future work.}

\begin{takeawaybox}
When NR is geometry-aligned on tree-like graphs, fully hyperbolic methods outperform others. Methods that rely on tangent space computations might degrade in performance if the representation manifold becomes highly curved, so this parameter must be controlled.
\end{takeawaybox}

\begin{table*}[t]
  \caption{\jonas{Link prediction (LP) and node classification(NC)} results on real-world datasets. We also report Gromov  $\delta$-hyperbolicity (lower is more hyperbolic). The last row shows the Absolute Gain of the best hyperbolic model over the best Euclidean GNN model.}
  \label{tab:lp_nc_combined}
  \begin{center}
    \begin{small}
      \begin{sc}
        \resizebox{\textwidth}{!}{
        \begin{tabular}{lccccc@{\hspace{2em}}ccccc}
        \hline
        & \multicolumn{5}{c}{\textbf{Link Prediction \jonas{(ROC AUC $\uparrow$)}}} & \multicolumn{5}{c}{\textbf{Node Classification \jonas{(Macro F1 $\uparrow$)}}} \\
        \cmidrule(lr){2-6} \cmidrule(lr){7-11}
        \textbf{Model}
        & \textbf{Disease} & \textbf{Airport} & \textbf{Cora} & \textbf{Pubmed} & \textbf{Citeseer}
        & \textbf{Disease*} & \textbf{Airport} & \textbf{Cora} & \textbf{Pubmed} & \textbf{Citeseer} \\
        & $\delta=0$ & $\delta=2$ & $\delta=4$ & $\delta=4.5$ & $\delta=6.5$
        & $\delta=0$ & $\delta=2$ & $\delta=4$ & $\delta=4.5$ & $\delta=6.5$ \\
        \hline
        MLP
        & 98.43$\pm${\scriptsize 0.74} & 97.66$\pm${\scriptsize 0.12} & 90.33$\pm${\scriptsize 0.94} & 94.52$\pm${\scriptsize 0.26} & 91.45$\pm${\scriptsize 0.68}
        & 48.62$\pm${\scriptsize 5.06} & 88.79$\pm${\scriptsize 1.97} & 55.63$\pm${\scriptsize 1.50} & 70.45$\pm${\scriptsize 1.12} & 54.29$\pm${\scriptsize 1.14} \\
        
        GCN
        & 91.67$\pm${\scriptsize 1.79} & 96.58$\pm${\scriptsize 0.35} & 92.98$\pm${\scriptsize 1.01} & 95.03$\pm${\scriptsize 0.19} & 96.02$\pm${\scriptsize 0.50}
        & 89.50$\pm${\scriptsize 4.82} & 92.46$\pm${\scriptsize 0.86} & 80.09$\pm${\scriptsize 1.01} & 78.49$\pm${\scriptsize 1.47} & 64.95$\pm${\scriptsize 1.50} \\
        
        GAT
        & 93.25$\pm${\scriptsize 1.59} & 95.67$\pm${\scriptsize 0.33} & 93.50$\pm${\scriptsize 0.68} & 96.77$\pm${\scriptsize 0.18} & 93.32$\pm${\scriptsize 0.92}
        & 85.71$\pm${\scriptsize 5.50} & 88.33$\pm${\scriptsize 5.73} & 80.71$\pm${\scriptsize 1.22} & 78.44$\pm${\scriptsize 1.88} & 65.16$\pm${\scriptsize 1.48} \\
        
        HyboNet
        & 97.99$\pm${\scriptsize 0.72} & 98.15$\pm${\scriptsize 0.17} & 93.70$\pm${\scriptsize 0.58} & 97.08$\pm${\scriptsize 0.09} & 91.53$\pm${\scriptsize 1.40}
        & 90.48$\pm${\scriptsize 5.12} & 90.87$\pm${\scriptsize 1.10} & 78.72$\pm${\scriptsize 1.01} & 78.33$\pm${\scriptsize 1.97} & 62.98$\pm${\scriptsize 1.48} \\
        
        HGCN
        & 95.16$\pm${\scriptsize 2.50} & 97.66$\pm${\scriptsize 0.15} & 93.88$\pm${\scriptsize 1.09} & 96.86$\pm${\scriptsize 0.11} & 94.46$\pm${\scriptsize 0.61}
        & 90.51$\pm${\scriptsize 3.45} & 90.92$\pm${\scriptsize 2.04} & 78.34$\pm${\scriptsize 0.95} & 78.21$\pm${\scriptsize 1.47} & 63.29$\pm${\scriptsize 2.07} \\
        
        \hline
        \textbf{Abs. Gain}
        & \textbf{+4.74} & \textbf{+1.57} & \textbf{+0.38} & \textbf{+0.31} & \textbf{-1.56}
        & \textbf{+1.01} & \textbf{-1.54} & \textbf{-1.99} & \textbf{-0.16} & \textbf{-1.87} \\
        \hline
        \end{tabular}
        }
      \end{sc}
    \end{small}
  \end{center}
  \vskip -0.1in
\end{table*}

\section{Real-World Tasks: Node Classification and Link Prediction}
\label{sec:realworld_tasks}

We assess whether HGNNs outperform Euclidean models through a unified evaluation on the standard HGNN benchmarks, covering node classification (NC) and link prediction (LP) on Disease, Airport, Cora, Pubmed, and Citeseer. We report mean $\pm$ std. over $10$ splits and use a controlled protocol with matched hyperparameter-search budgets, validation-based model selection, and comparable decoders across geometries (Appendix~\ref{app:tuning}). This comprehensive re-evaluation avoids copied literature numbers, fixed-capacity assumptions, single-split tuning artifacts, and undertuned Euclidean baselines, enabling a direct assessment of representation geometry. We further recompute Gromov $\delta$-hyperbolicity exactly with SageMath \citep{stein2005sage}, rather than by approximation (Appendix~\ref{app:gromov}).

\subsection{Node Classification}
\label{sec:nc_exp}

We evaluate two complementary aspects on standard datasets: predictive performance (Macro F1, Accuracy) and the geometry of learned representations. After end-to-end classification training, we quantify how faithfully node embeddings preserve the input graph metric using stress loss (Section~\ref{sec:pairs_distortion_pdp}). Experimental details are in Appendix~\ref{app:expset}. Compared to prior work, we also fix data-splitting bugs in Disease (marked by *) and data leakage in Airport.

\paragraph{Results.}
\label{sec:node_classification_real}

The results are reported in the right half of Table~\ref{tab:lp_nc_combined}. 
Across all benchmarks, including the most tree-like ones (Disease and Airport), HGNNs do not provide significant improvements over Euclidean counterparts, as also reflected by the Absolute Gain values.\nips{This behavior can be understood through the lens of geometry–task alignment. In NC, the label space is discrete and permutation-invariant, and therefore does not induce a continuous notion of distance between classes. As a result, training objectives such as cross-entropy primarily encourage class separability, often leading to representation collapse \cite{papyan2020prevalence}, rather than preserving the input graph metric. This is reflected in Fig.~\ref{fig:nc_real_stress_loss_norm}, where none of the methods recover shortest-path distances, regardless of the hyperbolicity of the input graph.
Importantly, this does not imply that NC is inherently misaligned with geometry. Rather, alignment in NC is more contingent on the structure of the supervision signal: when labels encode a meaningful metric structure (e.g., hierarchies or ordered relationships), geometry-aligned regimes can emerge (an example is provided in~\ref{app: node classification} in Table~\ref{tab:node_classification_synthetic}). However, such conditions are typically absent in standard benchmarks, making alignment weaker and less systematic than in LP.}

\begin{takeawaybox}
\nips{Input hyperbolicity alone is not sufficient for geometry–task alignment in NC. When the supervision signal lacks structured relationships between labels, training objectives prioritize class separation over distance preservation, limiting the advantages of hyperbolic representations.}
\end{takeawaybox}

\subsection{Link Prediction}
\label{sec:link_prediction_real}
For LP we consider the usual predictive performance metrics (ROC AUC, AP) and also the stress loss of the final embedding.
As in prior work, all models use a distance-based decoder on the final embeddings to focus on the effects of latent geometry. \dionisia{Motivated by \cite{katsman2025shedding}, which highlights that many commonly used LP benchmarks can be solved using node features alone via a simple MLP, we additionally consider controlled settings to disentangle feature-driven performance from structural advantages. Specifically, we progressively corrupt node features in highly $\delta$-hyperbolic datasets
(Disease and Airport) by injecting Gaussian noise, thereby degrading feature information.
Complementary ablations on Disease examine how the hyperbolic advantage depends on feature quality with a fixed distance-based decoder, and on decoder choice with clean features. These ablations, together with evaluations on the synthetic $\textit{Tree1111}_{\gamma}$ dataset of \cite{katsman2025shedding}, are reported in Appendix~\ref{app:link prediction}. Appendix~\ref{app:expset} contains all the details on the experimental setup. Finally, we provide a simple theoretical motivation for why LP is indeed geometry-aligned when using distance-based decoders in Appendix~\ref{app:dist_lp_decoders}, in the hope that this can serve as a starting point for future research.}

\paragraph{Results.}
\begin{figure}[t]
  \centering
  \captionsetup[subfigure]{aboveskip=2pt,belowskip=3pt}
  \begin{subfigure}{0.24\columnwidth}
      \includegraphics[width=\linewidth]{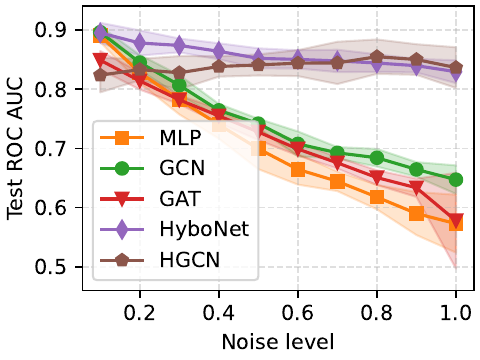}
      \caption{Disease ROC AUC}
      \label{fig:lp_noisy_Disease_roc}
  \end{subfigure}
  \begin{subfigure}{0.24\columnwidth}
      \includegraphics[width=\linewidth]{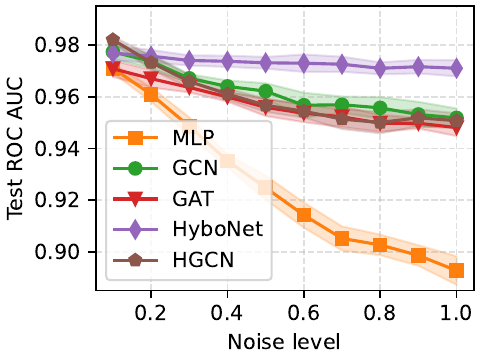}
      \caption{Airport ROC AUC}
      \label{fig:lp_noisy_airport_roc}
  \end{subfigure}
  \begin{subfigure}{0.24\columnwidth}
      \includegraphics[width=\linewidth]{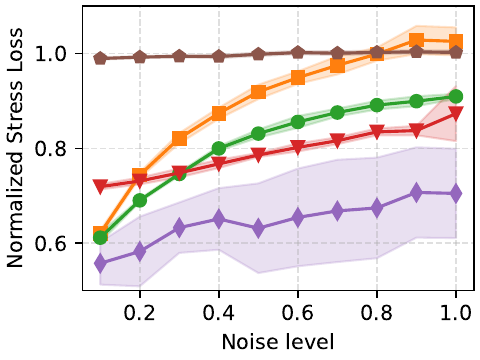}
      \caption{Disease Stress}
      \label{fig:lp_noisy_Disease_stress_loss_norm}
  \end{subfigure}
  \begin{subfigure}{0.24\columnwidth}
      \includegraphics[width=\linewidth]{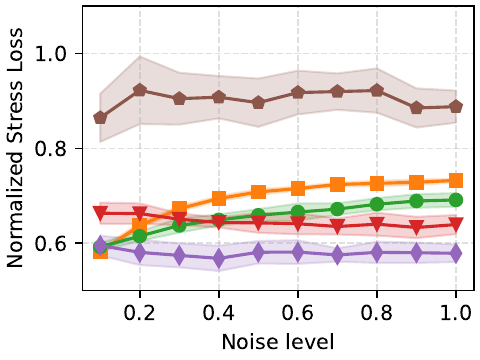}
      \caption{Airport Stress}
      \label{fig:lp_noisy_airport_stress_loss_norm}
  \end{subfigure}
  \caption{\textbf{LP} ROC AUC ($\uparrow$) and normalized Stress Loss ($\downarrow$) under feature noise on Disease and Airport. HyboNet achieves strong predictive performance while better preserving geometry.}
  \label{fig:real_noisy}
  \vspace{-1.0em}
\end{figure}

The LP results on real-world datasets are shown in the left half of Table~\ref{tab:lp_nc_combined}. HGNNs outperform Euclidean counterparts on the most tree-like datasets, Disease and Airport, whereas all GNNs obtain similar ROC AUC on the other benchmarks; the absolute gain is thus substantial on low-$\delta$ graphs and negligible otherwise. This suggests that LP is a \emph{geometry-aligned} task: edge prediction relies on local proximity and benefits from latent spaces that preserve tree-like graph structure. A particularly noteworthy observation is that concurrent work in hyperbolic benchmarking (Fig. 5 of \citet{aliakbarisani2025hypbench}) has also indirectly observed this phenomenon, for which we provide a concrete explanation for the first time. Stress analysis corroborates this interpretation: message-passing models reduce distortion relative to feature-only baselines, with HyboNet consistently achieving the lowest stress on Disease and Airport. Although MLP can match or exceed GNNs on these datasets, likely because standard LP splits fragment tree-like graphs and make strong node features sufficient \citep{bechler-speicher2024graph}, corrupting features with Gaussian noise reveals the structural advantage. As feature informativeness decreases (Fig.~\ref{fig:real_noisy}), GNNs outperform MLP, while hyperbolic models, especially HyboNet, maintain higher ROC AUC and lower distortion. The stress analysis shows the same trend (Fig.~\ref{fig:lp_noisy_Disease_stress_loss_norm}/\ref{fig:lp_noisy_airport_stress_loss_norm}): noise worsens metric preservation for all methods, but HyboNet remains the most stable and lowest-distortion model. An SNR-based separability score further explains the different MLP behavior on Disease and Airport, showing that Airport retains substantially more informative feature structure under noise (Appendix~\ref{app:LP}).

\begin{takeawaybox}

\jonas{With a distance-based decoder, LP is geometry-aligned, since it relies on relative node proximity. Consequently, recovering the graph metric is sufficient to solve the task and graph $\delta$-hyperbolicity can act as a practical proxy for identifying when HGNNs are effective.}
\end{takeawaybox}

\section{Conclusion}
\label{sec:conclusions}
We revisit when hyperbolic inductive biases benefit GNNs and show that input hyperbolicity alone is insufficient: gains arise when the task is aligned with the graph metric. HGNNs are effective for distance-based LP on tree-like graphs, where geometry and task naturally match, but do not consistently improve standard NC benchmarks, where class separability often matters more than geometric faithfulness. We also identify node regression as a promising but underexplored setting, currently limited by the lack of suitable hyperbolic benchmarks. Overall, our results provide practical guidance and position geometry--task alignment as a principled criterion for model selection.

\paragraph{Limitations and Future Work.}
Our theory currently covers geometry--task alignment for node regression with linear decoders. The extention it to distance-based decoders, especially for LP, is left to future work. To facilitate this, we provide an initial step in Appendix \ref{app:dist_lp_decoders}. It also does not yet yield a quantitative heuristic for selecting hyperbolic models on arbitrary datasets, which we aim to develop. Finally, existing benchmarks remain limited: Disease and Airport are partly saturated, and node-regression benchmarks with clear hyperbolic structure are missing. We therefore rely on synthetic datasets to isolate geometric effects under controlled conditions, while broader real-world validation will require the introduction of new benchmarks.

\clearpage
\section*{Acknowledgments}
This work was supported by the BRAINFEDERATION FISA-2022-00899, by NEXTGENERATIONEU (NGEU) and funded by the Italian Ministry of University and Research (MUR), National Recovery and Resilience Plan (NRRP), project MNESYS (PE0000006) (to NT)- A Multiscale integrated approach to the study of the nervous system in health and disease (DN. 1553 11.10.2022); the CROSSBRAIN project (European Union’s European Innovation Council under grant agreement No. 101070908; The BRAINSTORM PROJECT under Grant Agreement number 101099355.
The work is partially funded by the Norwegian Research Council through the Centre of Excellence Integreat - The
Norwegian Centre for knowledge-driven machine learning (project 332645).

\bibliographystyle{unsrtnat}
\bibliography{references}

\begin{thebibliography}{66}
\providecommand{\natexlab}[1]{#1}
\providecommand{\url}[1]{\texttt{#1}}
\expandafter\ifx\csname urlstyle\endcsname\relax
  \providecommand{\doi}[1]{doi: #1}\else
  \providecommand{\doi}{doi: \begingroup \urlstyle{rm}\Url}\fi

\bibitem[Khoshraftar and An(2024)]{khoshraftar2024survey}
Shima Khoshraftar and Aijun An.
\newblock A survey on graph representation learning methods.
\newblock \emph{ACM Transactions on Intelligent Systems and Technology}, 15\penalty0 (1):\penalty0 1--55, 2024.

\bibitem[Perozzi et~al.(2014)Perozzi, Al-Rfou, and Skiena]{perozzi2014deepwalk}
Bryan Perozzi, Rami Al-Rfou, and Steven Skiena.
\newblock Deepwalk: Online learning of social representations.
\newblock In \emph{Proceedings of the 20th ACM SIGKDD international conference on Knowledge discovery and data mining}, pages 701--710, 2014.

\bibitem[Grover and Leskovec(2016)]{grover2016node2vec}
Aditya Grover and Jure Leskovec.
\newblock node2vec: Scalable feature learning for networks.
\newblock In \emph{Proceedings of the 22nd ACM SIGKDD international conference on Knowledge discovery and data mining}, pages 855--864, 2016.

\bibitem[Matsumoto et~al.(2021)Matsumoto, Mimori, and Fukunaga]{matsumoto2021novel}
Hirotaka Matsumoto, Takahiro Mimori, and Tsukasa Fukunaga.
\newblock Novel metric for hyperbolic phylogenetic tree embeddings.
\newblock \emph{Biology Methods and Protocols}, 6\penalty0 (1):\penalty0 bpab006, 2021.

\bibitem[Kolyvakis et~al.(2020)Kolyvakis, Kalousis, and Kiritsis]{kolyvakis2019hyperkg}
Prodromos Kolyvakis, Alexandros Kalousis, and Dimitris Kiritsis.
\newblock Hyperbolic knowledge graph embeddings for knowledge base completion.
\newblock In \emph{European Semantic Web Conference}, pages 199--214. Springer, 2020.

\bibitem[Yang et~al.(2023)Yang, Wu, Zhang, Hong, Zhou, Zhang, Zhou, and Wang]{yang2023hyperbolic}
Yonghui Yang, Le~Wu, Kun Zhang, Richang Hong, Hailin Zhou, Zhiqiang Zhang, Jun Zhou, and Meng Wang.
\newblock Hyperbolic graph learning for social recommendation.
\newblock \emph{IEEE Transactions on Knowledge and Data Engineering}, 36\penalty0 (12):\penalty0 8488--8501, 2023.

\bibitem[Sarkar(2011)]{sarkar2011low}
Rik Sarkar.
\newblock Low distortion delaunay embedding of trees in hyperbolic plane.
\newblock In \emph{International symposium on graph drawing}, pages 355--366. Springer, 2011.

\bibitem[Nickel and Kiela(2018)]{nickel2018learning}
Maximillian Nickel and Douwe Kiela.
\newblock Learning continuous hierarchies in the lorentz model of hyperbolic geometry.
\newblock In \emph{International conference on machine learning}, pages 3779--3788. PMLR, 2018.

\bibitem[Kitsak et~al.(2020)Kitsak, Voitalov, and Krioukov]{kitsak2020link}
Maksim Kitsak, Ivan Voitalov, and Dmitri Krioukov.
\newblock Link prediction with hyperbolic geometry.
\newblock \emph{Physical Review Research}, 2\penalty0 (4):\penalty0 043113, 2020.

\bibitem[Wang et~al.(2019)Wang, Zhang, and Shi]{wang2019hyperbolic}
Xiao Wang, Yiding Zhang, and Chuan Shi.
\newblock Hyperbolic heterogeneous information network embedding.
\newblock In \emph{Proceedings of the AAAI conference on artificial intelligence}, volume~33, pages 5337--5344, 2019.

\bibitem[Chami et~al.(2020)Chami, Wolf, Juan, Sala, Ravi, and R{\'e}]{chami2020low}
Ines Chami, Adva Wolf, Da-Cheng Juan, Frederic Sala, Sujith Ravi, and Christopher R{\'e}.
\newblock Low-dimensional hyperbolic knowledge graph embeddings.
\newblock In \emph{Proceedings of the 58th Annual Meeting of the Association for Computational Linguistics}, pages 6901--6914, 2020.

\bibitem[Pan and Wang(2021)]{pan2021hyperbolic}
Zhe Pan and Peng Wang.
\newblock Hyperbolic hierarchy-aware knowledge graph embedding for link prediction.
\newblock In \emph{Findings of the Association for Computational Linguistics: EMNLP 2021}, pages 2941--2948, 2021.

\bibitem[Nickel and Kiela(2017)]{nickel2017poincare}
Maximillian Nickel and Douwe Kiela.
\newblock Poincar{\'e} embeddings for learning hierarchical representations.
\newblock \emph{Advances in neural information processing systems}, 30, 2017.

\bibitem[Scarselli et~al.(2008)Scarselli, Gori, Tsoi, Hagenbuchner, and Monfardini]{scarselli2008graph}
Franco Scarselli, Marco Gori, Ah~Chung Tsoi, Markus Hagenbuchner, and Gabriele Monfardini.
\newblock The graph neural network model.
\newblock \emph{IEEE transactions on neural networks}, 20\penalty0 (1):\penalty0 61--80, 2008.

\bibitem[Kipf and Welling(2017)]{kipf2017semi}
Thomas~N Kipf and Max Welling.
\newblock Semi-supervised classification with graph convolutional networks.
\newblock In \emph{International Conference on Learning Representations}, 2017.

\bibitem[Hamilton et~al.(2017)Hamilton, Ying, and Leskovec]{hamilton2017graphsage}
William~L. Hamilton, Rex Ying, and Jure Leskovec.
\newblock Inductive representation learning on large graphs.
\newblock In \emph{Advances in Neural Information Processing Systems (NeurIPS)}, 2017.

\bibitem[Veli{\v{c}}kovi{\'c} et~al.(2018)Veli{\v{c}}kovi{\'c}, Cucurull, Casanova, Romero, Li{\`o}, and Bengio]{velickovic2018gat}
Petar Veli{\v{c}}kovi{\'c}, Guillem Cucurull, Arantxa Casanova, Adriana Romero, Pietro Li{\`o}, and Yoshua Bengio.
\newblock Graph attention networks.
\newblock In \emph{International Conference on Learning Representations (ICLR)}, 2018.

\bibitem[Liu et~al.(2019)Liu, Nickel, and Kiela]{liu2019hyperbolic}
Qi~Liu, Maximilian Nickel, and Douwe Kiela.
\newblock Hyperbolic graph neural networks.
\newblock \emph{Advances in neural information processing systems}, 32, 2019.

\bibitem[Chami et~al.(2019)Chami, Ying, R{\'e}, and Leskovec]{chami2019hyperbolic}
Ines Chami, Zhitao Ying, Christopher R{\'e}, and Jure Leskovec.
\newblock Hyperbolic graph convolutional neural networks.
\newblock \emph{Advances in neural information processing systems}, 32, 2019.

\bibitem[Zhang et~al.(2021)Zhang, Wang, Shi, Liu, and Song]{zhang2021lorentzian}
Yiding Zhang, Xiao Wang, Chuan Shi, Nian Liu, and Guojie Song.
\newblock Lorentzian graph convolutional networks.
\newblock In \emph{Proceedings of the web conference 2021}, pages 1249--1261, 2021.

\bibitem[Chen et~al.(2022)Chen, Han, Lin, Zhao, Liu, Li, Sun, and Zhou]{chen2022fully}
Weize Chen, Xu~Han, Yankai Lin, Hexu Zhao, Zhiyuan Liu, Peng Li, Maosong Sun, and Jie Zhou.
\newblock Fully hyperbolic neural networks.
\newblock In \emph{Proceedings of the 60th Annual Meeting of the Association for Computational Linguistics (Volume 1: Long Papers)}, pages 5672--5686, 2022.

\bibitem[Gromov(1987)]{gromov1987hyperbolic}
Mikhael Gromov.
\newblock Hyperbolic groups.
\newblock In \emph{Essays in group theory}, pages 75--263. Springer, 1987.

\bibitem[Yang et~al.(2024)Yang, Verma, Zhang, Liu, King, and Ying]{yang2024hypformer}
Menglin Yang, Harshit Verma, Delvin~Ce Zhang, Jiahong Liu, Irwin King, and Rex Ying.
\newblock Hypformer: Exploring efficient transformer fully in hyperbolic space.
\newblock In \emph{Proceedings of the 30th ACM SIGKDD Conference on Knowledge Discovery and Data Mining}, pages 3770--3781, 2024.

\bibitem[Yang et~al.(2022)Yang, Zhou, Li, Liu, Pan, Xiong, and King]{yang2022hyperbolic}
Menglin Yang, Min Zhou, Zhihao Li, Jiahong Liu, Lujia Pan, Hui Xiong, and Irwin King.
\newblock Hyperbolic graph neural networks: A review of methods and applications.
\newblock \emph{arXiv preprint arXiv:2202.13852}, 2022.

\bibitem[Baker et~al.(2024)Baker, Su{\'a}rez-M{\'e}ndez, Smith, Marsh, Funke, Mosher, Maest{\'u}, Xu, and Pantazis]{baker2024hyperbolic}
Cole Baker, Isabel Su{\'a}rez-M{\'e}ndez, Grace Smith, Elisabeth~B Marsh, Michael Funke, John~C Mosher, Fernando Maest{\'u}, Mengjia Xu, and Dimitrios Pantazis.
\newblock Hyperbolic graph embedding of meg brain networks to study brain alterations in individuals with subjective cognitive decline.
\newblock \emph{IEEE journal of biomedical and health informatics}, 28\penalty0 (12):\penalty0 7357--7368, 2024.

\bibitem[Li et~al.(2024)Li, Zhang, Cui, and Ma]{li2024hyperbolic}
Yancong Li, Xiaoming Zhang, Ying Cui, and Shuai Ma.
\newblock Hyperbolic graph neural network for temporal knowledge graph completion.
\newblock In \emph{Proceedings of the 2024 Joint International Conference on Computational Linguistics, Language Resources and Evaluation (LREC-COLING 2024)}, pages 8474--8486, 2024.

\bibitem[Katsman and Gilbert(2025)]{katsman2025shedding}
Isay Katsman and Anna Gilbert.
\newblock Shedding light on problems with hyperbolic graph learning.
\newblock \emph{Transactions on Machine Learning Research}, 2025.
\newblock ISSN 2835-8856.
\newblock Reproducibility Certification.

\bibitem[Naddeo et~al.(2025)Naddeo, Azevedo, and Toschi]{naddeo2025we}
Dionisia Naddeo, Tiago Azevedo, and Nicola Toschi.
\newblock Do we need curved spaces? a critical look at hyperbolic graph learning in graph classification.
\newblock In \emph{Joint European Conference on Machine Learning and Knowledge Discovery in Databases}, pages 649--664. Springer, 2025.

\bibitem[Ganea et~al.(2018)Ganea, B{\'e}cigneul, and Hofmann]{ganea2018hyperbolic}
Octavian Ganea, Gary B{\'e}cigneul, and Thomas Hofmann.
\newblock Hyperbolic neural networks.
\newblock \emph{Advances in neural information processing systems}, 31, 2018.

\bibitem[Wu et~al.(2023)Wu, Zhao, Yang, Zhang, Nie, Jiang, Bian, and Yan]{wu2023sgformer}
Qitian Wu, Wentao Zhao, Chenxiao Yang, Hengrui Zhang, Fan Nie, Haitian Jiang, Yatao Bian, and Junchi Yan.
\newblock Sgformer: Simplifying and empowering transformers for large-graph representations.
\newblock \emph{Advances in neural information processing systems}, 36:\penalty0 64753--64773, 2023.

\bibitem[Canzar et~al.(2024)Canzar, Do, Jeli{\'c}, Laue, Matijevi{\'c}, and Prusina]{canzar2024metric}
Stefan Canzar, Van~Hoan Do, Slobodan Jeli{\'c}, S{\"o}ren Laue, Domagoj Matijevi{\'c}, and Tomislav Prusina.
\newblock Metric multidimensional scaling for large single-cell datasets using neural networks.
\newblock \emph{Algorithms for molecular biology}, 19\penalty0 (1):\penalty0 21, 2024.

\bibitem[Nash(1956)]{nash1956imbedding}
John Nash.
\newblock The imbedding problem for riemannian manifolds.
\newblock \emph{Annals of mathematics}, 63\penalty0 (1):\penalty0 20--63, 1956.

\bibitem[Zhang et~al.(2017)Zhang, Bengio, Hardt, Recht, and Vinyals]{zhang2017understanding}
Chiyuan Zhang, Samy Bengio, Moritz Hardt, Benjamin Recht, and Oriol Vinyals.
\newblock Understanding deep learning requires rethinking generalization.
\newblock In \emph{International Conference on Learning Representations}, 2017.
\newblock URL \url{https://openreview.net/forum?id=Sy8gdB9xx}.

\bibitem[Wen et~al.(2024)Wen, Tang, Ouyang, Shen, Yang, Zhu, Chen, and Wei]{wen2024hyperbolic}
Lingfeng Wen, Xuan Tang, Mingjie Ouyang, Xiangxiang Shen, Jian Yang, Daxin Zhu, Mingsong Chen, and Xian Wei.
\newblock Hyperbolic graph diffusion model.
\newblock In \emph{Proceedings of the AAAI Conference on Artificial Intelligence}, volume~38, pages 15823--15831, 2024.

\bibitem[Stein and Joyner(2005)]{stein2005sage}
William Stein and David Joyner.
\newblock Sage: system for algebra and geometry experimentation.
\newblock \emph{SIGSAM Bull.}, 39\penalty0 (2):\penalty0 61–64, June 2005.
\newblock ISSN 0163-5824.
\newblock \doi{10.1145/1101884.1101889}.

\bibitem[Papyan et~al.(2020)Papyan, Han, and Donoho]{papyan2020prevalence}
Vardan Papyan, XY~Han, and David~L Donoho.
\newblock Prevalence of neural collapse during the terminal phase of deep learning training.
\newblock \emph{Proceedings of the National Academy of Sciences}, 117\penalty0 (40):\penalty0 24652--24663, 2020.

\bibitem[Aliakbarisani et~al.(2025)Aliakbarisani, Jankowski, Serrano, and Bogu{\~n}a]{aliakbarisani2025hypbench}
Roya Aliakbarisani, Robert Jankowski, M.~Angeles Serrano, and Marian Bogu{\~n}a.
\newblock Hypbench: Hyperbolic benchmark for graph neural network performance.
\newblock In \emph{Non-Euclidean Foundation Models: Advancing AI Beyond Euclidean Frameworks}, 2025.
\newblock URL \url{https://openreview.net/forum?id=Mem1jlsEfp}.

\bibitem[Bechler-Speicher et~al.(2024)Bechler-Speicher, Amos, Gilad-Bachrach, and Globerson]{bechler-speicher2024graph}
Maya Bechler-Speicher, Ido Amos, Ran Gilad-Bachrach, and Amir Globerson.
\newblock Graph neural networks use graphs when they shouldn't.
\newblock In \emph{Proceedings of the 41st International Conference on Machine Learning}, ICML'24. JMLR.org, 2024.

\bibitem[Chepoi et~al.(2012)Chepoi, Dragan, Estellon, Habib, Vaxes, and Xiang]{chepoi2012additive}
Victor Chepoi, Feodor~F Dragan, Bertrand Estellon, Michel Habib, Yann Vaxes, and Yang Xiang.
\newblock Additive spanners and distance and routing labeling schemes for hyperbolic graphs.
\newblock \emph{Algorithmica}, 62\penalty0 (3):\penalty0 713--732, 2012.

\bibitem[Krioukov et~al.(2010)Krioukov, Papadopoulos, Kitsak, Vahdat, and Bogun{\'a}]{krioukov2010hyperbolic}
Dmitri Krioukov, Fragkiskos Papadopoulos, Maksim Kitsak, Amin Vahdat, and Mari{\'a}n Bogun{\'a}.
\newblock Hyperbolic geometry of complex networks.
\newblock \emph{Physical Review E—Statistical, Nonlinear, and Soft Matter Physics}, 82\penalty0 (3):\penalty0 036106, 2010.

\bibitem[Papadopoulos et~al.(2010)Papadopoulos, Krioukov, Bogun{\'a}, and Vahdat]{papadopoulos2010greedy}
Fragkiskos Papadopoulos, Dmitri Krioukov, Mari{\'a}n Bogun{\'a}, and Amin Vahdat.
\newblock Greedy forwarding in dynamic scale-free networks embedded in hyperbolic metric spaces.
\newblock In \emph{2010 Proceedings IEEE Infocom}, pages 1--9. IEEE, 2010.

\bibitem[Sonthalia and Gilbert(2020)]{sonthalia2020tree}
Rishi Sonthalia and Anna Gilbert.
\newblock Tree! i am no tree! i am a low dimensional hyperbolic embedding.
\newblock \emph{Advances in Neural Information Processing Systems}, 33:\penalty0 845--856, 2020.

\bibitem[Becigneul and Ganea(2019)]{becigneul2018riemannian}
Gary Becigneul and Octavian-Eugen Ganea.
\newblock Riemannian adaptive optimization methods.
\newblock In \emph{International Conference on Learning Representations}, 2019.

\bibitem[Shimizu et~al.(2021)Shimizu, Mukuta, and Harada]{shimizu2020hyperbolic}
Ryohei Shimizu, YUSUKE Mukuta, and Tatsuya Harada.
\newblock Hyperbolic neural networks++.
\newblock In \emph{International Conference on Learning Representations}, 2021.

\bibitem[Katsman et~al.(2023)Katsman, Chen, Holalkere, Asch, Lou, Lim, and De~Sa]{katsman2023riemannian}
Isay Katsman, Eric Chen, Sidhanth Holalkere, Anna Asch, Aaron Lou, Ser~Nam Lim, and Christopher~M De~Sa.
\newblock Riemannian residual neural networks.
\newblock \emph{Advances in Neural Information Processing Systems}, 36:\penalty0 63502--63514, 2023.

\bibitem[Benedetti and Petronio(1992)]{benedetti1992lectures}
Riccardo Benedetti and Carlo Petronio.
\newblock \emph{Lectures on hyperbolic geometry}.
\newblock Springer Science \& Business Media, 1992.

\bibitem[Martelli(2016)]{martelli2016introduction}
Bruno Martelli.
\newblock An introduction to geometric topology.
\newblock \emph{arXiv preprint arXiv:1610.02592}, 2016.

\bibitem[Ratcliffe(2019)]{ratcliffe2019hyperbolic}
John~G Ratcliffe.
\newblock Hyperbolic n-manifolds.
\newblock In \emph{Foundations of Hyperbolic Manifolds}, pages 506--596. Springer, 2019.

\bibitem[Lee(2018)]{lee2018introduction}
John~M Lee.
\newblock \emph{Introduction to Riemannian manifolds}, volume~2.
\newblock Springer, 2018.

\bibitem[Lee(2006)]{lee2006riemannian}
John~M Lee.
\newblock \emph{Riemannian manifolds: an introduction to curvature}, volume 176.
\newblock Springer Science \& Business Media, 2006.

\bibitem[Bonk and Schramm(2011)]{bonk2011embeddings}
Mario Bonk and Oded Schramm.
\newblock Embeddings of gromov hyperbolic spaces.
\newblock In \emph{Selected Works of Oded Schramm}, pages 243--284. Springer, 2011.

\bibitem[Hu et~al.(2020)Hu, Fey, Zitnik, Dong, Ren, Liu, Catasta, and Leskovec]{hu2020ogb}
Weihua Hu, Matthias Fey, Marinka Zitnik, Yuxiao Dong, Hongyu Ren, Bowen Liu, Michele Catasta, and Jure Leskovec.
\newblock Open graph benchmark: Datasets for machine learning on graphs.
\newblock \emph{Advances in neural information processing systems}, 33:\penalty0 22118--22133, 2020.

\bibitem[Lim et~al.(2021)Lim, Hohne, Li, Huang, Gupta, Bhalerao, and Lim]{lim2021large}
Derek Lim, Felix~Matthew Hohne, Xiuyu Li, Sijia~Linda Huang, Vaishnavi Gupta, Omkar~Prasad Bhalerao, and Ser-Nam Lim.
\newblock Large scale learning on non-homophilous graphs: New benchmarks and strong simple methods.
\newblock In A.~Beygelzimer, Y.~Dauphin, P.~Liang, and J.~Wortman Vaughan, editors, \emph{Advances in Neural Information Processing Systems}, 2021.
\newblock URL \url{https://openreview.net/forum?id=DfGu8WwT0d}.

\bibitem[Cohen et~al.(2012)Cohen, Coudert, and Lancin]{cohen2012exact}
Nathann Cohen, David Coudert, and Aur{\'e}lien Lancin.
\newblock \emph{Exact and approximate algorithms for computing the hyperbolicity of large-scale graphs}.
\newblock PhD thesis, INRIA, 2012.

\bibitem[Coudert et~al.(2022)Coudert, Nusser, and Viennot]{coudert2021hyperbolicity}
David Coudert, Andr{\'e} Nusser, and Laurent Viennot.
\newblock Hyperbolicity computation through dominating sets.
\newblock In \emph{ALENEX 2022-SIAM Symposium on Algorithm Engineering and Experiments}, pages 78--90, 2022.

\bibitem[Sen et~al.(2008)Sen, Namata, Bilgic, Getoor, Gallagher, and Eliassi‐Rad]{sen2008collective}
Prithviraj Sen, Galileo Namata, Mustafa Bilgic, Lise Getoor, Brian Gallagher, and Tina Eliassi‐Rad.
\newblock Collective classification in network data.
\newblock \emph{AI Mag.}, 29\penalty0 (3):\penalty0 93–106, 2008.
\newblock ISSN 0738-4602.
\newblock \doi{10.1609/aimag.v29i3.2157}.

\bibitem[Greene et~al.(2015)Greene, Krishnan, Wong, Ricciotti, Zelaya, Himmelstein, Zhang, Hartmann, Zaslavsky, Sealfon, et~al.]{greene2015understanding}
Casey~S Greene, Arjun Krishnan, Aaron~K Wong, Emanuela Ricciotti, Rene~A Zelaya, Daniel~S Himmelstein, Ran Zhang, Boris~M Hartmann, Elena Zaslavsky, Stuart~C Sealfon, et~al.
\newblock Understanding multicellular function and disease with human tissue-specific networks.
\newblock \emph{Nature genetics}, 47\penalty0 (6):\penalty0 569--576, 2015.

\bibitem[Rozemberczki and Sarkar(2020)]{rozemberczki2020characteristic}
Benedek Rozemberczki and Rik Sarkar.
\newblock Characteristic functions on graphs: Birds of a feather, from statistical descriptors to parametric models.
\newblock In \emph{Proceedings of the 29th ACM international conference on information \& knowledge management}, pages 1325--1334, 2020.

\bibitem[Platonov et~al.(2023)Platonov, Kuznedelev, Diskin, Babenko, and Prokhorenkova]{platonov2023a}
Oleg Platonov, Denis Kuznedelev, Michael Diskin, Artem Babenko, and Liudmila Prokhorenkova.
\newblock A critical look at the evaluation of {GNN}s under heterophily: Are we really making progress?
\newblock In \emph{The Eleventh International Conference on Learning Representations}, 2023.

\bibitem[McAuley et~al.(2015)McAuley, Targett, Shi, and van~den Hengel]{mcauley2015image}
Julian McAuley, Christopher Targett, Qinfeng Shi, and Anton van~den Hengel.
\newblock Image-based recommendations on styles and substitutes.
\newblock In \emph{Proceedings of the 38th International ACM SIGIR Conference on Research and Development in Information Retrieval}, SIGIR '15, page 43–52, New York, NY, USA, 2015. Association for Computing Machinery.
\newblock ISBN 9781450336215.
\newblock \doi{10.1145/2766462.2767755}.

\bibitem[Rozemberczki et~al.(2021)Rozemberczki, Allen, and Sarkar]{rozemberczki2021multi}
Benedek Rozemberczki, Carl Allen, and Rik Sarkar.
\newblock {Multi-Scale attributed node embedding}.
\newblock \emph{Journal of Complex Networks}, 9\penalty0 (2):\penalty0 cnab014, 05 2021.
\newblock ISSN 2051-1329.
\newblock \doi{10.1093/comnet/cnab014}.

\bibitem[Tang et~al.(2009)Tang, Sun, Wang, and Yang]{tang2009social}
Jie Tang, Jimeng Sun, Chi Wang, and Zi~Yang.
\newblock Social influence analysis in large-scale networks.
\newblock In \emph{Proceedings of the 15th ACM SIGKDD International Conference on Knowledge Discovery and Data Mining}, KDD '09, page 807–816, New York, NY, USA, 2009. Association for Computing Machinery.
\newblock ISBN 9781605584959.
\newblock \doi{10.1145/1557019.1557108}.

\bibitem[Ribeiro et~al.(2017)Ribeiro, Saverese, and Figueiredo]{ribeiro2017struc2vec}
Leonardo~FR Ribeiro, Pedro~HP Saverese, and Daniel~R Figueiredo.
\newblock struc2vec: Learning node representations from structural identity.
\newblock In \emph{Proceedings of the 23rd ACM SIGKDD international conference on knowledge discovery and data mining}, pages 385--394, 2017.

\bibitem[Pei et~al.(2020)Pei, Wei, Chang, Lei, and Yang]{pei2020geom}
Hongbin Pei, Bingzhe Wei, Kevin Chen-Chuan Chang, Yu~Lei, and Bo~Yang.
\newblock Geom-gcn: Geometric graph convolutional networks.
\newblock In \emph{International Conference on Learning Representations}, 2020.

\bibitem[Shchur et~al.(2018)Shchur, Mumme, Bojchevski, and G{\"u}nnemann]{shchur2019pitfallsgraphneuralnetwork}
Oleksandr Shchur, Maximilian Mumme, Aleksandar Bojchevski, and Stephan G{\"u}nnemann.
\newblock Pitfalls of graph neural network evaluation.
\newblock \emph{Relational Representation Learning Workshop, NeurIPS 2018}, 2018.

\bibitem[Kingma and Ba(2015)]{kingma2014adam}
Diederik~P. Kingma and Jimmy Ba.
\newblock Adam: {A} method for stochastic optimization.
\newblock In \emph{3rd International Conference on Learning Representations, {ICLR} 2015, San Diego, CA, USA, May 7-9, 2015, Conference Track Proceedings}, 2015.

\end{thebibliography}

\newpage
\appendix
\onecolumn

\section{Related Works}\label{app:relwork}
Trees constitute a fundamental class of data structures. From a metric perspective, it is well established that trees are prototypical examples of negatively curved spaces. In his seminal work on hyperbolic groups, Gromov introduced the notion of $\delta$-hyperbolicity \cite{gromov1987hyperbolic}, formalizing hyperbolicity as a coarse, quasi-isometry–invariant property that depends only on large-scale geometry rather than local details or the choice of generating set (See Appendix \ref{app:gromov} for computational details). In this sense, trees are $0$-hyperbolic and share with hyperbolic spaces key global properties, including exponential volume growth.

Building on this perspective, \citet{chepoi2012additive} showed that any $n$-node $\delta$-hyperbolic graph admits a tree approximation with additive error $O(\delta \log n)$, demonstrating that tree metrics provide faithful approximations of distances in hyperbolic graphs. In the special case of trees themselves, \citet{sarkar2011low} further strengthened this connection by providing a constructive embedding result: any weighted tree can be embedded into the two-dimensional hyperbolic plane with arbitrarily small distortion, preserving both topology and metric structure.

Real-world networks are rarely trees; however, many are approximately tree-like at large scales. This includes networks with scale-free degree distributions, where the probability $P(k)$ that a node has degree $k$ follows a power law $P(k)\sim k^{-\gamma}$, with exponent $\gamma$ typically between $2$ and $3$. Such networks often display strong heterogeneity, high clustering, and efficient navigability.

\citet{krioukov2010hyperbolic} proposed a geometric framework in which these  structural features arise naturally from an underlying hyperbolic geometry. In this model, degree heterogeneity is directly controlled by the negative curvature of the space, while clustering emerges as a consequence of its metric properties. Within the same framework, Papadopoulos et al. \citet{papadopoulos2010greedy} showed that scale-free network topologies naturally emerge from hyperbolic spaces and support efficient greedy routing based solely on geometric distances. Together, these results suggest that hyperbolic geometry provides a unifying explanation for the large-scale, tree-like organization observed in many real-world networks.

Rather than embedding a known tree, a complementary line of work aims to recover latent hierarchical structure directly from data by learning representations in hyperbolic space. Motivated by the exponential growth properties of hyperbolic geometry, these approaches leverage curvature as an inductive bias for modeling hierarchical organization. A seminal contribution in this direction was introduced by \citet{nickel2017poincare}, who proposed Poincaré embeddings for learning continuous representations of symbolic data with latent hierarchies in an unsupervised setting, demonstrating substantial improvements over Euclidean embeddings. Building on this idea, \citet{sonthalia2020tree} proposed to first infer an approximate tree structure from the data and subsequently embed it into a low-dimensional hyperbolic space. Overall, these works bridge classical geometric embedding theory and modern representation learning, positioning hyperbolic space as a powerful tool for hierarchy induction rather than explicit metric reconstruction. 

While hyperbolic space provides a powerful inductive bias for modeling hierarchical structure, learning in negatively curved spaces requires optimization techniques that go beyond standard Euclidean machinery. In particular, gradient-based learning must be generalized to Riemannian manifolds, where updates are performed along geodesics and parameters are constrained to remain on the manifold. \citet{becigneul2018riemannian} addressed this challenge by extending popular adaptive optimization methods, including Adam and AdaGrad, to the Riemannian setting, providing both practical algorithms and theoretical guarantees. In parallel, the choice of the specific hyperbolic model was shown to have significant implications for numerical stability and computational efficiency. \citet{nickel2018learning} demonstrated that learning embeddings in the Lorentz (hyperboloid) model is substantially more efficient than in the Poincaré ball, improving optimization behavior while preserving the representational advantages of hyperbolic geometry.

Building on advances in Riemannian optimization, hyperbolic geometry has been progressively integrated into standard neural architectures, enabling the transition from static embedding methods to fully trainable models. In this setting, core neural operations such as linear transformations, non-linearities, attention mechanisms, and residual connections are reformulated to operate directly in hyperbolic space while respecting its geometric constraints. \citet{ganea2018hyperbolic} provided a foundational framework for hyperbolic neural networks by generalizing key components of deep learning to negatively curved manifolds. Subsequent work further expanded this paradigm, proposing unified formulations of neural network layers in the Poincaré ball \cite{shimizu2020hyperbolic} and extending residual architectures to general Riemannian manifolds in a principled manner \cite{katsman2023riemannian}. Collectively, these contributions established hyperbolic neural networks as a viable and scalable modeling paradigm for learning hierarchical representations in an end-to-end fashion.

An active line of research has focused on extending GNNs to hyperbolic spaces in order to combine message passing with the representational advantages of negative curvature. Two seminal works published in parallel in 2019 introduced the first hyperbolic GNN architectures \cite{liu2019hyperbolic,chami2019hyperbolic}. These approaches generalize Euclidean message passing by mapping node representations between hyperbolic space and local Euclidean tangent spaces via logarithmic and exponential maps. In particular, \citet{chami2019hyperbolic} proposed an hyperbolic graph convolutional network with learnable curvature and attention mechanisms defined in hyperbolic space. Subsequent work sought to reduce reliance on tangent-space operations in order to better preserve geometric structure. \citet{zhang2021lorentzian} introduced a Lorentzian graph convolutional network that performs neighborhood aggregation directly in the hyperbolic manifold by computing centroids under the squared Lorentzian distance. Along similar lines, \citet{chen2022fully} proposed a fully hyperbolic neural network operating entirely in the Lorentz model, demonstrating strong empirical performance on a range of benchmarks.

Despite the growing body of work on HGNNs, the benefits of hyperbolic message passing in supervised settings remain under active debate. Recently, \citet{katsman2025shedding} provided a systematic reassessment of hyperbolic graph learning, highlighting several methodological issues related to baseline selection, modeling assumptions, and evaluation protocols. Notably, they showed that when Euclidean models with comparable capacity are carefully trained under the same experimental conditions, they often match or outperform existing hyperbolic graph representation learning methods, even on datasets previously characterized as highly hyperbolic according to Gromov hyperbolicity, including tree-structured graphs. Naddeo et al. \citet{naddeo2025we} investigate the interplay between node features and representation geometry in graph classification, showing that hyperbolic models do not necessarily outperform Euclidean ones even on tree-structured graphs. However, their analysis does not examine whether such advantages depend on the alignment between graph geometry and task supervision. Our work extends these reassessments by considering the interplay between representation geometry and task supervision, providing missing details on how the learning process might or might not require a geometric inductive bias that preserves the input metric structure. \\

\section{Extended Background}
\label{app:extended background}
\subsection{Hyperbolic Geometry}
\label{appendix:geometry}
In this section we provide more details on the exponential growth of (a model of) hyperbolic space and also the architectural mechanisms of HGNNs. For a more rigorous and detailed treatment of hyperbolic geometry and hyperbolic manifolds, we refer the reader to \cite{benedetti1992lectures,martelli2016introduction,ratcliffe2019hyperbolic}, while \cite{lee2018introduction, lee2006riemannian} are fantastic resources for the theory of Riemannian manifolds. For a general overview of hyperbolic graph learning and applications, please consult the review of \citet{yang2022hyperbolic}.

\paragraph{The Hyperboloid Model.}
To illustrate exponential volume growth, we consider (w.l.o.g) the hyperbolic space $\mathcal{L}^d_c$ with unit negative curvature which we will simply denote by $\mathcal{L}^d$.
\begin{definition}[Lorentzian scalar product]
Let $u = (u_0, u_1, \ldots, u_d)$ and $v = (v_0, v_1, \ldots, v_d)$ be vectors in $\mathbb{R}^{d+1}$. 
The \emph{Lorentzian scalar product} on $\mathbb{R}^{d+1}$ is defined as
\[
\langle u,v \rangle_{1,d} = -u_0 v_0 + \sum_{i=1}^d u_i v_i.
\]
\end{definition}
This bilinear form induces a pseudo-metric of signature $(1,d)$ on $\mathbb{R}^{d+1}$. We can now define an inner product space by equipping $\mathbb{R}^{d+1}$  with the Lorentzian scalar product, called the $(d+1)$-dimensional Minkowski space
\[
\mathbb{R}^{1,d} = (\mathbb{R}^{d+1}, \langle \cdot,\cdot \rangle_{1,d}).
\]

For any $x \in \mathbb{R}^{1,d}$ we define its \emph{orthogonal complement} as
\[
x^\perp = \{ y \in \mathbb{R}^{1,d} : \langle x,y \rangle_{1,d} = 0 \}.
\]

We can now define the hyperboloid as the set of all points of norm $-1$ in the $(d+1)$-dimensional Minkowski space. This set is a two-sheeted hyperboloid, and by restricting to the upper sheet to get the Lorentz model
\[
\mathcal{L}^d = \bigl\{ x = (x_0, \ldots, x_d) \in \mathbb{R}^{1,d} : 
-x_0^2 + x_1^2 + \cdots + x_d^2 = -1,\; x_0 > 0 \bigr\}.
\]

It is possible to show that $\mathcal{L}^d$ is a differentiable manifold and, for every $x \in \mathcal{L}^d$, the tangent space $T_x \mathcal{L}^d$ coincides with $x^\perp$. For intuition about this fact, consider that the hyperboloid is an anti-sphere (a sphere of imaginary radius). 

To define the metric on this manifold, we can use the notion of induced metric, also known as the first fundamental form. This simply consists in defining the Riemannian metric as the restriction of the ambient inner product to the tangent space: 
\[
g_{\bm{x}} : T_{\bm{x}}\mathcal{L}^d \times T_{\bm{x}}\mathcal{L}^d \to \mathbb{R}  \:\text{s.t.} \: g_{\bm{x}}(u,v) = \langle u,v \rangle_{1,d}, \:  u,v \in T_{\bm{x}}\mathcal{L}^d
\]

The induced metric provides the natural Riemannian metric which allows us to compute lengths, angles, and volumes in hyperbolic space. From this point forward, when we talk about $\mathcal{L}^d$, we actually refer to its Riemannian form, i.e., the tuple $(\mathcal{L}^d, g)$, where $g$ is the induced metric described above. This is the \emph{hyperboloid model} of hyperbolic space.

\paragraph{Exponential Growth of Volume.}
We compute the volume of a geodesic ball and show that it grows exponentially with the radius. 
To this end, we consider the parametrization
\[
x(r,w) = (\cosh r, \sinh r \, w),
\]
where $r \geq 0$ is the hyperbolic radius and $w \in S^{d-1} \subset \mathbb{R}^d$ provides the angular direction. It is immediate to check that the image of this parametrization lies in $\mathcal{L}^d$. For a clear exponsition, let $ \theta = \{\theta_1, \ldots, \theta_{d-1}\} \in S^{d-1}$ be the local coordinates on the sphere, which implies we have a coordinate function $w = w(\theta)$. To compute the induced metric, we can just compute the metric components by considering the tangent basis in above coordinates
\[
\partial_r x = (\sinh r, \cosh r \, w), 
\qquad
\partial_\theta x = (0, \sinh r \, \partial_\theta w).
\]
The coefficients of the first fundamental form in the coordinates $(r,w)$ are given, in Gauss notation, by
\[
E = \langle \partial_r x, \partial_r x \rangle_{1,d}, \qquad
F = \langle \partial_r x, \partial_\theta x \rangle_{1,d}, \qquad
G = \langle \partial_\theta x, \partial_\theta x \rangle_{1,d}.
\]

A direct computation yields
\[
E = -\sinh^2 r + \cosh^2 r = 1,
\]
\[
F = 0,
\]
\[
G = \sinh^2 r \: \|\partial_\theta w\|^2.
\]
Note that in the angular component of the metric $G$, we have the induced metric on the unit sphere $\|\partial_\theta w\|^2 = g_{S^{d-1}}$. We can now write the metric in matrix notation as
\[
g = \begin{pmatrix}
1 & 0 \\
0 & \sinh^2(r) \, g_{S^{d-1}}
\end{pmatrix},
\]

To obtain the Riemannian volume element, we simply scale the typical volume form by the factors induced by the metric, i.e., it's determinant
\[
\det g = \sinh^{2{d-2}}(r) \:  \det(g_{S^{d-1}}),
\]
and therefore in the coordinates $(r,\theta)$ the Riemannian volume element becomes
\[
\sqrt{\det g} = \sinh^{d-1}(r)\,\sqrt{\det g_{S^{d-1}}}.
\]
\[
dV = \sqrt{\det g}\, dr\, d\theta_1 \cdots d\theta_{d-1}.
\]
Since
\[
\sqrt{\det g} = \sinh^{d-1}(r)\,\sqrt{\det g_{S^{d-1}}},
\]
and the angular part can be written as
\[
\sqrt{\det g_{S^{d-1}}}\, d\theta_1 \cdots d\theta_{d-1} = d\Omega_{d-1},
\]
i.e., the standard volume element of the unit sphere, we obtain
\[
dV = \sinh^{d-1}(r)\, dr\, d\Omega_{d-1}.
\]
Integrating the volume element over a geodesic ball of radius $R$ from the origin we obtain
\[
\mathrm{Vol}(B_R) = \int_{B_R} dV 
= \int_{S^{d-1}} \int_0^R \sinh^{d-1}(r)\, dr\, d\Omega_{d-1}.
\]
Since the angular integral yields the area of the unit sphere $S^{d-1}$,
denoted by $\omega_{d-1} = \mathrm{Area}(S^{d-1})$, we arrive at
\[
\mathrm{Vol}(B_R) = \omega_{d-1} \int_0^R \sinh^{d-1}(r)\, dr.
\]
As $r \to \infty$, by the exponential definition of $\sinh$ we have the asymptotic behavior
\[
\sinh r \sim \tfrac{1}{2} e^r,
\]
and consequently
\[
\sinh^{d-1}(r) \sim 2^{-(d-1)} e^{(d-1)r},
\]
which allows us to conclude that
\[
\mathrm{Vol}(B_R) \sim \frac{\omega_{d-1}}{(d-1)\,2^{d-1}} e^{(d-1)R}
\]

Therefore the volume of geodesic balls in $\mathcal{L}^d$ grows exponentially with rate $(d-1)$ in the radius $R$.

\subsection{Hyperbolic Graph Learning}
\label{app:hyperbolic_graph_learning}
In this section, we review the GNN message-passing mechanism and provide an overview of standard hyperbolic GNN frameworks.
\paragraph{Graph Neural Networks.}
GNNs learn node representations through iterative neighborhood aggregation, providing an inductive alternative to transductive graph embeddings. Given initial node features $\bm{x}_v^{(0)} \in \mathbb{R}^{ d'}$, representations are updated layer by layer via message passing:
\begin{align}
\bm{m}_v^{(\ell)} &= \textsc{AGG}^{(\ell)}\!\left(\{\bm{x}_u^{(\ell-1)} : u \in \mathcal{N}(v)\}\right), \\
\bm{x}_v^{(\ell)} &= \textsc{UPD}^{(\ell)}\!\left(\bm{x}_v^{(\ell-1)}, \bm{m}_v^{(\ell)}\right),
\end{align}
where $\mathcal{N}(v)$ denotes the neighbors of node $v$, and $\textsc{AGG}$ is a permutation-invariant function (e.g., sum, mean, or max). This formulation encompasses a broad class of Euclidean GNN architectures, including GCNs and GATs, which differ in their specific choices of aggregation and update functions.

\paragraph{HGNNs.}
Hyperbolic Graph Neural Networks (HGNNs) extend standard message passing to non-Euclidean geometries by learning node representations on hyperbolic manifolds. In this setting, node representations at layer $\ell$ are denoted by $\bm{z}_v^{(\ell)} \in \mathcal{H}^d_c$, where $\mathcal{H}^d_c$ is a model of $d$-dimensional hyperbolic space with curvature $-c$ (e.g.,$\mathcal{B}^d_c$ or $\mathcal{L}^d_c$). The goal of HGNNs is to perform aggregation and transformation operations that are consistent with the underlying hyperbolic geometry. Existing approaches can be broadly grouped into two design families.

\emph{Tangent-space HGNNs} perform message passing in Euclidean tangent spaces. Node representations are first mapped from the manifold to a tangent space via a logarithmic map, updated using standard GNN operations, and then mapped back to hyperbolic space via an exponential map. A representative example is HGCN \cite{chami2019hyperbolic}, which operates in the tangent space at a reference point, typically the origin $\bm{o}$:
\begin{align}
\bm{x}_v^{(\ell-1)} &= \log_{\bm{o}}\!\left(\bm{z}_v^{(\ell-1)}\right) \in T_{\bm{o}}\mathcal{H}^d_c, \\
\bm{x}_v^{(\ell)} &= \textsc{UPD}^{(\ell)}\!\left(\bm{x}_v^{(\ell-1)}, \bm{m}_v^{(\ell)}\right), \\
\bm{z}_v^{(\ell)} &= \exp_{\bm{o}}\!\left(\bm{x}_v^{(\ell)}\right) \in \mathcal{H}^d_c,
\end{align}
where the curvature $c$ is in fact a learnable hyperparameter.

\emph{Fully HGNNs} instead aim to perform aggregation, transformation, and non-linear activation directly in hyperbolic space, avoiding explicit tangent-space projections. We consider HyboNet \cite{chen2022fully} as a representative of this class. HyboNet operates in the Lorentz model and combines linear transformation and non-linearity within a single hyperbolic layer. Given aggregated features $\bm{x}_v^{(\ell)}$, obtained via an attention-weighted center of mass in Lorentz space, the update is defined as
\begin{equation}
\bm{z}_v^{(\ell)} =
\begin{bmatrix}
\sqrt{\|\phi(\bm{W}\bm{x}_v^{(\ell)}, \bm{v})\|^2 - \frac{1}{c}} \\
\frac{\phi(\bm{W}\bm{x}_v^{(\ell)}, \bm{v})}{\|\phi(\bm{W}\bm{x}_v^{(\ell)}, \bm{v})\|}
\end{bmatrix},
\end{equation}
where $\bm{W}$ and $\bm{v}$ are learnable parameters and $\phi(\cdot)$ denotes a non-linear transformation. This formulation ensures that updated representations remain on the Lorentz hyperboloid without having to use rely on the tangent space formalism.

\section{Proofs}\label{app:proofs}
For the sake of clarity and exposition, we report here both the statements and their respective proofs. We remind the reader that in all the statements below, the representation manifolds are assumed to be Riemannian submanifolds in a Euclidean ambient space, therefore $\langle \cdot, \cdot \rangle$ and $\|\cdot\|$ correspond to the dot-product and Euclidean norm, respectively. We would like to note that this setting excludes the Lorentz model, but we can make considerations about the Poincaré model.

\paragraph{Theorem \ref{prop:model_lipschitz}.} Let $\mathcal{G} \coloneqq (\mathcal{V}, \mathcal{E})$ be the input graph, $f = g \circ h$ be the model, with $h$ an embedding as in Def. \ref{def:distortion}, and $g(\bm{z}_i = h(\bm{x}_i)| \bm{w}, b) : \mathcal{M} \to \mathbb{R} \coloneqq \langle \bm{w}, \bm{z}_i \rangle + b = \hat{y}_i$ a linear solver. Define:
\begin{gather}
    \xi(h) = \min_{\substack{i,j \in \mathcal{V} \\  i \neq j}} \frac{\|\mathbf{z}_i - \mathbf{z}_j \|}{d_{\mathcal{M}}(\mathbf{z}_i, \mathbf{z}_j)}, \:
    \zeta(h) = \max_{\substack{i,j \in \mathcal{V} \\  i \neq j}} \frac{\|\mathbf{z}_i - \mathbf{z}_j \|}{d_{\mathcal{M}}(\mathbf{z}_i, \mathbf{z}_j)}, \:
    \rho(h, \bm{w}) = \min_{\substack{i,j \in \mathcal{V} \\  i \neq j}} \frac{|\langle w, \mathbf{z}_i - \mathbf{z}_j\rangle|}{\|\bm{w}\| \|\mathbf{z}_i - \mathbf{z}_j\|}.
\end{gather}
Then, if $\rho(h, \bm{w}) > 0$, $f$ is bi-Lipschitz onto its image with upper constant $\beta = \|\bm{w}\| \zeta(h)\delta_e(h)$ and lower constant $\alpha = \frac{ \|\bm{w}\| \xi(h) \rho(h,\bm{w})}{\delta_c(h)}$, such that $\alpha d_{\mathcal{G}}(i,j) \leq |\hat{y_i} - \hat{y_j}| \leq \beta d_{\mathcal{G}}(i,j)$.

\begin{proof}
For any two nodes $i,j$ where $i \neq j$ we have:
\begin{equation}
    \frac{|\hat{y_i} - \hat{y_j}|}{d_{\mathcal{G}}(i,j)} = \frac{|\hat{y_i} - \hat{y_j}|}{d_{\mathcal{M}}(h(\mathbf{x}_i),h(\mathbf{x}_j))} \frac{d_{\mathcal{M}}(h(\mathbf{x}_i),h(\mathbf{x}_j))}{d_{\mathcal{G}}(i,j)}.
\end{equation}
By definition of the expansion factor of the embedding distortion (Eq. \ref{eq:dist_factors}), we have that:
\begin{equation}
    \frac{d_{\mathcal{M}}(h(\mathbf{x}_i),h(\mathbf{x}_j))}{d_{\mathcal{G}}(i,j)} \leq \delta_e(h).
\end{equation}
Furthermore:
\begin{equation}
    \frac{|\hat{y_i} - \hat{y_j}|}{d_{\mathcal{M}}(h(\mathbf{x}_i),h(\mathbf{x}_j))} = \frac{|\langle w, h(\mathbf{x}_i)-h(\mathbf{x}_j)\rangle|}{d_{\mathcal{M}}(h(\mathbf{x}_i),h(\mathbf{x}_j))} 
    \leq \frac{\|\bm{w}\| \|h(\mathbf{x}_i)-h(\mathbf{x}_j)\|}{d_{\mathcal{M}}(h(\mathbf{x}_i),h(\mathbf{x}_j))} \leq \|\bm{w}\| \zeta(h),
\end{equation}
where the first inequality is due to Cauchy-Schwarz.

Putting together all of the above we get the upper bound:
\begin{equation}
    \frac{|\hat{y_i} - \hat{y_j}|}{d_{\mathcal{G}}(i,j)} \leq \|\bm{w}\| \zeta(h)\delta_e(h).
\end{equation}

Following a similar line of reasoning we can write:
\begin{equation}
    \frac{|\hat{y_i} - \hat{y_j}|}{d_{\mathcal{G}}(i,j)} = \frac{|\hat{y_i} - \hat{y_j}|}{d_{\mathcal{M}}(h(\mathbf{x}_i),h(\mathbf{x}_j))} \frac{d_{\mathcal{M}}(h(\mathbf{x}_i),h(\mathbf{x}_j))}{d_{\mathcal{G}}(i,j)} \geq \frac{|\hat{y_i} - \hat{y_j}|}{d_{\mathcal{M}}(h(\mathbf{x}_i),h(\mathbf{x}_j))} \frac{1}{\delta_c(h)}.
\end{equation}
By definition of the Euclidean inner product we have:
\begin{equation}
    |\hat{y_i} - \hat{y_j}|= | \langle w, h(\mathbf{x}_i) - h(\mathbf{x}_j)\rangle| = \|\bm{w}\| \: \|h(\mathbf{x}_i)-h(\mathbf{x}_j)\| \:|\cos(\theta_{ij})|
\end{equation}
where $\theta_{ij} = \angle(w, h(\mathbf{x}_i) - h(\mathbf{x}_j))$ and as such $\rho(h, \bm{w}) = \min_{\substack{i,j \in \mathcal{V} \\  i \neq j}} |\cos(\theta_{ij})|$.
Knowing this we can write:
\begin{equation}
   \frac{| \langle w, h(\mathbf{x}_i) - h(\mathbf{x}_j)\rangle|}{d_{\mathcal{M}}(h(\mathbf{x}_i),h(\mathbf{x}_j))} = \frac{\|\bm{w}\| \: \|h(\mathbf{x}_i) - h(\mathbf{x} _j)\| \: |\cos(\theta_{ij})|}{d_{\mathcal{M}} (h(\mathbf{x}_i),h(\mathbf{x}_j))} \geq  \|\bm{w}\| \xi(h) \rho(h,\bm{w}),
\end{equation}
such that we can conclude:
\begin{equation}
    \frac{|\hat{y_i} - \hat{y_j}|}{d_{\mathcal{G}}(i,j)} \geq \frac{\|\bm{w}\| \xi(h) \rho(h,\bm{w})}{\delta_c(h)}
\end{equation}
and in succession:
\begin{equation}
    \frac{\|\bm{w}\| \xi(h) \rho(h,\bm{w})}{\delta_c(h)} d_{\mathcal{G}}(i,j) \leq |\hat{y_i} - \hat{y_j}| \leq \|\bm{w}\| \zeta(h) \delta_e(h) d_{\mathcal{G}}(i,j),
\end{equation}
which completes the proof.
\end{proof}

\paragraph{Proposition \ref{prop:Euclidean_model_distortion}.}
Let $\mathcal{G}$ be the input graph, $f$ be the model, $h$ the embedding function, and $g$ the solver function, all as defined in Thm. \ref{prop:model_lipschitz}. Furthermore, let $h$ be an embedding of the graph into Euclidean space $\mathbb{R}^d$. The model distortion of $f$ is:
\begin{equation}
    k(f) = \frac{\delta(h)}{\rho(h,\bm{w})}.
\end{equation}
\begin{proof}
Given that $h$ maps into the ambient Euclidean space, we immediately get that $\xi(h) = \zeta(h) = 1$ and from the definition of model distortion we obtain
\begin{equation}
    k(f)  = \frac{\delta(h)}{\rho(h,\bm{w})}.
\end{equation}
\end{proof}

\paragraph{Proposition \ref{prop:Isometries_distortion}.}
Let $\mathcal{G}$ be the input graph, $f$ be the model, $h$ the embedding function, and $g$ the solver function, all as defined in Thm.~\ref{prop:model_lipschitz}. If $\delta(h)=1$, the model distortion of $f$ is:
\begin{equation}
    k(f) = \frac{\zeta(h)}{\xi(h)\rho(h,\bm{w})}.
\end{equation}
\begin{proof}
This result is a direct consequence of the fact that the embedding has no distortion due to it being a similarity/isometry, which implies that $\delta(h) = \delta_c(h)\delta_e(h)=1$, and from the definition of model distortion we obtain
\begin{equation}
    k(f) = \frac{\zeta(h) }{\xi(h) \rho(h,\bm{w})}.
\end{equation}
\end{proof}

We will now present two Lemmas that hold more general value, which will be used in the proof of Prop. \ref{prop:poincare_model_distortion}.

\begin{lemma} [Submanifolds with induced Euclidean metric] 
\label{prop:induced_metric}
Let $\mathcal{G}$ be the input graph, $f$ be the model, $h$ the embedding function, and $g$ the solver function, all as defined in Thm. \ref{prop:model_lipschitz}. Furthermore, let $h$ be an embedding of the graph into a Riemannian submanifold of with the induced geodesic distance, i.e., $h: (\mathcal{G}, d_{\mathcal{G}}) \to (\mathcal{M}, d_{\mathcal{M}})$ and \[d_{\mathcal{M}}(a, b) = \inf_{\gamma} \int_0^1 \|\gamma'(t)\| \: dt,\] such that $\gamma(0) = a, \gamma(1) = b$. Then the model distortion of $f$ is $k(f) = \frac{\zeta(h) \delta(h)}{\xi(h) \rho(h,\bm{w})} \leq \frac{\delta(h)}{\xi(h) \rho(h,\bm{w})}$.
\end{lemma}
\begin{proof}
Note that under the induced metric, the length functional of a curve on the manifold $\gamma: [0,1] \to \mathcal{M}$ with endpoints $\gamma(0) = h(\mathbf{x}_i), \gamma(1) = h(\mathbf{x}_j)$ satisfies:
\begin{equation}
    L(\gamma) = \int_0^1 \|\gamma'(t)\| \: dt \geq \left\Vert \: \int_0^1 \gamma'(t) \: dt \:\right\Vert = \left\Vert \gamma(1) - \gamma(0) \right\Vert = \|h(\mathbf{x}_i) - h(\mathbf{x}_j)\|,
\end{equation}
due to the triangle inequality and the fundamental theorem of calculus. The geodesic distance, represented as the minimal possible length of any curve satisfying the length functional, must also respect the above inequality, which implies that:
\begin{equation}
    \zeta(h) = \max_{\substack{i,j \in \mathcal{V} \\  i \neq j}} \frac{\|h(\mathbf{x_i}) - h(\mathbf{x_j}) \|}{d_{\mathcal{M}}(h(\mathbf{x}_i),h(\mathbf{x}_j))} \leq 1.
\end{equation}
The model distortion then reads
\begin{equation}
    k(f) = \frac{\beta}{\alpha} = \frac{\|\bm{w}\| \zeta(h)\delta_e(h) }{\frac{ \|\bm{w}\| \xi(h) \rho(h,\bm{w})}{\delta_c(h)}} = \frac{\zeta(h) \delta(h)}{\xi(h) \rho(h,\bm{w})} \leq \frac{\delta(h)}{\xi(h) \rho(h,\bm{w})}.
\end{equation}
\end{proof}

\begin{lemma}[Submanifolds with conformal metric] 
\label{prop:conformal_metric}
Let $\mathcal{G}$ be the input graph, $f$ be the model, $h$ the embedding function, and $g$ the solver function, all as defined in Thm. \ref{prop:model_lipschitz}. Furthermore, let $h$ be an embedding of the graph into a Riemannian submanifold equipped with a conformal metric, i.e., $h: (\mathcal{G}, d_{\mathcal{G}}) \to (\mathcal{M}, d_{\mathcal{M}})$ where $\mathcal{M}$ has metric $g: T_x\mathcal{M} \times T_x\mathcal{M} \to \mathbb{R}$ such that for two tangent vectors $\bm{v}, \bm{w} \in T_x\mathcal{M}$, $\mathbf{g}_x(v,w) = \lambda(p)^2 \langle v,w \rangle$. Assume that the conformal factor $\lambda$ is uniformly bounded such that $0 <\lambda_{\text{min}} \leq \lambda \leq \lambda_{\text{max}}$. Define the quantity
\begin{equation}
    \xi^{\mathcal{M}^\mathbb{E}}(h) = \min_{\substack{i,j \in \mathcal{V} i \neq j}} \frac{\|h(\mathbf{x_i}) - h(\mathbf{x_j}) \|}{d_{\mathcal{M}^\mathbb{E}}(h(\mathbf{x}_i),h(\mathbf{x}_j))}, 
\end{equation}
where $d_{\mathcal{M}^\mathbb{E}}(\cdot, \cdot)$ is the geodesic distance under the induced metric. Then, the model distortion of $f$ is upper bounded by
\begin{equation}
    k(f) \leq \frac{\lambda_{\text{max}}\delta(h)}{\lambda_{\text{min}} \xi^{\mathcal{M}^\mathbb{E}}(h) \rho(h,\bm{w})}.
\end{equation}
\end{lemma}
\begin{proof}
The length functional of a curve $\gamma: [0,1] \to \mathcal{M}$ with endpoints $\gamma(0) = h(\mathbf{x_i}), \gamma(1) = h(\mathbf{x_j})$, on a manifold with metric $\mathbf{g}$ is defined as: 
\begin{equation}
    L_{\mathbf{g}}(\gamma) = \int_0^1 \sqrt{\mathbf{g}_{\gamma(t)} (\gamma'(t),\gamma'(t))} \: dt.
\end{equation}
When the metric takes the conformal form described above, it reads as:
\begin{equation}
    L_{\mathbf{g}}(\gamma) = \int_0^1 \lambda(\gamma(t)) \|\gamma'(t)\| \: dt,
\end{equation}
which indeed confirms that distances under a conformal metric are (appropriately) scaled counterparts of the distances under the induced Euclidean metric (remember that conformal in this scenario means that angles are preserved, while distances are scaled). Given that the conformal factor is positive, using the bound $\lambda_{\text{min}} \leq \lambda \leq \lambda_{\text{max}}$ pointwise gives 
\begin{equation}
    \lambda_{\text{min}} L_{\mathbf{g}^\mathbb{E}} (\gamma) \leq L_{\mathbf{g}}(\gamma) \leq \lambda_{\text{max}} L_{\mathbf{g}^\mathbb{E}} (\gamma),
\end{equation}
where $L_{\mathbf{g}^\mathbb{E}}$ is the length functional for a submanifold with the induced Euclidean metric, as described in the proof of Prop. \ref{prop:induced_metric}. The geodesic distance is the one corresponding to the curve that minimizes the length function, and thus taking the infimum over curves will respect monotonicity such that we can arrive at:
\begin{equation}
    \lambda_{\text{min}} d_{\mathcal{M}^{\mathbb{E}}} (h(\mathbf{x_i}), h(\mathbf{x_j}))  \leq d_{\mathcal{M}}(h(\mathbf{x_i}), h(\mathbf{x_j})) \leq \lambda_{\text{max}} d_{\mathcal{M}^{\mathbb{E}}} (h(\mathbf{x_i}), h(\mathbf{x_j})),
\end{equation}
where $d_{\mathcal{M}^{\mathbb{E}}}(\cdot, \cdot)$ denotes the distance under the induced Euclidean metric. We can then write:
\begin{gather}
    \zeta(h) = \max_{\substack{i,j \in \mathcal{V} \\  i \neq j}} \frac{\|h(\mathbf{x_i}) - h(\mathbf{x_j}) \|}{d_{\mathcal{M}}(h(\mathbf{x}_i),h(\mathbf{x}_j))} \leq \max_{\substack{i,j \in \mathcal{V} \\  i \neq j}} \frac{\|h(\mathbf{x_i}) - h(\mathbf{x_j}) \|}{\lambda_{\text{min}}  d_{\mathcal{M}^\mathbb{E}}(h(\mathbf{x}_i),h(\mathbf{x}_j))} \leq \frac{1}{\lambda_{\text{min}}}, \\
    \xi(h) = \min_{\substack{i,j \in \mathcal{V} \\ i \neq j}} \frac{\|h(\mathbf{x_i}) - h(\mathbf{x_j}) \|}{d_{\mathcal{M}}(h(\mathbf{x}_i),h(\mathbf{x}_j))} \geq \min_{\substack{i,j \in \mathcal{V} \\ i \neq j}} \frac{\|h(\mathbf{x_i}) - h(\mathbf{x_j}) \|}{\lambda_{\text{max}}  d_{\mathcal{M}^\mathbb{E}} (h(\mathbf{x}_i),h(\mathbf{x}_j))} \geq \frac{1}{\lambda_{\text{max}}} \xi^{\mathcal{M}^\mathbb{E}}(h).
\end{gather}

Using all this information, we can upper bound the model distortion according as redefining both the upper and lower constants
\begin{equation}
    k(f) = \frac{\beta}{\alpha} = \frac{\|\bm{w}\| \zeta(h)\delta_e(h) }{\frac{ \|\bm{w}\| \xi(h) \rho(h,\bm{w})}{\delta_c(h)}} \leq \frac{\lambda_{\text{max}}\delta(h)}{\lambda_{\text{min}} \xi^{\mathcal{M}^\mathbb{E}}(h) \rho(h,\bm{w})}.
\end{equation}
Unsurprisingly, the result is similar to the case for submanifolds with induced metric in Prop. \ref{prop:induced_metric}, but rescaled based on the conformal bounds.
\end{proof}

\paragraph{Proposition \ref{prop:poincare_model_distortion}.}
Let $\mathcal{G}$ be the input graph, $f$ be the model, $h$ the embedding function, and $g$ the solver function, all as defined in Thm. \ref{prop:model_lipschitz}. Let $h$ be an embedding function that embeds into the Poincaré ball $\mathcal{B}^d_{1}$ and define $\eta := \max_{i \in \mathcal{V}} \|h(\mathbf{x}_i)\|$. Then, the model distortion of $f$ is upper bounded by:
\begin{equation}
    k(f) \leq \frac{\delta(h)}{(1-\eta^2)\rho(h,\bm{w})}.
\end{equation}
\begin{proof}
We start by noting that, given the domain of definition of $\mathcal{B}^d_1$ we have:
\begin{equation}
    \lambda_{\text{min}} = \lim_{\|p\| \to 0} \lambda(p) = \lim_{\|p\| \to 0} \frac{2}{1 - \|p\|^2} = 2.
\end{equation}

Following Lemma \ref{prop:conformal_metric} we can see that for any piecewise $C^1$ curve $\gamma:[0,1]\to\mathcal{B}^d_1$ with $\gamma(0)=\bm h(\mathbf{x}_i)$ and $\gamma(1)=h(\mathbf{x}_j)$ the length functional satisfies:
\begin{align}
    L_{\mathbf g}(\gamma)
    &= \int_0^1 \lambda(\gamma(t))\|\gamma'(t)\|\,dt \\
    &\geq 2\int_0^1 \|\gamma'(t)\|\,dt \\
    &\geq 2 \: \|h(\mathbf{x}_i) - h(\mathbf{x}_j)\|,
\end{align}
where the last inequality follows because the length of any curve joining $h(\mathbf{x}_i)$ and $h(\mathbf{x}_j)$ must be at least the Euclidean chord length. Since this holds for every admissible curve $\gamma$, taking the infimum over curves gives $ d_{\mathcal{B}}(h(\mathbf{x}_i),h(\mathbf{x}_j)) \geq 2 \: \|h(\mathbf{x}_i) - h(\mathbf{x}_j) \|.$

To concretely bound the maximal conformal stretch factor we first proceed by making a direct comparison of the difference in the length of a chord connecting any two points within the ball and it's length under the manifold metric. Given the \emph{convexity of the domain}, we can write any chord connecting points $h(\mathbf{x}_i)$ and $h(\mathbf{x}_j)$ as:
\begin{equation}
    \gamma(t) = th(\mathbf{x_i}) + (1-t)h(\mathbf{x_j}),
\end{equation}
for $t \in (0,1)$. We can then deduce the following two facts:
\begin{gather}
    \| \gamma(t) \| \leq \|th(\mathbf{x_i})\| + \|(1-t)h(\mathbf{x_j})\| \leq \max\{\|h(\mathbf{x_i})\|, \|h(\mathbf{x_j})\|\}\\
    \| \gamma'(t) \| =  \| h(\mathbf{x_i}) - h(\mathbf{x_j})\|,
\end{gather}
where in the first equation we use the triangle inequality and the second implies a constant speed. Letting $c_{ij} = \max\{\| h(\mathbf{x}_i)\|, \|h(\mathbf{x}_j)\|\}$, the first inequality above implies that $1 - \| \gamma(t) \|^2 \geq 1 - c_{ij}^2$ and therefore we obtain:
\begin{equation}
\label{ineq:conformal_curve}
    \frac{2}{1 - \|\gamma(t)\|^2} \leq \frac{2}{1 - c_{ij}^2}
\end{equation}
We can now use the fact that $\eta = \max_i \|h(\mathbf{x}_i)\|$ and the formula for the length functional under the conformal metric, as presented in the Lemma Prop. \ref{prop:conformal_metric}, such that:
\begin{gather}
    L_{\mathbf{g}}(\gamma) = \int_0^1 \lambda(\gamma(t)) \|\gamma'(t)\| \: dt \leq \int_0^1 \frac{2\| h(\mathbf{x}_i) - h(\mathbf{x}_j) \|}{1 - c_{ij}^2} \: dt \\
    = \frac{2\| h(\mathbf{x}_i) - h(\mathbf{x}_j) \|}{1 - c_{ij}^2} \leq \frac{2\| h(\mathbf{x}_i) - h(\mathbf{x}_j) \|}{1 - \eta^2}.
\end{gather}
Since the geodesic distance is the infimum over all curves, we obtain the desired upper bound $d_{\mathcal{B}}(h(\mathbf{x}_i),h(\mathbf{x}_j)) \leq \frac{2}{1 - \eta^2} \: \|h(\mathbf{x}_i) - h(\mathbf{x}_j) \|$. Using the language of Lemma \ref{prop:conformal_metric} and Thm. \ref{prop:model_lipschitz}, the above translates to $\zeta(h) \leq \frac{1}{2}$ and $\xi(h) \geq \frac{1 - \eta^2}{2}$, which in turn implies that we can redefine the upper and lower constants as $\beta = \frac{1}{2}\|\bm{w}\| \delta_e(h)$ and $\alpha = \frac{(1 - \eta^2)\|\bm{w}\| \rho(h,\bm{w})}{2\delta_c(h)}$. By the definition of model distortion we obtain:
\begin{equation}
    k(f) \leq \frac{\delta(h)}{(1-\eta^2)\rho(h,\bm{w})}.
\end{equation}

We conclude by noting that the final result is essentially identical for a given curvature magnitude $c>0$ following the same arguments given that $\lambda_c(p)=\frac{2}{1-c\|p\|^2}$ on $\mathcal{B}^d_c=\{p:\|p\|^2<1/c\}$ yielding $k(f) \leq \frac{\delta(h)}{(1-c\eta^2)\rho(h,\bm{w})}$.
\end{proof}

\subsection{An initial analysis of distance-based LP decoders}
\label{app:dist_lp_decoders}
We consider the following setup a starting point for analysis on distance-based LP decoders. Consider a connected, simple, unweighted input graph $\mathcal{G}$ represented via an adjacency matrix $\mathbf{A}$ and $f = g \circ h$ be the model, with $h$ the embedding function as defined in Def. \ref{def:distortion}, and $g$ a distance-based decoder s.t.:
\[ \widehat{\mathbf{A}}_{ij} = g(h(\mathbf{x_i}),h(\mathbf{x_j})) = \mathbf{1} \left[d_{\mathcal M}(h(i),h(j))<r \right], \quad i\ne j, \]

What our problem then dictates for each pair $(i,j) \in \mathcal{V} \times \mathcal{V}, i\ne j$ can be separated into two cases: (i) If $\mathbf{A}_{ij} = 1$ we would like our model to satisfy $d_{\mathcal{M}}(h(\mathbf{x}_i),h(\mathbf{x}_j)) < r \implies \delta_e(h) < r$. On the other hand for a non-edge $\mathbf{A}_{ij} = 0$ we want $d_{\mathcal{M}}(h(\mathbf{x}_i),h(\mathbf{x}_j)) \geq r$. In this case, we also have that $d_{\mathcal{G}}(i,j) \geq 2 \implies d_{\mathcal{M}}(h(\mathbf{x}_i),h(\mathbf{x}_j)) \geq \frac{2}{\delta_c(h)} \geq r$. Since the edge case implies $\delta_e(h) < r$ we obtain that $\delta_e(h) < \frac{2}{\delta_c(h)}$, which directly leads us to a sufficient condition for the recovery of the input adjacency matrix: $\delta(h) < 2$. If this condition is satisfied then we are guaranteed the existence of a radius $r^* \in \left(\delta_e(h), \frac{2}{\delta_c(h)} \right]$ s.t. $\widehat{\mathbf{A}}_{ij} = \mathbf{A}_{ij}$.

We now proceed to show how this condition manifests in a practical scenario, which is a probabilistic distance-based decoder \[g(\mathbf{z_i},\mathbf{z_j})=2\sigma\left(-d_{\mathcal{M}}(h_i,h_j)\right), \]
where sigma is the sigmoid function $\sigma(x) = \frac{1}{1 + e^{-x}}$. We consider the typical classification threshold $p_{ij}>1/2$. Note that the above decoder correctly encodes the intuition that as $d_{\mathcal{M}} \to \inf$ then $p_{ij} \to 0$ and as $d_{\mathcal{M}} \to 0$ then $p_{ij} \to 1$. We can then write down the decision rule to find our decision radius $r$:
\[
\frac{2}{1 + e^{d_{\mathcal{M}}(\mathbf{z_i},\mathbf{z_j})}}
> \frac{1}{2} \implies d_{\mathcal M}(h_i,h_j)<\log 3.\]

Following our previous line of reasoning (using the expansion and contraction factors from Definition 3.1 and the fact that for any edge $d_G(i,j)=1$), a correct classification will happen when $ \delta_e(h)<\log 3$). For a non-edge between distinct nodes we then obtain $d_{\mathcal M}(h_i,h_j) \geq \log 3$, and since $d_G(i,j)\ge 2$, it is classified negatively whenever:
\[ \frac{2}{\delta_c(h)}\ge \log 3, \qquad\text{or equivalently}\qquad \delta_c(h)\le \frac{2}{\log 3}. \]

Thus these two conditions are sufficient for exact link recovery over distinct pairs (edges). We can clearly see that the sufficient condition is in fact $\delta(h) < 2$. 

Unlike Theorem 4.1, this statement depends only on intrinsic latent distances; there is no ambient-coordinate or readout-direction term. It therefore applies to either Poincare- or Lorentz-model embeddings when the decoder uses the corresponding geodesic distance. Finally, we get that the predictive ability can rely entirely on the embedding distortion, which reveals why we consider LP to be inherently more geometry-aligned in the paper. For masked-edge link prediction, the previous statements are about alignment with the target graph metric. Whether that metric can be inferred from the observed training graph is a separate learnability question.

\section{Gromov $\delta$-hyperbolicity}
\label{app:gromov}
We recall the definition of $\delta$-Gromov hyperbolicity based on the four-point condition, originally introduced by Gromov as a purely metric characterization of negative curvature \cite{gromov1987hyperbolic, bonk2011embeddings}. Trees are prototypical examples of hyperbolic metric spaces. Accordingly, the four-point condition quantifies the extent to which a metric space exhibits tree-like behavior.

\begin{definition}[Gromov hyperbolicity via the four-point condition]
A metric space $(M,d)$ is said to be $\delta$-hyperbolic, if for some $\delta \geq 0$, it satisfies the \emph{four-point condition}: for any $x,y,z,t \in M$, define
\[
S_1 = d(x,y) + d(z,t), \quad
S_2 = d(x,z) + d(y,t), \quad
S_3 = d(x,t) + d(y,z).
\]
Then the two largest values among $\{S_1,S_2,S_3\}$ differ by at most $2\delta$.
The \emph{Gromov hyperbolicity constant} $\delta^*$ of $(M,d)$ is the smallest $\delta^* \geq 0$ such that $(M,d)$ is $\delta^*$-hyperbolic.
\end{definition}

This formulation is particularly appealing from a computational perspective, as the hyperbolicity of a finite metric space can be evaluated directly by checking the four-point condition over all quadruples of points. 

\paragraph{Real-world datasets.} We recomputed  Gromov $\delta$-hyperbolicity for all datasets used, using an exact algorithm from the SageMath library \cite{stein2005sage} (BCCM).
This is different from previous papers that only provided results from an appoximate algorithm, that approximates $\delta$ by sampling quadruples of nodes.
For that reason there are some differences in the values reported.
In particular \citet{chami2019hyperbolic} reported $\delta=1$ for Airport, $\delta=3.5$ for PubMed and $\delta=11$ for Cora using the approximate algorithm, while the exact algorithm yields $\delta=2$ for Airport, $\delta=4.5$ for PubMed and $\delta=4$ for Cora. 

Because of the importance of Gromov $\delta$-hyperbolicity as a decision criterion in the context of HGNNs, we recomputed it for many well-known real-world datasets, which can be reviewed in Table~\ref{tab:delta_hyperbolicity}. 
Again, all values are computed using SageMath \cite{stein2005sage} (BCCM algorithm) on the largest connected components. 
In case of multiple components, the smaller components achieved a smaller or equal delta (e.g., on Cora and Airport). 
Most real-world dataset exhibit small delta values, with the exception of Minesweeper which is a synthetic dataset on a $50\times50$ grid graph. 
\jonas{The results on the heterophilic graphs span  a range of $\delta=2.5$ to $\delta=49$, showing that heterophily alone does not identify a specific geometric regime and therefore does not directly test our hypothesis.}
Exact computation for larger graphs, like the OGBN graphs \cite{hu2020ogb}, \jonas{or large heterophilic social networks such as Penn94 and Twitch-Gamers \cite{lim2021large},} is usually infeasible with an exact algorithm and in particular impossible in SageMath due to the limit of $65536$ nodes and memory limitations. 
In these cases sampling-based methods are the only option and therefore no results are reported here.
We would like to note however that computing or even only finding approximations to $\delta$-hyperbolicity on large graphs is a research topic in discrete mathematics on its own \citep{cohen2012exact, coudert2021hyperbolicity}.

\begin{table}[t]
\caption{Exact $\delta$-hyperbolicity values of many well-known real-world datasets. Values are computed using SageMath \cite{stein2005sage} (BCCM algorithm) on the largest connected components. In case of multiple components, the smaller components achieved a smaller or equal delta (e.g. on Cora and Airport). On PPI all subgraphs achieved the same $\delta=2$. Most real-world datasets exhibit small delta values, with the exception of Minesweeper which is a synthetic dataset on a $50\times50$ grid graph.}
\centering
\small
\begin{tabular}{llr}
\hline
\textbf{Category} & \textbf{Dataset} & $\boldsymbol{\delta}$ \\
\hline
\multirow{3}{*}{Planetoid Citation Graphs \cite{sen2008collective}} 
  & Cora        & 4.0 \\
  & Citeseer    & 6.5 \\
  & Pubmed      & 4.5 \\
\hline
\multirow{3}{*}{Hyperbolic GNN Benchmarks \cite{chami2019hyperbolic, greene2015understanding}}
  & Disease     & 0.0 \\
  & Airport     & 2.0 \\
  & PPI         & 2.0 \\
\hline
\multirow{2}{*}{Small-scale Social Networks \cite{rozemberczki2020characteristic}}
  & Facebook Page-Page & 4.0 \\
  & LastFM Asia        & 3.5 \\
\hline
\multirow{5}{*}{Heterophilous Node Classification Benchmarks \cite{platonov2023a}}
  & Roman-Empire   & 3.5 \\
  & Amazon-Ratings & 11.0 \\
  & Minesweeper    & 49.0 \\
  & Questions      & 3.5 \\
  & Tolokers       & 2.5 \\
\hline
\multirow{2}{*}{Product Co-purchase Graphs \cite{mcauley2015image}}
  & Amazon Computers & 2.5 \\
  & Amazon Photo     & 3.0 \\
\hline
\multirow{2}{*}{Wikipedia Networks \cite{rozemberczki2021multi}}
  & Chameleon & 2.5 \\
  & Squirrel  & 2.5 \\
\hline
Actor Co-occurrence Network \cite{tang2009social}
  & Actor & 3.0 \\
\hline
\multirow{3}{*}{Air Transportation Networks \cite{ribeiro2017struc2vec}}
  & USA     & 1.5 \\
  & Brazil  & 1.0 \\
  & Europe  & 1.0 \\
\hline
\multirow{3}{*}{University Web Pages (WebKB) \cite{pei2020geom}}
  & Cornell   & 2.0 \\
  & Texas     & 1.5 \\
  & Wisconsin & 2.0 \\
\hline
\multirow{2}{*}{Co-authorship Networks \cite{shchur2019pitfallsgraphneuralnetwork}}
  & Coauthor-CS      & 4.0 \\
  & Coauthor-Physics & 4.0 \\
\hline
\end{tabular}

\label{tab:delta_hyperbolicity}
\end{table}

\newpage
\section{Experimental Setup}\label{app:expset}
Here, we summarize the evaluation protocol for each task and synthetic or real-world dataset.

\subsection{Tasks and Real-world Datasets}
\paragraph{Tasks} On synthetic datasets, we consider all four tasks:  
(i) \textbf{Pairwise Distance Prediction (PDP)}, where the goal is to regress the shortest-path distance between pairs of nodes;  
(ii) \textbf{Node Regression}, where node-level continuous targets (distance to anchor node) are predicted;
(iii) \textbf{Node Classification}, where node-level discrete targets (distance to anchor node as a class) are predicted; and  
(iv) \textbf{Link Prediction}, formulated as a binary classification task.  

For LP tasks, in all experiments, we randomly split edges into $85/5/10\%$ for training, validation and test set, respectively. 
Negative edges are sampled uniformly using PyG's negative edge sampling utility.

\nips{\paragraph{Real-world Datasets. } We use well-established benchmarks from the hyperbolic GNN literature \cite{chami2019hyperbolic, zhang2021lorentzian, yang2024hypformer}, including Disease, Airport, Cora, Pubmed, and Citeseer, which exhibit varying levels of $\delta$-hyperbolicity. They cover different kind of real-world networks:
\begin{itemize}
    \item \textbf{Cora / Pubmed / Citeseer (citation networks)}\\
    Nodes represent papers with bag-of-words representations of text as features, edges are citations, and the node labels are the research topic of each paper. 
    \item \textbf{Airport (flight network)}\\
    Nodes are airports, edges are airline routes. The features contain geographic information (longitude, latitude and altitude) and the GDP of the airport’s country. The node label is the country-level population (discretized into classes).
    \item \textbf{Disease LP/NC (Disease propagation trees)} \\
    Nodes are individuals in a tree generated via an SIR process. An SIR process is a Disease-spread model in which individuals transition from \textbf{s}usceptible to \textbf{i}nfected to \textbf{r}ecovered as the Disease propagates through a population. Node features indicate the susceptibility to Disease and node labels indicate infected vs. not infected.
\end{itemize} }

On these datasets, we focus on \textbf{Link Prediction (LP)} and \textbf{Node Classification (NC)} tasks, following standard dataset splits and evaluation protocols used in prior work.
For NC, we use a corrected $70/15/15\%$ split for both Airport and Disease, and standard splits \cite{kipf2017semi} with $20$ train examples per class, $500$ overall for validation and the rest as test set for Cora, PubMed and Citeseer. 

For each dataset, we generate $10$ independent data splits on which we perform the experimental runs. Each splits also has a different random seed for initializing all random operations (like weight initialization). 
All reported results are averaged over these $10$ splits, and means and standard deviations are reported.

\paragraph{Disease NC data split bug.} There was a bug in the original data splitting function on Disease that resulted in a $71.6/4/24.3\%$ split instead of the reported $30/10/60\%$. 
Additionally the dataset is very unbalanced (of the overall nodes $20\%$ are label $1$ and $80\%$ are label $0$), but using the original splitting function this was not reflected on the train set, because the test and validation sets had a fixed $50\%/50\%$ ratio of the two labels. 
This led to even less training samples of the minority class, changing the training dynamics significantly compared to all other NC experiments.
In our experiments we refer to the corrected version by Disease*, where we split the nodes using a standard $70/15/15\%$ split, keeping the ratio of the two classes the same in all splits.
For full transparency we also report the results using the original splitting ratio with a slightly adapted $70/5/25\%$ splitting (compared to the original $71.6/4/24.3\%$ split) in Appendix~\ref{app: node classification}.

\paragraph{Airport data leakage.}
The use of the same splitting function that was used on Disease on the Airport dataset, resulted in data leakage as it was built for binary classification. 
Applying it to the $4$-class classification of the Airport dataset resulted in leakage of validation and test labels to the train set.
We fixed this problem by using PyG's standard random node splitting function and only report the results on the corrected splits. 

\subsection{Synthetic Datasets} 
For the \textbf{PDP task}, we use two synthetic graphs: a tree graph and a grid graph. For each split, node pairs are randomly partitioned using a $70/15/15\%$ split.
In both cases, the nodes are equipped with $100$-dimensional random (Gaussian) and sparse ($10\%$) features. 
We chose these features for all our synthetic tasks, to isolate the effect of geometry, but to be still close to real-world features that are often high-dimensional but sparse.

For \textbf{Node Regression} and \textbf{Node Classification}, we use a synthetic tree with depth $h=4$ and a varying branching factor of $b=100,2,2,2$. 
This means after root the tree has four levels, with increasing number of nodes: $100, 200, 400, 800$.
The motivation for this was to work in a tree setting that is not extremely unbalanced w.r.t. to the numbers of nodes per level.
The nodes are equipped with $100$-dimensional random (Gaussian) and sparse ($10\%$) features. 
For regression, the anchor node is set as the root, therefore the label is the distance to the root node (values $1$ to $4$).
For classification these four values are the four classes.
In both cases we use splits with $20$ train examples per class, $500$ overall for validation and the rest as test set as for Cora etc.
Since only HGCN is not perfectly reproducible for NR with $d=3$, we additionally run the experiments on the usual $10$ splits, ten times and report the average over these in order to be as close as possible to full reproducibility. 

For \textbf{Link Prediction}, we use the $\textit{Tree1111}_{\gamma}$ synthetic dataset introduced in \cite{katsman2025shedding} with the standard $85/5/10\%$.
In this dataset, the graph is a branching factor $10$ tree with height $4$ and the features are generated procedurally, level-by-level starting from $x^{(0)}\sim\mathcal{N}(0,I_{1000})$ and then $x^{(n)}=\gamma\cdot \operatorname{parent(x^{(n)})+\nu}$, where $\nu\sim\mathcal{N}(0,I_{1000})$. 
Now, when $\gamma = 0$, node features are i.i.d. samples from the base distribution and contain no graph information; therefore, any nontrivial performance must come from the model’s ability to exploit the graph. 
As $\gamma$ increases, features progressively “leak” structural information. 
In this sense, it provides a controlled setting to evaluate whether and to what extent the graph structure itself is actually necessary to solve the task.
\citet{katsman2025shedding} already showed results for the MLP on a range of $\gamma$-values between $0$ and $1.0$ and some advantage of using HyboNet for $\gamma=0$. 
We provide results over all values of $\gamma$ using all our GNN models and the MLP.

Additionally, as described in the main text, we experiment with versions of Disease and Airport, where we corrupt the features by Gaussian noise of varying standard deviation.

\subsection{Models}
Our experimental comparison includes:
\begin{itemize}
    \item \textbf{HGNNs}, operating either fully in hyperbolic space, such as HyboNet~\cite{chen2022fully}, or using tangent-space-based methods, such as HGCN~\cite{chami2019hyperbolic};
    \item \textbf{Euclidean GNNs}, operating in standard Euclidean space, including GCN~\cite{kipf2017semi}, GAT~\cite{velickovic2018gat}, and a Euclidean variant of HyboNet~\cite{chen2022fully} that we denote by HyboNet (Eucl.);
    \item \textbf{MLP} that does not consider the graph structure and thus operates solely on the node features (using Euclidean Neural Network operations).
\end{itemize}

For hyperbolic models, node features are first mapped from Euclidean space to hyperbolic space using the exponential map at the origin. All subsequent message-passing and prediction operations are performed within the chosen hyperbolic manifold. Euclidean counterparts follow identical architectural choices where applicable, differing only in the underlying geometry.

\subsection{Hyperparameter Tuning and Control of Confounding Factors}
\label{app:tuning}
\nips{To minimize confounding effects from architecture, decoder choice, optimization, and capacity, we retune all models, including Euclidean and hyperbolic baselines, under the same hyperparameter-search budget for each model–dataset–task combination, selecting configurations solely by validation performance. For the PDP experiments, where training is slower, we use a two-stage search consisting of $150$ random configurations over a broad grid. We then select the top 10 configurations based on validation performance and perform a focused grid search around their hyperparameters to obtain a refined set of configurations. For LP, NC, and NR, we instead sample $1000$ random hyperparameter configurations per model–dataset–task combination. These configurations are evaluated sequentially on two tuning splits using the task-specific validation criterion, and the selected configuration is then evaluated on the full set of $10$ independent splits. The validation criterion is stress loss for PDP, the average of ROC-AUC and AP for LP, Macro-F1 for NC, and MAE for NR. The search space is shared across models whenever applicable, including learning rate, dropout, weight decay, activation function, normalization, gradient clipping, early stopping, hidden dimension, number of message-passing layers, and, when applicable, the use of attention mechanisms, with only model-specific parameters treated separately.\\\\
As is standard in the literature, all models are trained for a minimum of $100$ epochs and a maximum of $5000$ epochs, with early stopping applied if the validation criterion does not improve for $500$ consecutive epochs. Hyperbolic models are optimized using the Riemannian Adam optimizer~\cite{becigneul2018riemannian}, while Euclidean models use the standard Adam optimizer~\cite{kingma2014adam}.\\
This careful tuning protocol is crucial because, for GNNs, larger models are not necessarily better: increasing depth or hidden dimensionality can lead to oversmoothing, optimization instability, or unnecessary capacity. Consequently, fixing hidden dimension or depth across all methods may itself bias comparisons, especially in the hyperbolic graph learning literature, where Euclidean baselines are often evaluated in undertuned regimes while hyperbolic models are favored by low-dimensional settings. Our protocol avoids this issue by allowing each model family to select its best configuration under an identical tuning budget and by evaluating the final choice across multiple splits rather than relying on a single favorable run. All reported results correspond to the configuration achieving the best validation performance, and experiments were conducted on NVIDIA A100 and H100 GPUs. Full hyperparameter configurations are provided in the accompanying code repository. Generally, the experiments run in minutes and require only a few GB of VRAM, except for PubMed, which is batched to fit within the VRAM for LP. For PDP, experiments take hours and may also use the full available VRAM. \\\\
To further isolate the effect of geometry, we keep the solver/decoder as simple and comparable as possible across methods. PDP uses a distance-based decoder by design, and LP employs the same functional form across all models, with distances computed in the corresponding latent manifold. Thus, the decoder depends on the learned metric but does not introduce additional model-specific expressivity. For NC and NR, we use GCN and GAT decoders for the corresponding Euclidean models, while the remaining architectures employ either linear decoders or fully hyperbolic decoders, as in the case of HyboNet. Curvature is handled according to the original implementations: HGCN uses learnable curvature, whereas HyboNet uses fixed curvature, which we do not tune except in a dedicated ablation. Finally, where applicable, we include Euclidean counterparts of hyperbolic architectures (e.g., HyboNet (Eucl.)), ensuring that message-passing structure and capacity remain fixed while only the representation geometry varies. Overall, this controlled protocol aims to ensure that performance differences are not driven by tuning artifacts, architectural mismatches, or biased capacity choices, but instead reflect the contribution of the geometric inductive bias.}

\section{Additional Results}

\subsection{Pairwise Distance Prediction Task}

\begin{figure}[h!]
  \begin{subfigure}{0.45\textwidth}
    \includegraphics[width=\linewidth]{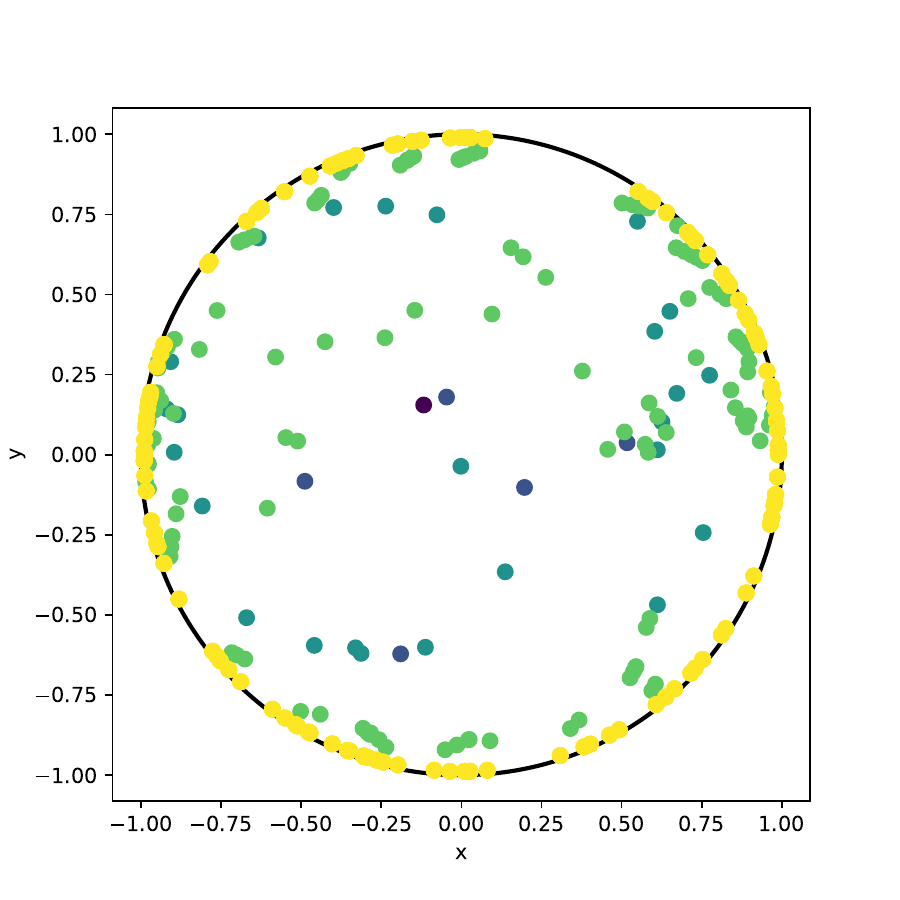}
    \caption{HyboNet, $d=3$}
  \end{subfigure}\hfill
  \begin{subfigure}{0.60\textwidth}
    \includegraphics[width=\linewidth]{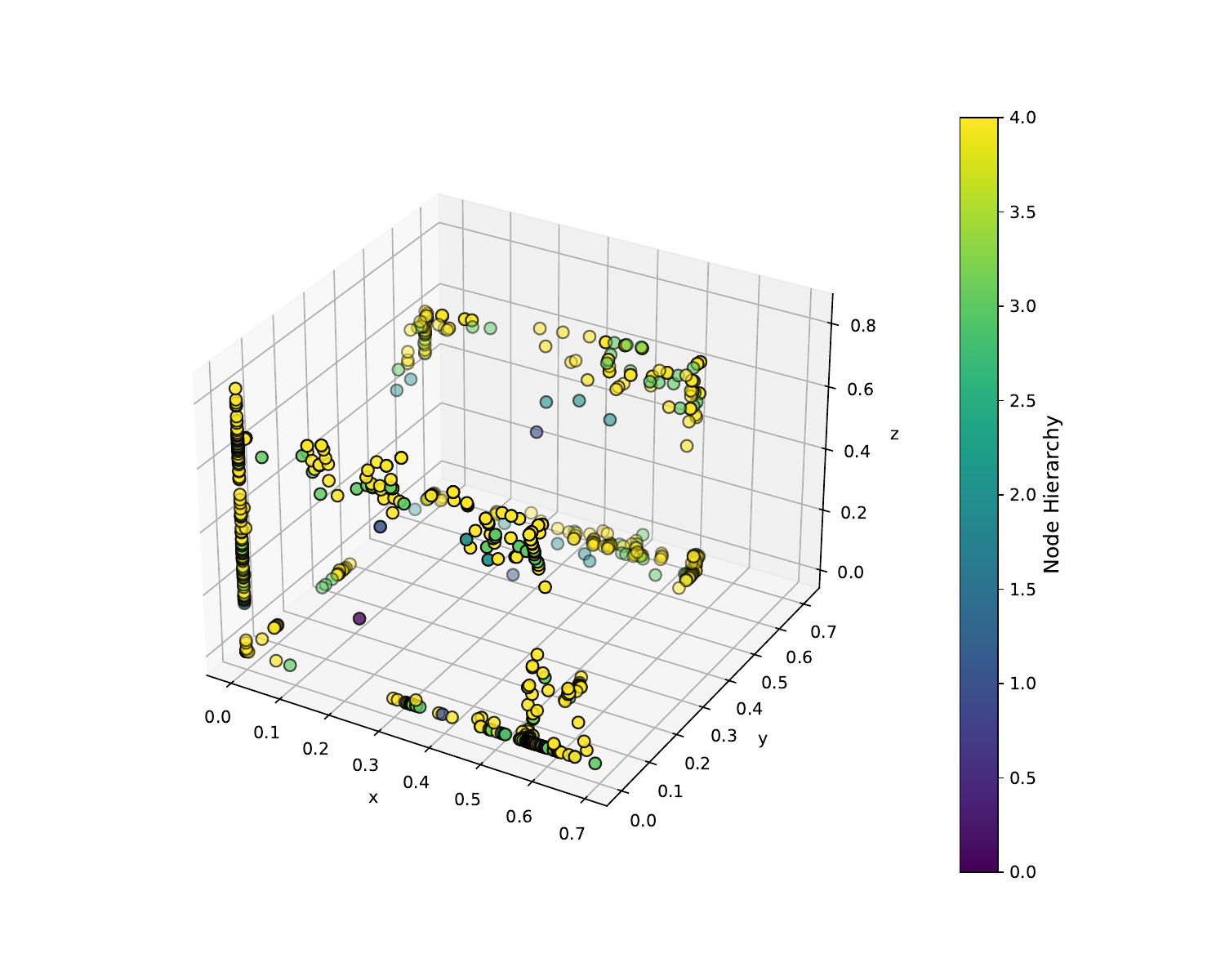}
    \caption{GCN, $d=3$}
  \end{subfigure}

  \caption{Visualization of the embedding spaces learned by HyboNet and GCN models for $d=3$ for the PDP task on the \textit{tree}.}
  \label{fig:embed}
\end{figure}

\label{appendix:pairs_task}
We report the results of the Pairwise Distance Prediction (PDP) task in Tables~\ref{tab:stress_results_tree} and~\ref{tab:stress_results_grid}, evaluated across varying embedding dimensions on both the \textit{tree} and \textit{grid} graphs, and corresponding to Fig.~\ref{fig:stress_vs_embdim}. Figure~\ref{fig:embed} shows the learned embeddings of HyboNet and GCN for the \textit{tree} graph with embedding dimension $d=3$. For HyboNet, the points form $3$D Lorentz space are mapped to the Poincaré disc. Here, node embeddings of close to the root nodes are located closer to the center and further away from the root and leaf nodes positioned toward the boundary, reflecting the underlying hierarchical structure of the graph. In contrast, GCN embeddings tend to occupy a roughly cubic region of the Euclidean space, suggesting an increased geometric distortion of pairwise distances.

\dionisia{Table~\ref{tab:pdp_learned_curvature} reports the final-layer curvature magnitudes $c$ learned by HGCN on the PDP task, where the manifold curvature is $-c$. On the grid, the learned magnitude decreases from $0.3829$ at $d=3$ to $0.0262$ at $d=8$ and remains below $0.005$ for $d\geq16$. In contrast, the tree yields substantially nonzero curvature magnitudes across all tested dimensions. These results support
the interpretation that HGCN adapts toward an approximately Euclidean final-layer geometry on the grid, while retaining negative curvature on the tree.}

The PDP results can also be viewed from another, more qualitative perspective by explicitly confronting pairwise distances in data versus embedding space. 
To do this we plot the averaged pairwise distances (and the standard deviation) in embedding space conditioned on the true pairwise graph distances. 
Then we perform a linear fit on these points in the plot and report the $R^2$ score. 
The $R^2$ score or the coefficient of determination, measures how much of the variation in the dependent variable (average pairwise distance in embeddings space) is explained by the linear model (which is the true underlying relation here). 
An $R^2$ close to $1$ means the points lie close to the fitted line, while an $R^2$ near $0$ means the line explains little of the data’s variability.
So for us, it measures how faithfully the embedding preserves the graph metric on average. 
Results comparing HyboNet and Euclidean HyboNet are reported in Fig.\ref{fig:r2_regression}. 
At low embedding dimension, HyboNet achieves the highest $R^2$ on the \textit{tree} graph (Table~\ref{table:r2_tree_grid}). 
As the embedding dimension increases, the performance gap progressively narrows and eventually vanishes.

\dionisia{Table~\ref{table:null_tree_grid} provides the full results of the distance-permutation control discussed in Section~\ref{sec:pairs_distortion_pdp}. After permutation, test stress is broadly similar across models at both embedding dimensions: approximately $0.15$--$0.16$ on the tree and $0.53$--$0.58$ on the grid. The hyperbolic advantage observed with geometry-aligned supervision is therefore
absent in this control.}

Overall, these results show that when supervision explicitly depends on pairwise distances, message passing can exploit the geometric inductive bias of the latent space to recover low-distortion structure, especially in low-dimensional regimes where representational capacity is limited.

\begin{table}[h]
  \caption{Stress loss ($\downarrow$) (mean $\pm$ std) across embedding dimensionality for the pairwise-distance prediction task on the synthetic tree graph.}
  \label{tab:stress_results_tree}
  \begin{center}
    \begin{small}
      \begin{sc}
      \resizebox{0.99\columnwidth}{!}{
        \begin{tabular}{lcccccc}
        \toprule
        \textbf{Model} & \textbf{3} & \textbf{8} & \textbf{16} &
        \textbf{32} & \textbf{64} & \textbf{128} \\
        \midrule
        GCN 
        & 0.0678 $\pm$ {\scriptsize0.0035}
        & 0.0355 $\pm$ {\scriptsize0.0012}
        & 0.0194 $\pm$ {\scriptsize0.0015}
        & 0.0127 $\pm$ {\scriptsize0.0005}
        & 0.0081 $\pm$ {\scriptsize0.0002}
        & 0.0037 $\pm$ {\scriptsize0.0002}
        \\
        GAT
        & 0.0464 $\pm$ {\scriptsize0.0018}
        & 0.0183 $\pm$ {\scriptsize0.0004}
        & 0.0106 $\pm$ {\scriptsize0.0004}
        & 0.0070 $\pm$ {\scriptsize0.0004}
        & 0.0047 $\pm$ {\scriptsize0.0003}
        & 0.0036 $\pm$ {\scriptsize0.0003}
        \\
        HyboNet (Eucl.)
        & 0.0464 $\pm$ {\scriptsize0.0013}
        & 0.0227 $\pm$ {\scriptsize0.0041}
        & 0.0158 $\pm$ {\scriptsize0.0007}
        & 0.0094 $\pm$ {\scriptsize0.0006}
        & 0.0083 $\pm$ {\scriptsize0.0003}
        & 0.0091 $\pm$ {\scriptsize0.0016}
        \\
        HyboNet 
        & \textbf{0.0383 $\pm$ {\scriptsize 0.0051}}
        & \textbf{0.0058 $\pm$ {\scriptsize0.0007}}
        & \textbf{0.0035 $\pm$ {\scriptsize0.0001}}
        & 0.0034 $\pm$ {\scriptsize0.0002}
        & 0.0034 $\pm$ {\scriptsize0.0002}
        & 0.0035 $\pm$ {\scriptsize0.0001}
        \\
        HGCN
        & 0.0542 $\pm$ {\scriptsize0.0054}
        & 0.0092 $\pm$ {\scriptsize0.0006}
        & 0.0055 $\pm$ {\scriptsize0.0003}
        & \textbf{0.0021 $\pm$ {\scriptsize0.0004}}
        & \textbf{0.0007 $\pm$ {\scriptsize0.0001}}
        & \textbf{0.0006 $\pm$ {\scriptsize0.0001}}
        \\
        \bottomrule
        \end{tabular}
        }
      \end{sc}
    \end{small}
  \end{center}
\end{table}

\begin{table}[h]
  \caption{Stress loss ($\downarrow$) (mean $\pm$ std) across embedding dimensionality for the pairwise-distance prediction task on the synthetic grid graph.}
  \label{tab:stress_results_grid}
  \begin{center}
    \begin{small}
      \begin{sc}
      \resizebox{0.99\columnwidth}{!}{
        \begin{tabular}{lcccccc}
        \toprule
        \textbf{Model} & \textbf{3} & \textbf{8} & \textbf{16} &
        \textbf{32} & \textbf{64} & \textbf{128} \\
        \midrule 
        GCN
        & \textbf{0.0263} $\pm$ {\scriptsize 0.0020}
        & \textbf{0.0128} $\pm$ {\scriptsize 0.0002}
        & \textbf{0.0119} $\pm$ {\scriptsize 0.0001}
        & \textbf{0.0116} $\pm$ {\scriptsize 0.0002}
        & \textbf{0.0106} $\pm$ {\scriptsize 0.0002}
        & \textbf{0.0102} $\pm$ {\scriptsize 0.0001}
        \\
        
        GAT
        & \textbf{0.0235} $\pm$ {\scriptsize 0.0031}
        & \textbf{0.0132} $\pm$ {\scriptsize 0.0006}
        & 0.0120 $\pm$ {\scriptsize 0.0010}
        & \textbf{0.0116} $\pm$ {\scriptsize 0.0005}
        & 0.0114 $\pm$ {\scriptsize 0.0003}
        & 0.0112 $\pm$ {\scriptsize 0.0002}
        \\
        HyboNet (Eucl.)
        & 0.1258 $\pm$ {\scriptsize 0.0664}
        & 0.0206 $\pm$ {\scriptsize 0.0013}
        & 0.0152 $\pm$ {\scriptsize 0.0004}
        & 0.0131 $\pm$ {\scriptsize 0.0001}
        & 0.0122 $\pm$ {\scriptsize 0.0002}
        & 0.0113 $\pm$ {\scriptsize 0.0002}
        \\
        HyboNet
        & 0.1792 $\pm$ {\scriptsize 0.0532}
        & 0.0274 $\pm$ {\scriptsize 0.0048}
        & 0.0184 $\pm$ {\scriptsize 0.0008}
        & 0.0180 $\pm$ {\scriptsize 0.0026}
        & 0.0186 $\pm$ {\scriptsize 0.0017}
        & 0.0181 $\pm$ {\scriptsize 0.0030}
        \\
        
        HGCN
        & 0.1834 $\pm$ {\scriptsize 0.0283}
        & 0.0236 $\pm$ {\scriptsize 0.0120}
        & 0.0124 $\pm$ {\scriptsize 0.0004}
        & 0.0115 $\pm$ {\scriptsize 0.0001}
        & 0.0112 $\pm$ {\scriptsize 0.0002}
        & 0.0112 $\pm$ {\scriptsize 0.0002}
        \\
        \bottomrule
        \end{tabular}
        }
      \end{sc}
    \end{small}
  \end{center}
\end{table}

\begin{table}[h!]
    \caption{$R^2$ score ($\uparrow$) of a linear fit to the embedding distances (averaged per true distance) at embedding dimensions 3 and 128, for tree and grid graphs, averaged over 10 random seeds.}
    \label{table:r2_tree_grid}
  \begin{center}
    \begin{small}
      \begin{sc}
        \begin{tabular}{lcc@{\hspace{2em}}cc}
        \toprule
         & \multicolumn{2}{c}{\textbf{Tree}} & \multicolumn{2}{c}{\textbf{Grid}} \\
        \cmidrule(lr){2-3} \cmidrule(lr){4-5}
        \textbf{Model} & \textbf{3} & \textbf{128} & \textbf{3} & \textbf{128} \\
        \midrule
        GCN 
        & 0.3260 $\pm$ {\scriptsize 0.0491}
        & 0.9737 $\pm$ {\scriptsize 0.0008}
        & 0.9324 $\pm$ {\scriptsize 0.0026}
        & 0.9679 $\pm$ {\scriptsize 0.0004}
        \\
        
        GAT
        & 0.4625 $\pm$ {\scriptsize 0.0033}
        & 0.9778 $\pm$ {\scriptsize 0.0013}
        & 0.9404 $\pm$ {\scriptsize 0.0060}
        & 0.9617 $\pm$ {\scriptsize 0.0005}
        \\
        HyboNet (Eucl.)
        & 0.5006 $\pm$ {\scriptsize 0.0158}
        & 0.9341 $\pm$ {\scriptsize 0.0022}
        & 0.5776 $\pm$ {\scriptsize 0.2634}
        & 0.9664 $\pm$ {\scriptsize 0.0004}
        \\
        HyboNet 
        & 0.5979 $\pm$ {\scriptsize 0.0449}
        & 0.9882 $\pm$ {\scriptsize 0.0007}
        & 0.5129 $\pm$ {\scriptsize 0.1746}
        & 0.9365 $\pm$ {\scriptsize 0.0074}
        \\
        
        HGCN
        & 0.3595 $\pm$ {\scriptsize 0.0527}
        & 0.9992 $\pm$ {\scriptsize 0.0001}
        & 0.3896 $\pm$ {\scriptsize 0.1331}
        & 0.9681 $\pm$ {\scriptsize 0.0012}
        \\
        \bottomrule
        \end{tabular}
      \end{sc}
    \end{small}
  \end{center}
\end{table}

\begin{table}[t]
\centering
\caption{\dionisia{Final-layer curvature magnitudes $c$ learned by HGCN for the PDP task; the manifold curvature is $-c$. Entries report mean $\pm$ standard deviation over 10 seeds.}}
\label{tab:pdp_learned_curvature}
\small
\resizebox{\linewidth}{!}{%
\begin{tabular}{lcccccc}
\toprule
\textbf{Dataset}
& $\boldsymbol{d=3}$
& $\boldsymbol{d=8}$
& $\boldsymbol{d=16}$
& $\boldsymbol{d=32}$
& $\boldsymbol{d=64}$
& $\boldsymbol{d=128}$ \\
\midrule
Tree
& $1.9282 \pm 0.0972$
& $1.0412 \pm 0.0295$
& $2.7957 \pm 0.0978$
& $2.0926 \pm 0.1792$
& $1.5245 \pm 0.0877$
& $1.2167 \pm 0.0177$ \\
Grid
& $0.3829 \pm 0.2024$
& $0.0262 \pm 0.0714$
& $0.0025 \pm 0.0023$
& $0.0024 \pm 0.0012$
& $0.0009 \pm 0.0007$
& $0.0044 \pm 0.0009$ \\
\bottomrule
\end{tabular}%
}
\end{table}

\begin{table}[h!]
    \caption{\dionisia{PDP distance-permutation control: test stress loss ($\downarrow$) after randomly permuting target distances across node pairs on the tree and grid graphs. Entries report mean $\pm$ standard deviation over 10 seeds.}}
    \label{table:null_tree_grid}
  \begin{center}
    \begin{small}
      \begin{sc}
        \begin{tabular}{lcc@{\hspace{2em}}cc}
        \toprule
         & \multicolumn{2}{c}{\textbf{Tree}} & \multicolumn{2}{c}{\textbf{Grid}} \\
        \cmidrule(lr){2-3} \cmidrule(lr){4-5}
        \textbf{Model} & \textbf{3} & \textbf{128} & \textbf{3} & \textbf{128} \\
        \midrule
        GCN 
        & 0.1515 $\pm$ {\scriptsize 0.0048}
        & 0.1528 $\pm$ {\scriptsize 0.0050}
        & 0.5349 $\pm$ {\scriptsize 0.0111}
        & 0.5796 $\pm$ {\scriptsize 0.0079}
        \\
        
        GAT
        & 0.1515 $\pm$ {\scriptsize 0.0048}
        & 0.1515 $\pm$ {\scriptsize 0.0048}
        & 0.5784 $\pm$ {\scriptsize 0.0102}
        & 0.5687 $\pm$ {\scriptsize 0.0057}
        \\
        HyboNet (Eucl.)
        & 0.1577 $\pm$ {\scriptsize 0.0053}
        & 0.1561 $\pm$ {\scriptsize 0.0086}
        & 0.5702 $\pm$ {\scriptsize 0.0066}
        & 0.5824 $\pm$ {\scriptsize 0.0077}
        \\
        HyboNet 
        & 0.1520 $\pm$ {\scriptsize 0.0048}
        & 0.1564 $\pm$ {\scriptsize 0.0054}
        & 0.5767 $\pm$ {\scriptsize 0.0077}
        & 0.5713 $\pm$ {\scriptsize 0.0076}
        \\
        HGCN
        & 0.1522 $\pm$ {\scriptsize 0.0048}
        & 0.1525 $\pm$ {\scriptsize 0.0053}
        & 0.5700 $\pm$ {\scriptsize 0.0054}
        & 0.5772 $\pm$ {\scriptsize 0.0049}
        \\
        \bottomrule
        \end{tabular}
      \end{sc}
    \end{small}
  \end{center}
\end{table}

\begin{figure}[h!]
  \centering
  \begin{subfigure}{0.48\textwidth}
    \includegraphics[width=\linewidth]{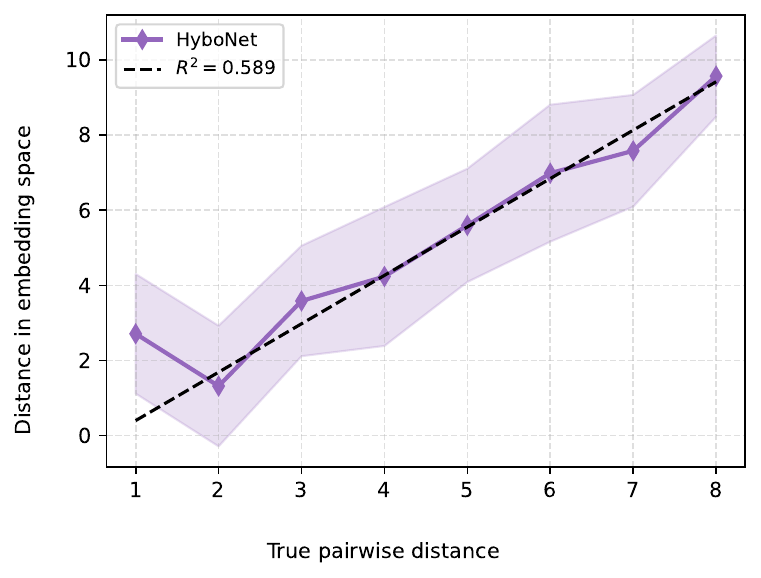}
    \caption{HyboNet, $d=3$}
  \end{subfigure}\hfill
  \begin{subfigure}{0.48\textwidth}
    \includegraphics[width=\linewidth]{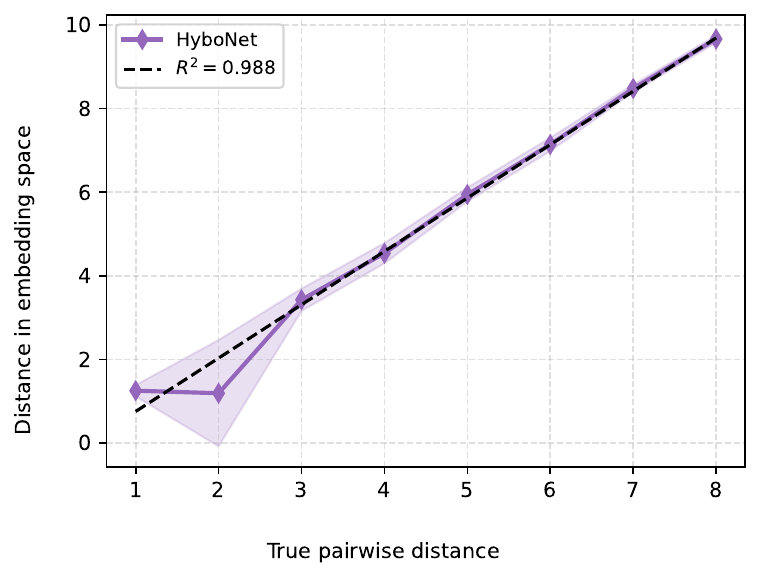}
    \caption{HyboNet, $d=128$}
  \end{subfigure}

  \medskip

  \begin{subfigure}{0.48\textwidth}
    \includegraphics[width=\linewidth]{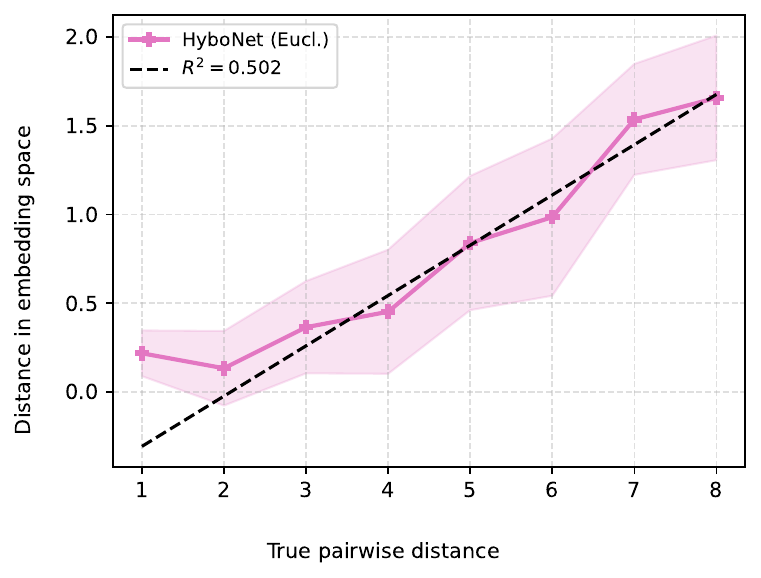}
    \caption{HyboNet (Eucl.), $d=3$}
  \end{subfigure}\hfill
  \begin{subfigure}{0.48\textwidth}
    \includegraphics[width=\linewidth]{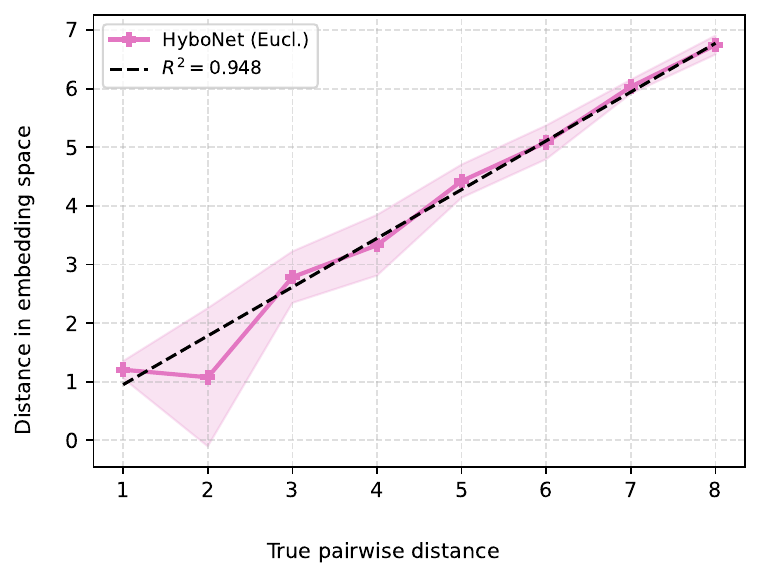}
    \caption{HyboNet (Eucl.), $d=128$}
  \end{subfigure}

  \caption{Line fit of the embedding distances plotted and averaged against the true graph distances for the pairwise–distance prediction task on the \textit{tree}. Results are shown for HyboNet and HyboNet (Eucl.) graphs at low ($d=3$) and high ($d=128$) embedding dimensions. The $R^2$ score measures the goodness of fit of the line, i.e. how faithfully the embedding preserves graph distances.}
  \label{fig:r2_regression}
\end{figure}

\subsubsection{Sensitivity to Tree Depth and Branching Factor}
\label{app:pdp_tree_structure}
\dionisia{To assess whether the PDP findings extend beyond the original tree configuration, we vary the depth and branching factor of balanced rooted trees. For branching factor \(b>1\) and depth \(\ell\), measured as the number of edges from the root to a leaf, the number of nodes is
\[
|\mathcal V|=\sum_{k=0}^{\ell}b^k
=\frac{b^{\ell+1}-1}{b-1}.
\]
Increasing depth at fixed branching factor increases hierarchy length and graph diameter, whereas increasing the branching factor at fixed depth increases the rate of expansion between consecutive levels. Every resulting tree remains \(0\)-hyperbolic. These experiments therefore vary structural properties within the same hyperbolicity class, while also changing graph size.
Table~\ref{tab:pdp_tree_sweeps} summarizes the two controlled sweeps and their corresponding node counts. We fix the hidden dimension to $d=8$, the smallest tested dimension at which a noticeable separation between Euclidean and hyperbolic models consistently appears (Fig.~\ref{fig:stress_vs_embdim}). All configurations follow the same feature-generation, training, and ten-seed evaluation protocol as the original PDP experiments.}
\begin{table}[t]
\centering
\caption{\dionisia{Structural configurations for the additional PDP experiments on balanced rooted trees. Node counts correspond, in order, to the listed values of the varied parameter. Hidden dimension is fixed to $d=8$.}}
\label{tab:pdp_tree_sweeps}
\small
\begin{tabular}{llll}
\toprule
\textbf{Sweep} & \textbf{Fixed parameter}
& \textbf{Varied parameter} & \textbf{Number of nodes} \\
\midrule
A: Depth
& $b=3$
& $\ell\in\{2,3,4,5,6\}$
& $13,40,121,364,1093$ \\[6pt]
B: Branching factor
& $\ell=4$
& $b\in\{2,3,4,5,6\}$
& $31,121,341,781,1555$ \\
\bottomrule
\end{tabular}
\end{table}

\dionisia{Tables~\ref{tab:pdp_depth} and~\ref{tab:pdp_branching} show that hyperbolic models achieve the lowest mean test stress across all tested configurations. HyboNet performs best for depths $\ell=2$--$5$ and branching factors $b=2$--$5$, while HGCN achieves the lowest mean at $\ell=6$ and $b=6$, respectively. HyboNet also consistently outperforms its Euclidean counterpart, supporting the benefit of hyperbolic geometry within the matched architecture. These findings extend the original PDP results beyond a single tree configuration, showing that the observed
advantage persists across varying hierarchy lengths, expansion rates, and associated graph sizes.}

\begin{table}[t]
\centering
\caption{\dionisia{PDP test stress loss ($\downarrow$) on balanced trees with fixed branching factor $b=3$ and varying depth $\ell$, using hidden dimension $d=8$. Entries report mean $\pm$ standard deviation over ten seeds. Bold indicates the lowest mean in each column.}}
\label{tab:pdp_depth}
\small
\resizebox{\linewidth}{!}{%
\begin{tabular}{lccccc}
\toprule
& \multicolumn{5}{c}{\textbf{Depth $\ell$}} \\
\cmidrule(lr){2-6}
\textbf{Model} & \textbf{2} & \textbf{3} & \textbf{4}
& \textbf{5} & \textbf{6} \\
\midrule
GAT
& $0.2081 \pm 0.0922$
& $0.0428 \pm 0.0185$
& $0.0249 \pm 0.0026$
& $0.0193 \pm 0.0013$
& $0.0216 \pm 0.0011$ \\
GCN
& $0.2547 \pm 0.1032$
& $0.0986 \pm 0.0300$
& $0.0551 \pm 0.0031$
& $0.0435 \pm 0.0041$
& $0.0412 \pm 0.0033$ \\
HGCN
& $0.2297 \pm 0.1175$
& $0.0636 \pm 0.0374$
& $0.0307 \pm 0.0163$
& $0.0155 \pm 0.0032$
& $\mathbf{0.0150 \pm 0.0055}$ \\
HyboNet
& $\mathbf{0.1589 \pm 0.0750}$
& $\mathbf{0.0407 \pm 0.0185}$
& $\mathbf{0.0151 \pm 0.0068}$
& $\mathbf{0.0068 \pm 0.0013}$
& $0.0181 \pm 0.0071$ \\
HyboNet (Eucl.)
& $0.1885 \pm 0.0952$
& $0.0719 \pm 0.0221$
& $0.0347 \pm 0.0038$
& $0.0294 \pm 0.0030$
& $0.0360 \pm 0.0050$ \\
\bottomrule
\end{tabular}%
}
\end{table}

\begin{table}[t]
\centering
\caption{\dionisia{PDP test stress loss ($\downarrow$) on balanced
trees with fixed depth $\ell=4$ and varying branching factor $b$,
using hidden dimension $d=8$. Entries report mean $\pm$ standard
deviation over ten seeds. Bold indicates the lowest mean in
each column.}}
\label{tab:pdp_branching}
\small
\resizebox{\linewidth}{!}{%
\begin{tabular}{lccccc}
\toprule
& \multicolumn{5}{c}{\textbf{Branching factor $b$}} \\
\cmidrule(lr){2-6}
\textbf{Model} & \textbf{2} & \textbf{3} & \textbf{4}
& \textbf{5} & \textbf{6} \\
\midrule
GAT
& $0.0465 \pm 0.0145$
& $0.0249 \pm 0.0026$
& $0.0211 \pm 0.0019$
& $0.0185 \pm 0.0006$
& $0.0175 \pm 0.0004$ \\
GCN
& $0.0883 \pm 0.0214$
& $0.0551 \pm 0.0031$
& $0.0453 \pm 0.0040$
& $0.0355 \pm 0.0023$
& $0.0311 \pm 0.0022$ \\
HGCN
& $0.0445 \pm 0.0156$
& $0.0307 \pm 0.0163$
& $0.0092 \pm 0.0006$
& $0.0133 \pm 0.0037$
& $\mathbf{0.0100 \pm 0.0026}$ \\
HyboNet
& $\mathbf{0.0393 \pm 0.0216}$
& $\mathbf{0.0151 \pm 0.0068}$
& $\mathbf{0.0058 \pm 0.0007}$
& $\mathbf{0.0071 \pm 0.0014}$
& $0.0106 \pm 0.0017$ \\
HyboNet (Eucl.)
& $0.0564 \pm 0.0158$
& $0.0347 \pm 0.0038$
& $0.0290 \pm 0.0026$
& $0.0235 \pm 0.0013$
& $0.0191 \pm 0.0013$ \\
\bottomrule
\end{tabular}%
}
\end{table}

\subsection{Link Prediction}
\label{app:LP}
\paragraph{Real-world datasets.} As mentioned in the main text, in addition to the Receiver Operator Characteristic Area under the Curve (ROC AUC) results, we also report average precision (AP) to give a more comprehensive picture. 
The full LP results on the real word datasets are in Table~\ref{tab:lp_real_roc} for ROC AUC and Table~\ref{tab:lp_real_ap} for AP. 
As for the PDP task, we also include HyboNet (Eucl.) as an additional comparison to isolate the effect of model architecture here.  
The AP results therefore do not offer a completely new perspective on the model's performance, but merely underline the findings already made for ROC AUC.
For full disclosure, we report the exact values of Stress Loss for the distortion analysis for LP. 
The absolute values are in Table~\ref{tab:lp_stress_loss_abs} and the normalized values are in Table~\ref{tab:lp_real_stress_loss_rel}. 
\dionisia{Along with the ROC AUC results reported in Figure~\ref{fig:real_noisy}, we report AP and absolute stress loss in Figure~\ref{fig:real_noisy_app} for the feature-noise experiments on Disease and Airport.}
\nips{To analyze the difference in MLP performance at noise level $1.0$ in Figure~\ref{fig:real_noisy}, we compute an SNR-based feature separability score. 
Given a pairwise scoring function $s(\mathbf{x}_i, \mathbf{x}_j)$ (dot product or mean absolute difference), we define
\[
\mathrm{SNR} = \frac{\mu_{\text{pos}} - \mu_{\text{neg}}}{\sigma_{\text{neg}}},
\]
where $\mu_{\text{pos}}$ and $\mu_{\text{neg}}$ denote the mean scores over positive and negative node pairs (i.e., edges vs non-edges), and $\sigma_{\text{neg}}$ is the standard deviation over negative pairs.
As shown in Table~\ref{table:SNR}, Airport retains substantially higher feature separability than Disease at the same noise level, indicating that noise degrades but does not eliminate informative structure. This explains why an MLP, which relies solely on features, still achieves high ROC AUC on Airport even under strong noise.}

\dionisia{To examine how the observed hyperbolic advantage depends on feature informativeness and decoder choice,
Table~\ref{tab:lp_geometry_interactions} brings together two complementary ablations on Disease link prediction. Panel A summarizes the feature-noise analysis shown in Figure~\ref{fig:real_noisy}: with a distance-based decoder, increasing Gaussian noise raises the gain of the best hyperbolic model over the best Euclidean GNN from $-0.12$ to $+18.92$ ROC AUC points. Together with the decreasing feature separability reported in Table~\ref{table:SNR}, this suggests that hyperbolic representations become more useful as informative feature cues weaken. With clean features (Panel B), the hyperbolic advantage observed with distance-based decoding disappears with both concatenation-based decoders. These results support the interpretation that the benefit depends on how predictions use the learned representations: distance-based decoding explicitly relies on the latent metric, whereas concatenation-based decoding can predict links without requiring metric faithfulness.}

\begin{table}[t]
\centering
\caption{SNR-based feature separability under different noise levels. Higher values indicate better separability.}
\label{table:SNR}
\begin{small}
\begin{sc}
\begin{tabular}{llcccc}
\hline
\textbf{Dataset} & \textbf{Method}& \textbf{0.0} & \textbf{0.5} & \textbf{1.0} & \textbf{Noise only} \\
\hline
\multirow{2}{*}{Airport}
 & dot  & 0.6128 & 0.2404 & 0.1041 & 0.0170 \\
 & diff & 1.1763 & 0.3650 & 0.0979 & 0.0246 \\
\hline
\multirow{2}{*}{Disease}
 & dot  & 0.0826 & 0.0494 & 0.0385 & 0.0250 \\
 & diff & 0.1946 & 0.0533 & 0.0106 & 0.0024 \\
\hline
\end{tabular}
\end{sc}
\end{small}
\end{table}

\begin{table}[t]
\centering
\caption{\dionisia{Complementary ablations on Disease link prediction: (A) feature noise with a fixed distance-based decoder, and (B) decoder choice with clean features. Entries report ROC AUC ($\uparrow$), mean $\pm$ standard
deviation over ten splits. $\Delta$ denotes the best hyperbolic score minus the best Euclidean GNN score, excluding the graph-free MLP.}}
\label{tab:lp_geometry_interactions}
\small
\medskip
\textbf{A: Feature quality with a distance-based decoder}
\medskip

\resizebox{\linewidth}{!}{%
\begin{tabular}{llccc}
\toprule
\textbf{Model} & \textbf{Geometry}
& $\boldsymbol{\sigma=0}$
& $\boldsymbol{\sigma=0.5}$
& $\boldsymbol{\sigma=1.0}$ \\
\midrule
MLP (graph-free) & Euclidean
& $89.03 \pm 1.07$ & $69.94 \pm 3.45$ & $57.29 \pm 4.89$ \\
GCN & Euclidean
& $89.62 \pm 0.99$ & $74.18 \pm 0.82$ & $64.72 \pm 2.40$ \\
GAT & Euclidean
& $84.90 \pm 1.56$ & $72.77 \pm 1.36$ & $57.78 \pm 8.08$ \\
HGCN & Hyperbolic (Poincaré)
& $82.36 \pm 2.90$ & $84.09 \pm 1.78$ & $83.64 \pm 3.44$ \\
HyboNet & Hyperbolic (Lorentz)
& $89.50 \pm 1.70$ & $85.20 \pm 1.70$ & $82.95 \pm 1.93$ \\
\midrule
$\boldsymbol{\Delta}$ &
& $\mathbf{-0.12}$ & $\mathbf{+11.02}$ & $\mathbf{+18.92}$ \\
\bottomrule
\end{tabular}%
}

\bigskip
\textbf{B: Decoder choice with clean features}
\medskip

\resizebox{\linewidth}{!}{%
\begin{tabular}{llccc}
\toprule
\textbf{Model} & \textbf{Geometry}
& \textbf{Distance}
& \textbf{Concat + Linear}
& \textbf{Concat + MLP} \\
\midrule
MLP (graph-free) & Euclidean
& $98.43 \pm 0.74$ & $78.56 \pm 10.68$ & $97.91 \pm 0.82$ \\
GCN & Euclidean
& $91.67 \pm 1.79$ & $93.14 \pm 2.24$ & $97.98 \pm 0.59$ \\
GAT & Euclidean
& $93.25 \pm 1.59$ & $76.09 \pm 2.63$ & $96.17 \pm 0.88$ \\
HyboNet (Eucl.) & Euclidean
& $79.76 \pm 4.84$ & --- & --- \\
HGCN & Hyperbolic (Poincaré)
& $95.16 \pm 2.50$ & $87.64 \pm 1.46$ & $93.49 \pm 4.05$ \\
HyboNet & Hyperbolic (Lorentz)
& $97.99 \pm 0.72$ & $92.01 \pm 1.23$ & $95.06 \pm 3.10$ \\
\midrule
$\boldsymbol{\Delta}$ &
& $\mathbf{+4.74}$ & $\mathbf{-1.13}$ & $\mathbf{-2.92}$ \\
\bottomrule
\end{tabular}%
}
\end{table}

\paragraph{Synthetic Datasets.}
We evaluate all our models on the $\textit{Tree1111}_{\gamma}$ dataset from \citet{katsman2025shedding}, which varies node feature informativeness and enables probing model behavior across different feature--structure regimes. Similar to the noisy Disease and Airport experiments in real-world settings, these experiments further stress the role of geometry in a controlled synthetic setting. 
The results in Figure~\ref{fig:tree1111} reproduces the qualitative behavior of the real-world datasets across varying feature informativeness. 
All GNNs outperform the MLP for small values of $\gamma$ as they can make use of structural information.
All GNNs achieve smaller stress losses than the features and the MLP for small values of $\gamma$. 
GAT and HyboNet continue to do so for larger $\gamma$, while the MLP largely follows the trend of the features, which naturally achieve lower stress losses for higher $\gamma$ values (features are more informative about the graph structure).
The hyperbolic methods HGCN and HyboNet generally achieve better performance than GCN, but only HyboNet does so by achieving low stress losses.
The notable exception here is GAT, which performs unusually strongly in this setting. 
We attribute this to the nearly orthogonal $1000$-dimensional features used in $\textit{Tree1111}_{\gamma}$, which act as near-unique node identifiers (similar to one-hot encodings); attention can leverage these signals while avoiding excessive smoothing, unlike degree-normalized aggregation as in GCN.
Last, other than on the noisy real-world experiences, HGCN performs better than HyboNet, which we also attribute to the different attention mechanism in HGCN.

\label{app:link prediction}
\begin{table}[t]
\centering
\caption{ROC AUC ($\uparrow$) results for real-world link prediction (LP). We report Gromov $\delta$-hyperbolicity (lower is more hyperbolic).}
\label{tab:lp_real_roc}
\begin{small}
\begin{sc}
\begin{tabular}{lccccc}
\hline
 & \textbf{Disease} & \textbf{Airport} & \textbf{Cora} & \textbf{Pubmed} & \textbf{Citeseer} \\
\textbf{Model}& \textbf{$\delta=0$} & \textbf{$\delta=2$} & \textbf{$\delta=4$} & \textbf{$\delta=4.5$} & \textbf{$\delta=6.5$} \\
\hline
MLP     & 98.43$\pm$0.74 & 97.66$\pm$0.12 & 90.33$\pm$0.94 & 94.52$\pm$0.26 & 91.45$\pm$0.68 \\
GCN     & 91.67$\pm$1.79 & 96.58$\pm$0.35 & 92.98$\pm$1.01 & 95.03$\pm$0.19 & 96.02$\pm$0.50 \\
GAT     & 93.25$\pm$1.59 & 95.67$\pm$0.33 & 93.50$\pm$0.68 & 96.77$\pm$0.18 & 93.32$\pm$0.92 \\
HyboNet (Eucl.) & 79.76$\pm$4.84 & 95.43$\pm$0.29 & 92.92$\pm$0.75 & 95.72$\pm$0.17 & 92.92$\pm$0.76 \\
HyboNet   & 97.99$\pm$0.72 & 98.15$\pm$0.17 & 93.70$\pm$0.58 & 97.08$\pm$0.09 & 91.53$\pm$1.40 \\
HGCN & 95.16$\pm$2.50 & 97.66$\pm$0.15 & 93.88$\pm$1.09 & 96.86$\pm$0.11 & 94.46$\pm$0.61 \\
\hline
\end{tabular}
\end{sc}
\end{small}
\end{table}

\begin{table}[t]
    \caption{AP ($\uparrow$) results for real-world link prediction (LP). We report Gromov  $\delta$-hyperbolicity (lower is more hyperbolic). }
    \label{tab:lp_real_ap}
  \begin{center}
    \begin{small}
      \begin{sc}
        \begin{tabular}{lccccc}
        \hline
         & \textbf{Disease} & \textbf{Airport} & \textbf{Cora} & \textbf{Pubmed} & \textbf{Citeseer} \\
        \textbf{Model}  & $\delta=0$ & $\delta=2$ & $\delta=4$ & $\delta=4.5$ & $\delta=6.5$ \\
        \hline
        MLP     & 97.71$\pm$1.18 & 96.77$\pm$0.22 & 89.12$\pm$1.35 & 93.54$\pm$0.31 & 90.44$\pm$0.71 \\
        GCN     & 85.47$\pm$3.36 & 94.62$\pm$0.63 & 92.57$\pm$1.29 & 94.57$\pm$0.25 & 96.04$\pm$0.61 \\
        GAT     & 90.38$\pm$2.05 & 95.34$\pm$0.41 & 93.81$\pm$0.76 & 96.95$\pm$0.18 & 93.76$\pm$0.98 \\
        HyboNet (Eucl.) & 75.03$\pm$6.59 & 94.52$\pm$0.47 & 92.78$\pm$0.88 & 95.79$\pm$0.13 & 93.80$\pm$0.74 \\
        HyboNet   & 96.54$\pm$1.62 & 97.55$\pm$0.38 & 93.78$\pm$0.66 & 97.01$\pm$0.09 & 92.47$\pm$1.44 \\
        HGCN & 92.60$\pm$3.90 & 96.97$\pm$0.22 & 93.80$\pm$1.25 & 96.71$\pm$0.14 & 94.71$\pm$0.77 \\
        \hline
        \end{tabular}   
      \end{sc}
    \end{small}
  \end{center}
  \vskip -0.1in
\end{table}

\begin{table}[h!]
    \caption{Stress loss ($\downarrow$) when using embeddings from LP to predict pairwise shortest path graph distances with linear regression.}
    \label{tab:lp_stress_loss_abs}
  \begin{center}
    \begin{small}
       \begin{sc}
            \resizebox{0.99\columnwidth}{!}{
        \begin{tabular}{lccccc}
        \hline
         & \textbf{Disease} & \textbf{Airport} & \textbf{Cora} & \textbf{Pubmed} & \textbf{Citeseer} \\
        \textbf{Model}  & $\delta=0$ & $\delta=2$ & $\delta=4$ & $\delta=4.5$ & $\delta=6.5$ \\
        \hline
        Features & 0.0842 & 0.0907 & 0.1507 & 0.0754 & 0.2584 \\
        MLP & 0.0573$\pm$0.0003 & 0.0476$\pm$0.0005 & 0.1498$\pm$0.0000 & 0.0754$\pm$0.0000 & 0.2611$\pm$0.0004 \\
        GCN & 0.0457$\pm$0.0011 & 0.0567$\pm$0.0009 & 0.1501$\pm$0.0001 & 0.0754$\pm$0.0000 & 0.2618$\pm$0.0004 \\
        GAT & 0.0598$\pm$0.0007 & 0.0518$\pm$0.0012 & 0.1495$\pm$0.0001 & 0.0754$\pm$0.0000 & 0.2625$\pm$0.0007 \\
        HyboNet (Eucl.) & 0.0506$\pm$0.0042 & 0.0661$\pm$0.0010 & 0.1563$\pm$0.0001 & 0.0754$\pm$0.0000 & 0.2618$\pm$0.0002 \\
        HyboNet & 0.0316$\pm$0.0022 & 0.0396$\pm$0.0009 & 0.1561$\pm$0.0004 & 0.0754$\pm$0.0000 & 0.2652$\pm$0.0004 \\
        HGCN & 0.0497$\pm$0.0032 & 0.0479$\pm$0.0027 & 0.1502$\pm$0.0021 & 0.0754$\pm$0.0000 & 0.2569$\pm$0.0006 \\
        \hline
        \end{tabular} 
        }
      \end{sc}
    \end{small}
  \end{center}
  \vskip -0.1in
\end{table}

\begin{table}[h!]
\centering
\caption{Stress loss ($\downarrow$) when using embeddings from LP to predict pairwise shortest path graph distances with linear regression. Values are normalized such that $1$ corresponds to the same loss as with the distances computed directly in feature-space.}
\label{tab:lp_real_stress_loss_rel}
    \begin{small}
      \begin{sc}
\begin{tabular}{lccccc}
\hline
\textbf{Model} & \textbf{Disease} & \textbf{Airport} & \textbf{Cora} & \textbf{Pubmed} & \textbf{Citeseer} \\
& \textbf{$\delta=0$} & \textbf{$\delta=2$} & \textbf{$\delta=4$} & \textbf{$\delta=4.5$} & \textbf{$\delta=6.5$} \\
\hline
MLP & 0.68$\pm$0.00 & 0.53$\pm$0.01 & 0.99$\pm$0.00 & 1.00$\pm$0.00 & 1.01$\pm$0.00 \\
GCN & 0.54$\pm$0.01 & 0.63$\pm$0.01 & 1.00$\pm$0.00 & 1.00$\pm$0.00 & 1.01$\pm$0.00 \\
GAT & 0.71$\pm$0.01 & 0.58$\pm$0.01 & 0.99$\pm$0.00 & 1.00$\pm$0.00 & 1.02$\pm$0.00 \\
HyboNet (Eucl.) & 0.60$\pm$0.05 & 0.74$\pm$0.01 & 1.04$\pm$0.00 & 1.00$\pm$0.00 & 1.01$\pm$0.00 \\
HyboNet & 0.38$\pm$0.03 & 0.44$\pm$0.01 & 1.03$\pm$0.00 & 1.00$\pm$0.00 & 1.03$\pm$0.00 \\
HGCN & 0.59$\pm$0.04 & 0.54$\pm$0.03 & 1.00$\pm$0.01 & 1.00$\pm$0.00 & 0.99$\pm$0.00 \\
\hline
\end{tabular}
      \end{sc}
    \end{small}
\end{table}

\begin{figure*}[h]
  \vskip 0.2in
  \begin{center}
    \begin{subfigure}{0.48\textwidth}
        \centerline{\includegraphics[width=\columnwidth]{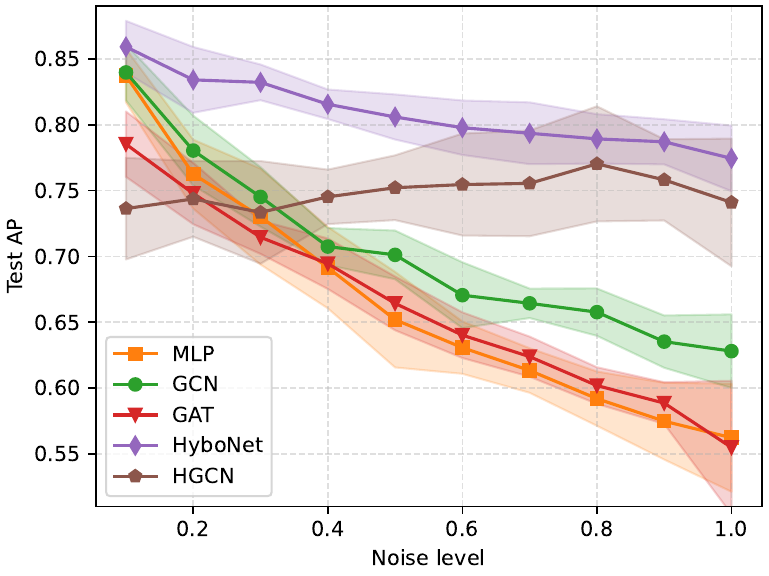}}
        \subcaption{AP ($\uparrow$) for LP on noisy Disease.}
        \label{fig:real_noisy_Disease_ap}
    \end{subfigure}
    \begin{subfigure}{0.48\textwidth}
        \centerline{\includegraphics[width=\columnwidth]{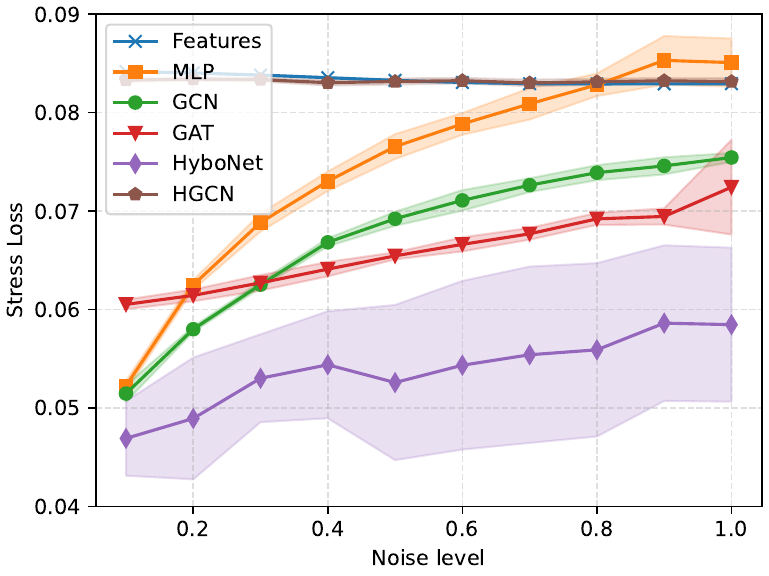}}
        \subcaption{Stress loss ($\downarrow$) for LP on noisy Disease.}
        \label{fig:real_noisy_Disease_stress}
    \end{subfigure}
    \begin{subfigure}{0.48\textwidth}
        \centerline{\includegraphics[width=\columnwidth]{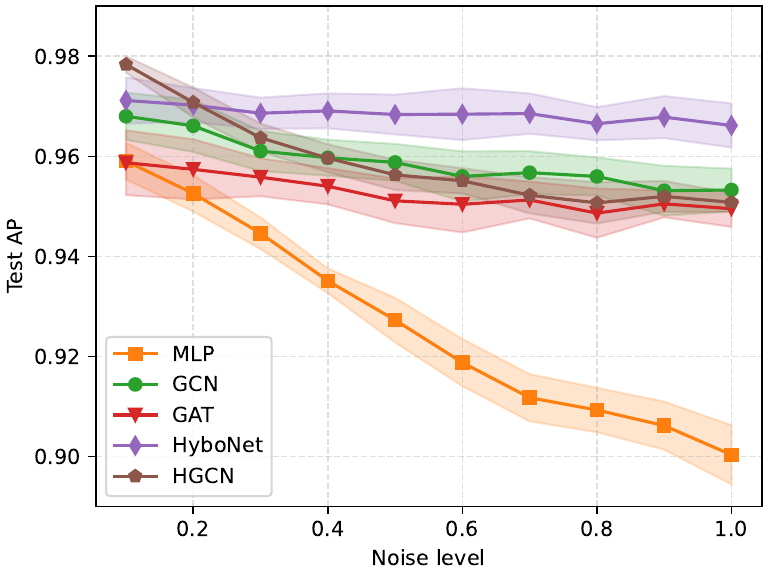}}
        \subcaption{AP ($\uparrow$) for LP on noisy Airport.}
        \label{fig:real_noisy_airport_ap}
    \end{subfigure}
    \begin{subfigure}{0.48\textwidth}
        \centerline{\includegraphics[width=\columnwidth]{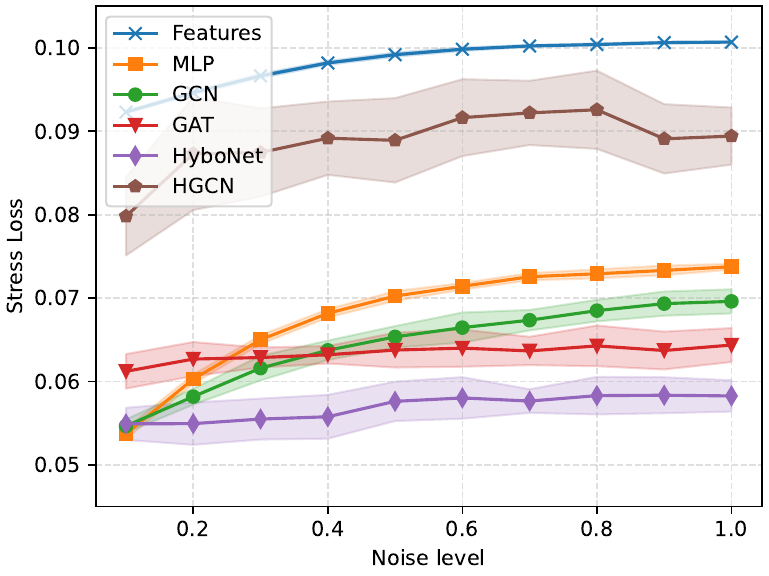}}
        \subcaption{Stress loss ($\downarrow$) on noisy Airport.}
        \label{fig:real_noisy_airport_stress}
    \end{subfigure}
    \caption{LP results (AP and Stress Loss) on real-world datasets Disease ($\delta=0$) and Airport ($\delta=2$) for different levels of noise added to the features. In all cases the GNNs clearly outperform the MLP and the hyperbolic methods on Disease and HyboNet on Airport outperform the Euclidean models.}
    \label{fig:real_noisy_app}
  \end{center}
\end{figure*}
\newpage
\begin{figure}[h]
  \begin{center}
    \begin{subfigure}{ 0.48\textwidth }
        \centerline{\includegraphics[width=\columnwidth]{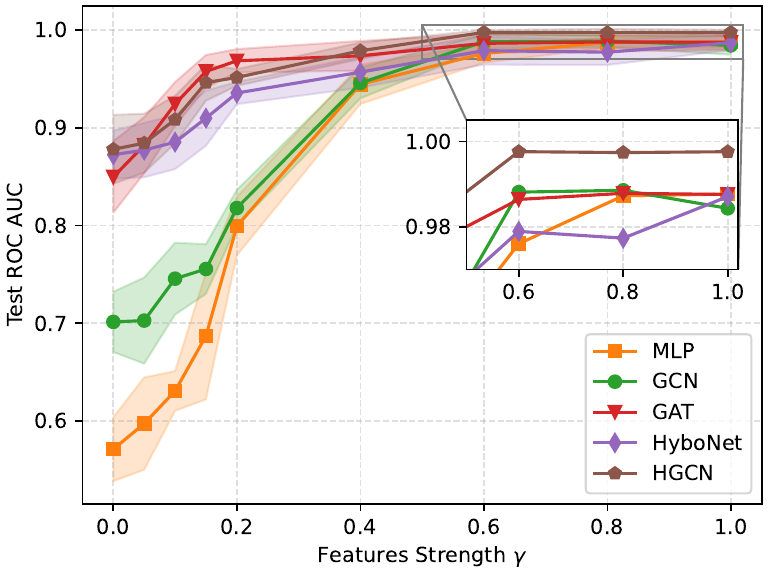}}
        \subcaption{ROC AUC ($\uparrow$) for LP on $\textit{Tree1111}_{\gamma}$ over different $\gamma$ values.}
        \label{fig:tree1111_roc}
    \end{subfigure}
    \begin{subfigure}{ 0.48\textwidth }
        \centerline{\includegraphics[width=\columnwidth]{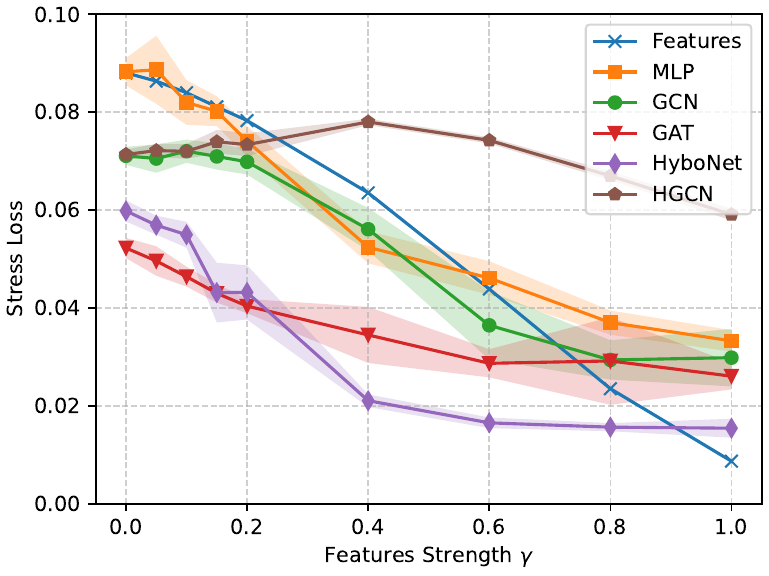}}
        \subcaption{Stress loss ($\downarrow$) for LP on $\textit{Tree1111}_{\gamma}$ over different $\gamma$ values.}
        \label{fig:tree1111_stress_loss}
    \end{subfigure}
      \begin{subfigure}{0.48\textwidth}
        \centerline{\includegraphics[width=\columnwidth]{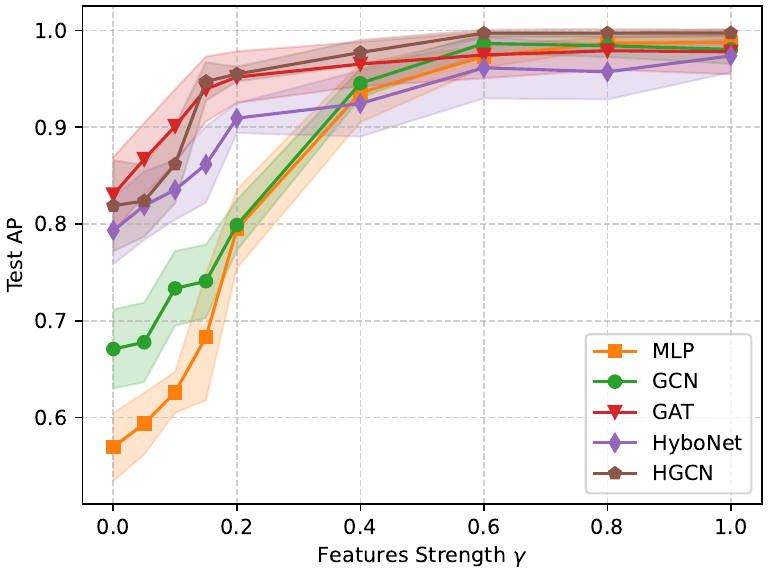}}
        \subcaption{AP ($\uparrow$) for LP on $\textit{Tree1111}_{\gamma}$ over different $\gamma$ values.}
        \label{fig:tree1111_ap}
    \end{subfigure}
    \begin{subfigure}{0.48\textwidth}
        \centerline{\includegraphics[width=\columnwidth]{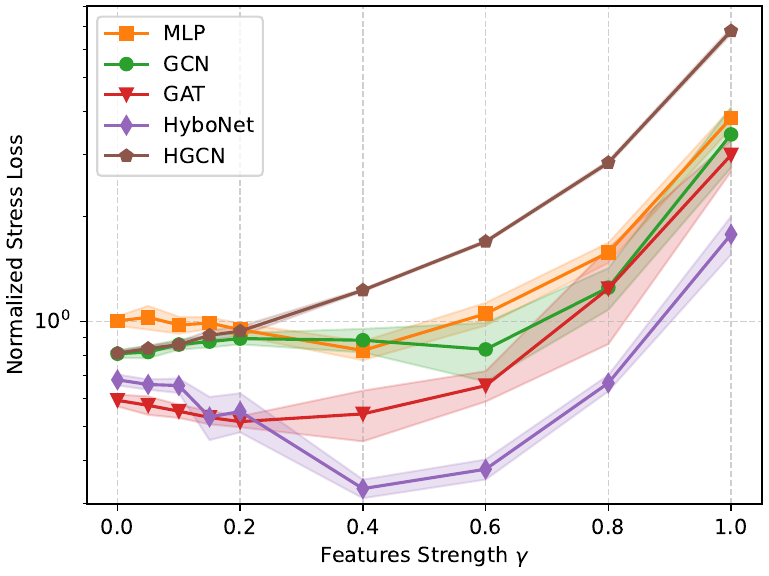}}
        \subcaption{Normalized Stress loss ($\downarrow$) on $\textit{Tree1111}_{\gamma}$ over different $\gamma$ values.}
        \label{fig:tree1111_stress_norm}
    \end{subfigure}
    \caption{LP ROC AUC, AP, Stress Loss and Normalized Stress Loss on the synthetic $\textit{Tree1111}_{\gamma}$ datasets. $\gamma$ encodes the features informativeness. The GNNs outperform MLP for small $\gamma$ and GAT, HGCN and HyboNet perform best, while HyboNet achieves the smallest distortion as measured by the stress loss.}
    \label{fig:tree1111}
  \end{center}
\end{figure}

\subsection{Node Regression}
\label{app:node-regression}
Here we provide additional results on the synthetic tree and grid datasets.
Since for NR, we also use a stress loss as an evaluation metric in training, but on the regression labels from the train set (instead of pairwise shortest path graph distances), we report the results on the test set in Table~\ref{tab:node_regression_test_stress}.
The findings remain the same as for MAE, HyboNet dominates on the tree dataset and the GCN/GAT on the grid. 
HGCN performs reasonably well on both datasets due to its learnable curvature. 
The curvature--performance trade-off for HGCN (with fixed curvature) shown in Figure~\ref{fig:hgcn_nr_curvature_combined} is also very similar as for the MAE. 
We additionally added the result for $d=3$.
Tables~\ref{tab:node_regression_stress_norm} and \ref{tab:node_regression_stress} provide a distortion analysis, measuring how faithfully the learned node representations preserve the input graph metric via the stress loss (Section~\ref{sec:pairs_distortion_pdp}).
They show that no method does preserve pairwise graph distances as measured by the Normalized Stress Loss on the tree and control dataset.
This is expected since the task does not require the method to preserve the global geometry. 
In particular with the root node as anchor, models can completely ignore angular effects and collapse all node embeddings in the same level of the tree. 
On the grid dataset, GCN, GAT and HGCN do reduce distortion to some extent.
By a similar analogy, using a central point as the center of the grid implies the grid geometry has to be preserved to some extent, particularly for smaller distances.

\begin{table*}[h!]
\caption{Node Regression test stress loss ($\downarrow$) results on the synthetic tree and grid dataset for embedding dimension $3$ and $\leq 128$.}
\centering 
\begin{small}
\begin{sc}
\begin{tabular}{l c c c c c}
\hline
 & \multicolumn{2}{c}{\textbf{Tree}} & \multicolumn{2}{c}{\textbf{Grid}} & \textbf{Control Tree}\\
\cmidrule(lr){2-3} \cmidrule(lr){4-5} \cmidrule(lr){6-6}
\textbf{Model} & \textbf{3} & \textbf{$\leq$128} & \textbf{3} & \textbf{$\leq$128}& \textbf{$\leq$128} \\
\hline
MLP     
& 0.2144$\pm$\scriptsize0.0076 
& 0.2236$\pm$\scriptsize0.0086 
& 0.5772$\pm$\scriptsize0.1847 
& 0.7364$\pm$\scriptsize0.3329 
& 0.2319$\pm$\scriptsize0.0215\\

GCN     
& 0.0375$\pm$\scriptsize0.0536 
& 0.0341$\pm$\scriptsize0.0542 
& 0.0029$\pm$\scriptsize0.0010 
& 0.0027$\pm$\scriptsize0.0009 
& 0.2229$\pm$\scriptsize0.0196\\

GAT     
& 0.0708$\pm$\scriptsize0.0713 
& 0.0104$\pm$\scriptsize0.0061 
& 0.0053$\pm$\scriptsize0.0044 
& 0.0022$\pm$\scriptsize0.0012 
& 0.2184$\pm$\scriptsize0.0163\\

HyboNet 
& 0.0044$\pm$\scriptsize0.0046 
& 0.0008$\pm$\scriptsize0.0007 
& 0.0583$\pm$\scriptsize0.0263 
& 0.0069$\pm$\scriptsize0.0040 
& 0.2198$\pm$\scriptsize0.0125\\

HGCN    
& 0.0243$\pm$\scriptsize0.0072 
& 0.0057$\pm$\scriptsize0.0020 
& 0.0039$\pm$\scriptsize0.0022 
& 0.0019$\pm$\scriptsize0.0014 
& 0.2325$\pm$\scriptsize0.0217\\
\hline
\end{tabular}
\end{sc}
\end{small}
\label{tab:node_regression_test_stress}
\end{table*}

\begin{figure}[h!]
  \centering
    \begin{subfigure}{0.3\textwidth}
        \centering
        \includegraphics[width=\columnwidth]{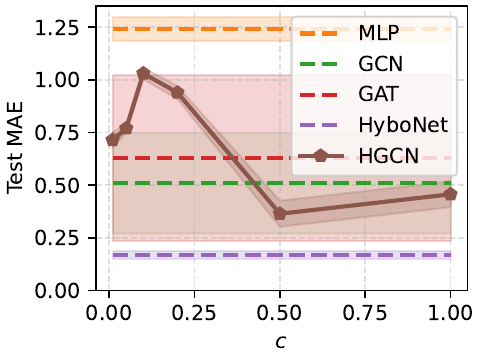}
        \subcaption{MAE, Dimension $3$.}
        \label{fig:hgcn_nr_curvature_mae_dim3}
    \end{subfigure}
    \hfill
    \begin{subfigure}{0.3\textwidth}
        \centering
        \includegraphics[width=\columnwidth]{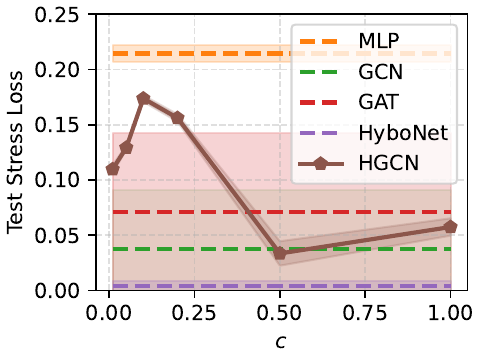}
        \subcaption{Stress, Dimension $3$.}
        \label{fig:hgcn_nr_curvature_stress_dim3}
    \end{subfigure}
    \hfill
    \begin{subfigure}{0.3\textwidth}
        \centering
        \includegraphics[width=\columnwidth]{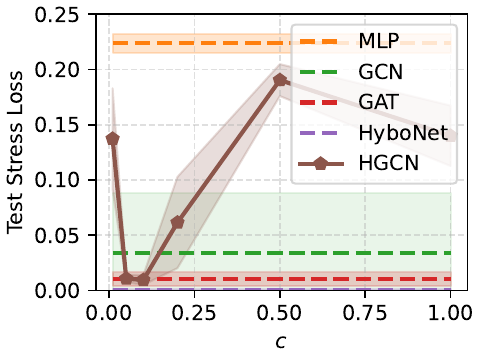}
        \subcaption{Stress, Dimension $128$.}
        \label{fig:hgcn_nr_curvature_stress_high_dim}
    \end{subfigure}

    \caption{NR performance of HGCN vs. its fixed curvature magnitude $c$ on the synthetic tree dataset. We report NR MAE ($\downarrow$) for hidden dimension $3$ (\subref{fig:hgcn_nr_curvature_mae_dim3}) and test stress loss ($\downarrow$) for hidden dimension $3$ (\subref{fig:hgcn_nr_curvature_stress_dim3}) and $128$ (\subref{fig:hgcn_nr_curvature_stress_high_dim}).}
    \label{fig:hgcn_nr_curvature_combined}
\end{figure}

\begin{table}[h]
\caption{Normalized Stress Loss ($\downarrow$) on the synthetic tree and grid dataset using NR embeddings for dimension $3$ and $\leq128$.}
\centering
\begin{sc}
\small
\begin{tabular}{l c c c c c}
\hline
& \multicolumn{3}{c}{\textbf{Tree}} & \multicolumn{2}{c}{\textbf{Grid}} \\
\cmidrule(lr){2-4} \cmidrule(lr){5-6}
\textbf{Model} 
& \textbf{3} 
& \textbf{$\leq$128} 
& \textbf{Control} 
& \textbf{3} 
& \textbf{$\leq$128} \\
\hline

MLP      
& 1.03$\pm$\scriptsize0.07 
& 1.30$\pm$\scriptsize0.08 
& 1.03$\pm$\scriptsize0.02
& 1.18$\pm$\scriptsize0.06 
& 1.06$\pm$\scriptsize0.01 \\

GCN      
& 1.21$\pm$\scriptsize0.03 
& 1.23$\pm$\scriptsize0.09 
& 1.31$\pm$\scriptsize0.08
& 0.64$\pm$\scriptsize0.01 
& 0.97$\pm$\scriptsize0.01 \\

GAT      
& 1.29$\pm$\scriptsize0.09 
& 1.38$\pm$\scriptsize0.10 
& 1.16$\pm$\scriptsize0.09
& 0.63$\pm$\scriptsize0.02
& 0.63$\pm$\scriptsize0.02 \\

HyboNet  
& 1.40$\pm$\scriptsize0.01 
& 1.49$\pm$\scriptsize0.03 
& 1.08$\pm$\scriptsize0.13
& 1.04$\pm$\scriptsize0.09 
& 0.93$\pm$\scriptsize0.06 \\

HGCN     
& 1.25$\pm$\scriptsize0.15 
& 1.23$\pm$\scriptsize0.23 
& 0.90$\pm$\scriptsize0.03
& 0.83$\pm$\scriptsize0.11 
& 0.86$\pm$\scriptsize0.15 \\

\hline
\end{tabular}
\end{sc}
\label{tab:node_regression_stress_norm}
\end{table}

\begin{table*}[h]
\caption{Absolute Stress Loss ($\downarrow$) on the synthetic tree and grid dataset using the embeddings of dimensions $3$ and $\leq 128$ from Node Regression.}
\centering
\begin{small}
\begin{sc}
\begin{tabular}{l c c c c c}
\hline
 & \multicolumn{2}{c}{\textbf{Tree}} & \multicolumn{2}{c}{\textbf{Grid}} & \textbf{Control Tree}\\
\cmidrule(lr){2-3} \cmidrule(lr){4-5} \cmidrule(lr){6-6}
\textbf{Model} & \textbf{3} & \textbf{$\leq$128} & \textbf{3} & \textbf{$\leq$128}& \textbf{$\leq$128} \\
\hline
Features 
& 0.1192
& 0.1192
& 0.5791 
& 0.5791 
& 0.1192\\

MLP      
& 0.1226$\pm$\scriptsize0.0085 
& 0.1550$\pm$\scriptsize0.0100 
& 0.6860$\pm$\scriptsize0.0368 
& 0.6121$\pm$\scriptsize0.0074 
& 0.1231$\pm$\scriptsize0.0030\\

GCN      
& 0.1439$\pm$\scriptsize0.0030 
& 0.1467$\pm$\scriptsize0.0105 
& 0.3731$\pm$\scriptsize0.0075 
& 0.5625$\pm$\scriptsize0.0077 
& 0.1560$\pm$\scriptsize0.0096\\

GAT      
& 0.1539$\pm$\scriptsize0.0104 
& 0.1641$\pm$\scriptsize0.0121 
& 0.3637$\pm$\scriptsize0.0087 
& 0.3667$\pm$\scriptsize0.0093 
& 0.1380$\pm$\scriptsize0.0105\\

HyboNet  
& 0.1665$\pm$\scriptsize0.0006 
& 0.1775$\pm$\scriptsize0.0039 
& 0.6049$\pm$\scriptsize0.0534 
& 0.5409$\pm$\scriptsize0.0355 
& 0.1291$\pm$\scriptsize0.0155\\

HGCN     
& 0.1486$\pm$\scriptsize0.0181 
& 0.1463$\pm$\scriptsize0.0270 
& 0.4800$\pm$\scriptsize0.0638 
& 0.5007$\pm$\scriptsize0.0873 
& 0.1077$\pm$\scriptsize0.0038\\

\hline
\end{tabular}
\end{sc}
\end{small}
\label{tab:node_regression_stress}
\end{table*}

\dionisia{Beyond evaluating learned representations, we also explore whether task-side geometric structure can be detected before training. As a preliminary diagnostic
for node regression, we compare pairwise distances in the graph and label spaces, normalizing each by its maximum:}
\[
\widetilde d_{\mathcal G}(i,j)
=\frac{d_{\mathcal G}(i,j)}
       {\max_{u\neq v}d_{\mathcal G}(u,v)},
\qquad
\widetilde d_{\mathcal Y}(i,j)
=\frac{|y_i-y_j|}
       {\max_{u\neq v}|y_u-y_v|}.
\]
\dionisia{For connected graphs and nonconstant scalar labels, we define the pairwise slopes and their maximum as}
\[
s_{ij}
=\frac{\widetilde d_{\mathcal Y}(i,j)}
       {\widetilde d_{\mathcal G}(i,j)},
\qquad
L_{\max}=\max_{i\neq j}s_{ij}.
\]
\dionisia{When computed over all pairs, $L_{\max}$ is the smallest upper Lipschitz constant of the normalized label map on the observed nodes. Smaller values indicate smoother target variation relative to the graph metric. We compare the observed labels with controls obtained by permuting labels across nodes or replacing them with independent uniform random labels.}

\dionisia{Table~\ref{tab:alignment_diagnostic} shows that both controls increase the score markedly on Tree and Grid, indicating that the diagnostic detects the loss of task-side smoothness. Airport exhibits a smaller separation between observed
and randomized labels, suggesting weaker geometric structure detectable by this diagnostic. Comparisons should be made within each dataset, since normalization
makes the absolute score depend on graph diameter and label range.}

\dionisia{This diagnostic captures only upper-Lipschitz smoothness, rather than the two-sided distance preservation considered in Section~\ref{sec:rq2_task_geometry}. Smooth labels may collapse distances between distinct nodes; consequently,
a small score neither establishes geometry--task alignment nor guarantees a hyperbolic advantage. The maximum is also sensitive to extreme pairs. Although the diagnostic requires no model training, computing all pairwise graph distances can be costly; sampling pairs reduces this cost
but may underestimate the maximum. Robust slope quantiles, lower-distance preservation, and rank correlations between graph and label distances could provide complementary
information.}

\begin{table}[t]
\centering
\caption{\dionisia{Maximum normalized pairwise slope
$L_{\max}$ for observed, permuted, and uniformly randomized
labels. Comparisons within each dataset assess how
randomization affects target smoothness relative to the
graph metric.}}
\label{tab:alignment_diagnostic}
\small
\begin{tabular}{lccc}
\toprule
\textbf{Labels} & \textbf{Airport}
& \textbf{Tree} & \textbf{Grid} \\
\midrule
Observed       & 10.36 & 2.00 & 27.00 \\
Permuted       & 12.00 & 8.00 & 46.29 \\
Random uniform & 11.88 & 7.76 & 52.25 \\
\bottomrule
\end{tabular}
\end{table}

\newpage
\subsection{Node Classification}
\label{app: node classification}
\paragraph{Real-world datasets.} As mentioned in the main text, in addition to the Macro F1 results, we also report accuracy (Acc) to give a more comprehensive picture. 
The full NC results on the real word datasets are in Table~\ref{tab:nc_real_f1} for Macro F1 and Table~\ref{tab:nc_real_acc} for Acc. 
As for the PDP task, we also include HyboNet (Eucl.) as an additional comparison to isolate the effect of model architecture here.  
The Acc results do not offer a completely new perspective on the model's performance, but merely underline the findings already made for Macro F1.
The reason for this is simply because most of these datasets are not very imbalanced, with the exception of Disease.
If a model ignores the minority class, the accuracy results can still be good.

On the Disease* dataset with the $70/15/15\%$ split that keeps the ratio of the two classes in all splits, we can see a larger difference between the Macro F1 and Accuracy columns in the two tables (Table~\ref{tab:nc_real_f1} and Table~\ref{tab:nc_real_acc}). 
All models (with the exception of the MLP) achieve almost equally good performance w.r.t. to accuracy but using Macro F1, the differences become more pronounced. 
This is not the case for the old Disease results using the $70/5/25\%$ split, with balanced $50/50\%$ ratio in validation and test set, where naturally the F1 and the Accuracy results match more closely. 
We would also like to note, that the performance advantage of HyboNet is much more pronounced in this setting.
This likely arises due to the very few training samples of the minority class in train set (due to the old splitting technique).
In this case, the performance comes more from the inductive bias in a heavily restricted sample regime (low sample complexity, which is something we proposed as future work), rather than from geometry of the task.
For precisely this reason, we unified the splitting (sufficiently large data to learn in all cases), resulting in a much more fair scenario to isolate the effect of geometry--task alignment.

Again, as for LP, for full disclosure, we report the exact values of Stress Loss for the distortion analysis for NC. 
The absolute values are in Table~\ref{tab:nc_stress_loss_abs} and the normalized values are in Table~\ref{tab:nc_real_stress_loss_rel}.
No model is able to significantly reduce distortion for any dataset.

\nips{We provide an ablation with alternative decoders (Linear and MLP), with results shown in Table~\ref{table:nc_decoder}. The NC decoder ablation suggests that decoder choice does not change the relative behavior of hyperbolic and Euclidean models.}

\begin{table}[t]
\centering
\caption{NC decoder ablation (Test Macro F1 $\uparrow$).}
\label{table:nc_decoder}
\begin{small}
\begin{sc}
\begin{tabular}{lccccc}
\hline
& \multicolumn{5}{c}{\textbf{MLP Decoder}} \\
\cmidrule(lr){2-6}
\textbf{Model} & \textbf{Disease} & \textbf{Airport} & \textbf{Cora} & \textbf{Pubmed} & \textbf{Citeseer} \\
\hline

GCN     
& 83.01$\pm${\scriptsize 6.32} & 91.38$\pm${\scriptsize 1.53} & 74.59$\pm${\scriptsize 0.96} & 75.27$\pm${\scriptsize 1.18} & 61.86$\pm${\scriptsize 1.39} \\

GAT     
& 64.31$\pm${\scriptsize 5.77} & 89.82$\pm${\scriptsize 1.57} & 76.98$\pm${\scriptsize 1.41} & 75.65$\pm${\scriptsize 1.50} & 62.75$\pm${\scriptsize 1.71} \\

HyboNet 
& 91.11$\pm${\scriptsize 3.20} & 90.21$\pm${\scriptsize 1.63} & 71.63$\pm${\scriptsize 4.16} & 78.30$\pm${\scriptsize 1.68} & 53.24$\pm${\scriptsize 2.75} \\

HGCN    
& 88.90$\pm${\scriptsize 5.22} & 91.19$\pm${\scriptsize 2.01} & 77.53$\pm${\scriptsize 0.91} & 77.94$\pm${\scriptsize 1.05} & 62.92$\pm${\scriptsize 1.98} \\

\hline
& \multicolumn{5}{c}{\textbf{Linear Decoder}} \\
\cmidrule(lr){2-6}

GCN     
& 84.43$\pm${\scriptsize 4.86} & 91.25$\pm${\scriptsize 1.53} & 75.89$\pm${\scriptsize 0.85} & 75.78$\pm${\scriptsize 1.78} & 63.20$\pm${\scriptsize 1.74} \\

GAT     
& 63.29$\pm${\scriptsize 8.34} & 89.78$\pm${\scriptsize 1.07} & 78.79$\pm${\scriptsize 1.29} & 75.77$\pm${\scriptsize 1.56} & 63.94$\pm${\scriptsize 1.49} \\

HyboNet 
& 87.13$\pm${\scriptsize 7.01} & 90.66$\pm${\scriptsize 0.90} & 77.55$\pm${\scriptsize 1.40} & 78.18$\pm${\scriptsize 1.54} & 61.18$\pm${\scriptsize 1.54} \\

\hline
\end{tabular}
\end{sc}
\end{small}
\end{table}

\paragraph{Synthetic Datasets.}
We evaluate all our models on the synthetic tree graph we used in our NR experiments, the results are summarized in Table~\ref{tab:node_classification_synthetic}.
Different from NR, we could not verify the performance advantage of HGNNs in this setting.
In the embedding dimension $\leq128$ setting, GCN, HyboNet and HGCN are all able to perform perfect classification.
Again they do this without preserving the graph structure as indicated by the (Normalized) Stress Loss.
If we constrain the embedding dimension to $d=3$, the results seem to indicate some advantage of hyperbolic methods, but with a few caveats.
The results suggest that HGCN clearly performs best, with HyboNet following, but checking the standard deviations, GCN could also be among the best. 
The reason for this that across the $10$ splits, all models (except for HGCN and MLP), perform quite well, while on others they basically fail to learn the task.
Now the hyperparameters were chosen only across $2$ runs, on which all models (except MLP) performed quite well. 
All in all, the results suggest that in future work a more in-depth analysis is necessary to make a clear statement, maybe also considering sample complexity.

\begin{table}[t]
    \caption{Macro F1 scores ($\uparrow$) for real-world node classification (NC). We report Gromov  $\delta$-hyperbolicity (lower is more hyperbolic).}
    \label{tab:nc_real_f1}
  \begin{center}
    \begin{small}
       \begin{sc}
            \resizebox{0.99\columnwidth}{!}{
        \begin{tabular}{lcccccc}
        \hline
        & \textbf{Disease}& \textbf{Disease*} & \textbf{Airport} & \textbf{Cora} & \textbf{Pubmed} & \textbf{Citeseer} \\
        \textbf{Model}  & $\delta=0$ &$\delta=0$ & $\delta=2$ & $\delta=4$ & $\delta=4.5$ & $\delta=6.5$ \\
        \hline
        MLP &  77.12$\pm$7.53 &48.62$\pm$5.06 & 88.79$\pm$1.97 & 55.63$\pm$1.50 & 70.45$\pm$1.12 & 54.29$\pm$1.14 \\
        GCN & 89.14$\pm$2.33 & 89.50$\pm$4.82 & 92.46$\pm$0.86 & 80.09$\pm$1.01 & 78.49$\pm$1.47 & 64.95$\pm$1.50 \\
        GAT & 86.78$\pm$9.14 & 85.71$\pm$5.50 & 88.33$\pm$5.73 & 80.71$\pm$1.22 & 78.44$\pm$1.88 & 65.16$\pm$1.48 \\
        HyboNet (Eucl.) & 89.18$\pm$1.56 & 90.51$\pm$3.74 & 90.54$\pm$1.40 & 78.65$\pm$1.04 & 78.55$\pm$1.28 & 62.51$\pm$1.67 \\
        HyboNet & 94.42$\pm$1.70& 90.48$\pm$5.12 & 90.87$\pm$1.10 & 78.72$\pm$1.01 & 78.33$\pm$1.97 & 62.98$\pm$1.48 \\
        HGCN & 89.52$\pm$1.60 & 90.51$\pm$3.45 & 90.92$\pm$2.04 & 78.34$\pm$0.95 & 78.21$\pm$1.47 & 63.29$\pm$2.07 \\
        \hline
        \end{tabular}
        }
      \end{sc}
    \end{small}
  \end{center}
  \vskip -0.1in
\end{table}

\begin{table}[t]
    \caption{Accuracy ($\uparrow$) results for node classification (NC) tasks. We report Gromov  $\delta$-hyperbolicity (lower is more hyperbolic). Note that accuracy on Disease is less meaningful in comparison to the other models than the one of Disease* due to class imbalance effects.}
    \label{tab:nc_real_acc}
  \begin{center}
    \begin{small}
      \ \begin{sc}
            \resizebox{0.99\columnwidth}{!}{
        \begin{tabular}{lcccccc}
        \hline
         & \textbf{Disease}& \textbf{Disease*} & \textbf{Airport} & \textbf{Cora} & \textbf{Pubmed} & \textbf{Citeseer} \\
        \textbf{Model}  & $\delta=0$ &$\delta=0$ & $\delta=2$ & $\delta=4$ & $\delta=4.5$ & $\delta=6.5$ \\
        \hline
        MLP & 71.04$\pm$4.30& 60.45$\pm$3.14 & 89.48$\pm$2.70 & 57.40$\pm$1.73 & 71.03$\pm$1.14 & 56.86$\pm$1.02 \\
        GCN & 89.88$\pm$3.04& 95.41$\pm$2.36 & 93.18$\pm$0.98 & 81.08$\pm$0.76 & 78.33$\pm$1.24 & 68.45$\pm$1.40 \\
        GAT & 87.69$\pm$8.41& 94.01$\pm$2.28 & 89.77$\pm$4.63 & 81.86$\pm$1.13 & 77.95$\pm$1.79 & 68.56$\pm$1.22 \\
        HyboNet (Eucl.) & 89.46$\pm$1.99& 96.11$\pm$1.46 & 91.51$\pm$1.35 & 79.84$\pm$1.06 & 78.04$\pm$1.13 & 66.35$\pm$1.30 \\
        HyboNet & 94.50$\pm$1.85& 95.99$\pm$2.23 & 92.05$\pm$1.54 & 79.89$\pm$1.14 & 77.91$\pm$1.81 & 66.76$\pm$1.95 \\
        HGCN & 89.92$\pm$1.28& 96.18$\pm$1.40 & 91.69$\pm$1.73 & 79.48$\pm$0.66 & 78.51$\pm$1.61 & 66.94$\pm$1.70 \\
        \hline
        \end{tabular}
        }
      \end{sc}
    \end{small}
  \end{center}
  \vskip -0.1in
\end{table}

\begin{table}[t]
    \caption{Absolute Stress Loss ($\downarrow$) results of the real-world datasets when using embeddings from NC.}
    \label{tab:nc_stress_loss_abs}
  \begin{center}
    \begin{small}
      \begin{sc}
            \resizebox{0.99\columnwidth}{!}{
        \begin{tabular}{lcccccc}
        \hline
         & \textbf{Disease}& \textbf{Disease*} & \textbf{Airport} & \textbf{Cora} & \textbf{Pubmed} & \textbf{Citeseer} \\
        \textbf{Model}  & $\delta=0$& $\delta=0$ & $\delta=2$ & $\delta=4$ & $\delta=4.5$ & $\delta=6.5$ \\
        \hline
        Features & 0.2301& 0.2301 & 0.0907 & 0.1508 & 0.0754 & 0.2584 \\
        MLP & 0.4700$\pm$0.1301& 0.4249$\pm$0.1156 & 0.0833$\pm$0.0013 & 0.2820$\pm$0.0077 & 0.0754$\pm$0.0000 & 0.2572$\pm$0.0005 \\
        GCN & 0.3098$\pm$0.0298& 0.3465$\pm$0.0353 & 0.0930$\pm$0.0016 & 0.1495$\pm$0.0001 & 0.0754$\pm$0.0000 & 0.2652$\pm$0.0007 \\
        GAT & 0.3057$\pm$0.0829& 0.3447$\pm$0.0677 & 0.0832$\pm$0.0101 & 0.1503$\pm$0.0004 & 0.0754$\pm$0.0000 & 0.2698$\pm$0.0009 \\
        HyboNet (Eucl.) & 0.2858$\pm$0.0270& 0.3216$\pm$0.0184 & 0.0937$\pm$0.0011 & 0.1563$\pm$0.0001 & 0.0754$\pm$0.0000 & 0.2701$\pm$0.0007 \\
        HyboNet & 0.2778$\pm$0.0270& 0.2457$\pm$0.0219 & 0.0644$\pm$0.0042 & 0.1511$\pm$0.0001 & 0.0754$\pm$0.0000 & 0.2653$\pm$0.0011 \\
        HGCN & 0.2861$\pm$0.0245& 0.2627$\pm$0.0324 & 0.0834$\pm$0.0019 & 0.1512$\pm$0.0027 & 0.0754$\pm$0.0000 & 0.2578$\pm$0.0045 \\
        \hline
        \end{tabular}
        }
      \end{sc}
    \end{small}
  \end{center}
  \vskip -0.1in
\end{table}

\begin{table}[t]
    \caption{Normalized Stress Loss ($\downarrow$) results of the real-world datasets when using embeddings from NC. Values are normalized such that $1$ corresponds to the same loss as with the distances using directly the features.}
    \label{tab:nc_real_stress_loss_rel}
  \begin{center}
    \begin{small}
      \begin{sc}
        \begin{tabular}{lcccccc}
        \hline
          & \textbf{Disease}& \textbf{Disease*} & \textbf{Airport} & \textbf{Cora} & \textbf{Pubmed} & \textbf{Citeseer} \\
        \textbf{Model}  & $\delta=0$ &$\delta=0$ & $\delta=2$ & $\delta=4$ & $\delta=4.5$ & $\delta=6.5$ \\
        \hline
        MLP & 2.04$\pm$0.57& 1.85$\pm$0.50 & 0.93$\pm$0.01 & 1.87$\pm$0.05 & 1.00$\pm$0.00 & 0.99$\pm$0.00 \\
        GCN & 1.35$\pm$0.13& 1.51$\pm$0.15 & 1.04$\pm$0.02 & 0.99$\pm$0.00 & 1.00$\pm$0.00 & 1.03$\pm$0.00 \\
        GAT & 1.33$\pm$0.36& 1.50$\pm$0.29 & 0.93$\pm$0.11 & 1.00$\pm$0.00 & 1.00$\pm$0.00 & 1.04$\pm$0.00 \\
        HyboNet (Eucl.) & 1.24$\pm$0.12& 1.40$\pm$0.08 & 1.05$\pm$0.01 & 1.04$\pm$0.00 & 1.00$\pm$0.00 & 1.05$\pm$0.00 \\
        HyboNet & 1.21$\pm$0.12& 1.07$\pm$0.10 & 0.72$\pm$0.05 & 1.00$\pm$0.00 & 1.00$\pm$0.00 & 1.03$\pm$0.00 \\
        HGCN & 1.24$\pm$0.11& 1.14$\pm$0.14 & 0.93$\pm$0.02 & 1.00$\pm$0.02 & 1.00$\pm$0.00 & 1.00$\pm$0.02 \\
        \hline
        \end{tabular}
      \end{sc}
    \end{small}
  \end{center}
  \vskip -0.1in
\end{table}

\begin{table}[h]
\caption{Macro F1 scores and Accuracy ($\uparrow$) and (Normalized) Stress Loss ($\downarrow$) for node classification (NC) on the synthetic tree dataset also used for node regression. The features loss for the normalization is $0.1192$.}
\centering
\begin{sc}
\resizebox{0.99\columnwidth}{!}{
\begin{tabular}{l c c c c c c c c}
\hline
 & \multicolumn{2}{c}{Macro F1} & \multicolumn{2}{c}{Accuracy} & \multicolumn{2}{c}{Stress Loss} & \multicolumn{2}{c}{Norm. Stress Loss}\\
\cmidrule(lr){2-3}
\cmidrule(lr){4-5}
\cmidrule(lr){6-7}
\cmidrule(lr){8-9}
\textbf{Model} & \textbf{3} & \textbf{$\leq$128}& \textbf{3} & \textbf{$\leq$128}& \textbf{3} & \textbf{$\leq$128}& \textbf{3} & \textbf{$\leq$128} \\
\hline
MLP      
& 22.69$\pm$\scriptsize2.27 
& 21.80$\pm$\scriptsize1.79
& 35.15$\pm$\scriptsize9.53 
& 35.24$\pm$\scriptsize8.70
& 0.1341$\pm$\scriptsize0.0088 
& 0.3490$\pm$\scriptsize0.3242
& 1.13$\pm$\scriptsize0.07 
& 2.93$\pm$\scriptsize2.72 \\

GCN      
& 70.62$\pm$\scriptsize29.87 
& 100.00$\pm$\scriptsize0.00
& 75.75$\pm$\scriptsize26.36 
& 100.00$\pm$\scriptsize0.00
& 0.1447$\pm$\scriptsize0.0108 
& 0.1489$\pm$\scriptsize0.0044
& 1.21$\pm$\scriptsize0.09 
& 1.25$\pm$\scriptsize0.04 \\

GAT      
& 69.09$\pm$\scriptsize18.16 
& 69.96$\pm$\scriptsize12.10
& 69.09$\pm$\scriptsize18.16 
& 63.36$\pm$\scriptsize14.49
& 0.1433$\pm$\scriptsize0.0438 
& 0.1352$\pm$\scriptsize0.0098
& 1.20$\pm$\scriptsize0.37 
& 1.13$\pm$\scriptsize0.08 \\

HyboNet  
& 79.78$\pm$\scriptsize31.81 
& 97.99$\pm$\scriptsize3.21 
& 82.78$\pm$\scriptsize27.13 
& 98.85$\pm$\scriptsize2.10
& 0.1265$\pm$\scriptsize0.0131 
& 0.1134$\pm$\scriptsize0.0228
& 1.06$\pm$\scriptsize0.11 
& 0.95$\pm$\scriptsize0.19 \\

HGCN     
& 97.02$\pm$\scriptsize1.73 
& 100.00$\pm$\scriptsize0.00
& 99.42$\pm$\scriptsize1.73 
& 100.00$\pm$\scriptsize0.00
& 0.1457$\pm$\scriptsize0.0125 
& 0.1304$\pm$\scriptsize0.0059
& 1.22$\pm$\scriptsize0.11 
& 1.09$\pm$\scriptsize0.05 \\

\hline
\end{tabular}
}
\end{sc}
\label{tab:node_classification_synthetic}
\end{table}

\clearpage
\section*{NeurIPS Paper Checklist}

\begin{enumerate}

\item {\bf Claims}
    \item[] Question: Do the main claims made in the abstract and introduction accurately reflect the paper's contributions and scope?
    \item[] Answer:\answerYes{} 
    \item[] Justification: See Abstract and Section \ref{sec:intro}.
    \item[] Guidelines:
    \begin{itemize}
        \item The answer \answerNA{} means that the abstract and introduction do not include the claims made in the paper.
        \item The abstract and/or introduction should clearly state the claims made, including the contributions made in the paper and important assumptions and limitations. A \answerNo{} or \answerNA{} answer to this question will not be perceived well by the reviewers. 
        \item The claims made should match theoretical and experimental results, and reflect how much the results can be expected to generalize to other settings. 
        \item It is fine to include aspirational goals as motivation as long as it is clear that these goals are not attained by the paper. 
    \end{itemize}

\item {\bf Limitations}
    \item[] Question: Does the paper discuss the limitations of the work performed by the authors?
    \item[] Answer: \answerYes{} 
    \item[] Justification: See Section \ref{sec:conclusions}.
    \item[] Guidelines:
    \begin{itemize}
        \item The answer \answerNA{} means that the paper has no limitation while the answer \answerNo{} means that the paper has limitations, but those are not discussed in the paper. 
        \item The authors are encouraged to create a separate ``Limitations'' section in their paper.
        \item The paper should point out any strong assumptions and how robust the results are to violations of these assumptions (e.g., independence assumptions, noiseless settings, model well-specification, asymptotic approximations only holding locally). The authors should reflect on how these assumptions might be violated in practice and what the implications would be.
        \item The authors should reflect on the scope of the claims made, e.g., if the approach was only tested on a few datasets or with a few runs. In general, empirical results often depend on implicit assumptions, which should be articulated.
        \item The authors should reflect on the factors that influence the performance of the approach. For example, a facial recognition algorithm may perform poorly when image resolution is low or images are taken in low lighting. Or a speech-to-text system might not be used reliably to provide closed captions for online lectures because it fails to handle technical jargon.
        \item The authors should discuss the computational efficiency of the proposed algorithms and how they scale with dataset size.
        \item If applicable, the authors should discuss possible limitations of their approach to address problems of privacy and fairness.
        \item While the authors might fear that complete honesty about limitations might be used by reviewers as grounds for rejection, a worse outcome might be that reviewers discover limitations that aren't acknowledged in the paper. The authors should use their best judgment and recognize that individual actions in favor of transparency play an important role in developing norms that preserve the integrity of the community. Reviewers will be specifically instructed to not penalize honesty concerning limitations.
    \end{itemize}

\item {\bf Theory assumptions and proofs}
    \item[] Question: For each theoretical result, does the paper provide the full set of assumptions and a complete (and correct) proof?
    \item[] Answer: \answerYes{} 
    \item[] Justification: We do provide in  Appendix \ref{app:proofs}.
    \item[] Guidelines:
    \begin{itemize}
        \item The answer \answerNA{} means that the paper does not include theoretical results. 
        \item All the theorems, formulas, and proofs in the paper should be numbered and cross-referenced.
        \item All assumptions should be clearly stated or referenced in the statement of any theorems.
        \item The proofs can either appear in the main paper or the supplemental material, but if they appear in the supplemental material, the authors are encouraged to provide a short proof sketch to provide intuition. 
        \item Inversely, any informal proof provided in the core of the paper should be complemented by formal proofs provided in appendix or supplemental material.
        \item Theorems and Lemmas that the proof relies upon should be properly referenced. 
    \end{itemize}

    \item {\bf Experimental result reproducibility}
    \item[] Question: Does the paper fully disclose all the information needed to reproduce the main experimental results of the paper to the extent that it affects the main claims and/or conclusions of the paper (regardless of whether the code and data are provided or not)?
    \item[] Answer: \answerYes{} 
    \item[] Justification:  Sections \ref{sec:pdp_task}, \ref{sec:node_regression} \ref{sec:nc_exp}, \ref{sec:link_prediction_real}  and Appendix \ref{app:tuning}.
    \item[] Guidelines:
    \begin{itemize}
        \item The answer \answerNA{} means that the paper does not include experiments.
        \item If the paper includes experiments, a \answerNo{} answer to this question will not be perceived well by the reviewers: Making the paper reproducible is important, regardless of whether the code and data are provided or not.
        \item If the contribution is a dataset and\slash or model, the authors should describe the steps taken to make their results reproducible or verifiable. 
        \item Depending on the contribution, reproducibility can be accomplished in various ways. For example, if the contribution is a novel architecture, describing the architecture fully might suffice, or if the contribution is a specific model and empirical evaluation, it may be necessary to either make it possible for others to replicate the model with the same dataset, or provide access to the model. In general. releasing code and data is often one good way to accomplish this, but reproducibility can also be provided via detailed instructions for how to replicate the results, access to a hosted model (e.g., in the case of a large language model), releasing of a model checkpoint, or other means that are appropriate to the research performed.
        \item While NeurIPS does not require releasing code, the conference does require all submissions to provide some reasonable avenue for reproducibility, which may depend on the nature of the contribution. For example
        \begin{enumerate}
            \item If the contribution is primarily a new algorithm, the paper should make it clear how to reproduce that algorithm.
            \item If the contribution is primarily a new model architecture, the paper should describe the architecture clearly and fully.
            \item If the contribution is a new model (e.g., a large language model), then there should either be a way to access this model for reproducing the results or a way to reproduce the model (e.g., with an open-source dataset or instructions for how to construct the dataset).
            \item We recognize that reproducibility may be tricky in some cases, in which case authors are welcome to describe the particular way they provide for reproducibility. In the case of closed-source models, it may be that access to the model is limited in some way (e.g., to registered users), but it should be possible for other researchers to have some path to reproducing or verifying the results.
        \end{enumerate}
    \end{itemize}

\item {\bf Open access to data and code}
    \item[] Question: Does the paper provide open access to the data and code, with sufficient instructions to faithfully reproduce the main experimental results, as described in supplemental material?
    \item[] Answer: \answerYes{} 
    \item[] Justification: See supplementary materials.
    \item[] Guidelines:
    \begin{itemize}
        \item The answer \answerNA{} means that paper does not include experiments requiring code.
        \item Please see the NeurIPS code and data submission guidelines (\url{https://neurips.cc/public/guides/CodeSubmissionPolicy}) for more details.
        \item While we encourage the release of code and data, we understand that this might not be possible, so \answerNo{} is an acceptable answer. Papers cannot be rejected simply for not including code, unless this is central to the contribution (e.g., for a new open-source benchmark).
        \item The instructions should contain the exact command and environment needed to run to reproduce the results. See the NeurIPS code and data submission guidelines (\url{https://neurips.cc/public/guides/CodeSubmissionPolicy}) for more details.
        \item The authors should provide instructions on data access and preparation, including how to access the raw data, preprocessed data, intermediate data, and generated data, etc.
        \item The authors should provide scripts to reproduce all experimental results for the new proposed method and baselines. If only a subset of experiments are reproducible, they should state which ones are omitted from the script and why.
        \item At submission time, to preserve anonymity, the authors should release anonymized versions (if applicable).
        \item Providing as much information as possible in supplemental material (appended to the paper) is recommended, but including URLs to data and code is permitted.
    \end{itemize}

\item {\bf Experimental setting/details}
    \item[] Question: Does the paper specify all the training and test details (e.g., data splits, hyperparameters, how they were chosen, type of optimizer) necessary to understand the results?
    \item[] Answer: \answerYes{} 
    \item[] Justification: See Sections \ref{sec:pdp_task}, \ref{sec:node_regression} \ref{sec:nc_exp}, \ref{sec:link_prediction_real}  and Appendix \ref{app:tuning}.
    \item[] Guidelines:
    \begin{itemize}
        \item The answer \answerNA{} means that the paper does not include experiments.
        \item The experimental setting should be presented in the core of the paper to a level of detail that is necessary to appreciate the results and make sense of them.
        \item The full details can be provided either with the code, in appendix, or as supplemental material.
    \end{itemize}

\item {\bf Experiment statistical significance}
    \item[] Question: Does the paper report error bars suitably and correctly defined or other appropriate information about the statistical significance of the experiments?
    \item[] Answer: \answerYes{} 
    \item[] Justification: We report variability estimates for all experimental results.
    \item[] Guidelines:
    \begin{itemize}
        \item The answer \answerNA{} means that the paper does not include experiments.
        \item The authors should answer \answerYes{} if the results are accompanied by error bars, confidence intervals, or statistical significance tests, at least for the experiments that support the main claims of the paper.
        \item The factors of variability that the error bars are capturing should be clearly stated (for example, train/test split, initialization, random drawing of some parameter, or overall run with given experimental conditions).
        \item The method for calculating the error bars should be explained (closed form formula, call to a library function, bootstrap, etc.)
        \item The assumptions made should be given (e.g., Normally distributed errors).
        \item It should be clear whether the error bar is the standard deviation or the standard error of the mean.
        \item It is OK to report 1-sigma error bars, but one should state it. The authors should preferably report a 2-sigma error bar than state that they have a 96\% CI, if the hypothesis of Normality of errors is not verified.
        \item For asymmetric distributions, the authors should be careful not to show in tables or figures symmetric error bars that would yield results that are out of range (e.g., negative error rates).
        \item If error bars are reported in tables or plots, the authors should explain in the text how they were calculated and reference the corresponding figures or tables in the text.
    \end{itemize}

\item {\bf Experiments compute resources}
    \item[] Question: For each experiment, does the paper provide sufficient information on the computer resources (type of compute workers, memory, time of execution) needed to reproduce the experiments?
    \item[] Answer: \answerYes{} 
    \item[] Justification: See Appendix \ref{app:tuning}.
    \item[] Guidelines:
    \begin{itemize}
        \item The answer \answerNA{} means that the paper does not include experiments.
        \item The paper should indicate the type of compute workers CPU or GPU, internal cluster, or cloud provider, including relevant memory and storage.
        \item The paper should provide the amount of compute required for each of the individual experimental runs as well as estimate the total compute. 
        \item The paper should disclose whether the full research project required more compute than the experiments reported in the paper (e.g., preliminary or failed experiments that didn't make it into the paper). 
    \end{itemize}
    
\item {\bf Code of ethics}
    \item[] Question: Does the research conducted in the paper conform, in every respect, with the NeurIPS Code of Ethics \url{https://neurips.cc/public/EthicsGuidelines}?
    \item[] Answer: \answerYes{} 
    \item[] Justification: The research presented in this paper fully complies with the NeurIPS Code of Ethics.
    \item[] Guidelines:
    \begin{itemize}
        \item The answer \answerNA{} means that the authors have not reviewed the NeurIPS Code of Ethics.
        \item If the authors answer \answerNo, they should explain the special circumstances that require a deviation from the Code of Ethics.
        \item The authors should make sure to preserve anonymity (e.g., if there is a special consideration due to laws or regulations in their jurisdiction).
    \end{itemize}

\item {\bf Broader impacts}
    \item[] Question: Does the paper discuss both potential positive societal impacts and negative societal impacts of the work performed?
    \item[] Answer: \answerNo{} 
    \item[] Justification: This paper presents work whose goal is to advance the field of machine learning. There are many potential societal consequences of our work, none of which we feel must be specifically highlighted here.
    \item[] Guidelines:
    \begin{itemize}
        \item The answer \answerNA{} means that there is no societal impact of the work performed.
        \item If the authors answer \answerNA{} or \answerNo, they should explain why their work has no societal impact or why the paper does not address societal impact.
        \item Examples of negative societal impacts include potential malicious or unintended uses (e.g., disinformation, generating fake profiles, surveillance), fairness considerations (e.g., deployment of technologies that could make decisions that unfairly impact specific groups), privacy considerations, and security considerations.
        \item The conference expects that many papers will be foundational research and not tied to particular applications, let alone deployments. However, if there is a direct path to any negative applications, the authors should point it out. For example, it is legitimate to point out that an improvement in the quality of generative models could be used to generate Deepfakes for disinformation. On the other hand, it is not needed to point out that a generic algorithm for optimizing neural networks could enable people to train models that generate Deepfakes faster.
        \item The authors should consider possible harms that could arise when the technology is being used as intended and functioning correctly, harms that could arise when the technology is being used as intended but gives incorrect results, and harms following from (intentional or unintentional) misuse of the technology.
        \item If there are negative societal impacts, the authors could also discuss possible mitigation strategies (e.g., gated release of models, providing defenses in addition to attacks, mechanisms for monitoring misuse, mechanisms to monitor how a system learns from feedback over time, improving the efficiency and accessibility of ML).
    \end{itemize}
    
\item {\bf Safeguards}
    \item[] Question: Does the paper describe safeguards that have been put in place for responsible release of data or models that have a high risk for misuse (e.g., pre-trained language models, image generators, or scraped datasets)?
    \item[] Answer: \answerNA{} 
    \item[] Justification: The paper poses no such risks.
    \item[] Guidelines:
    \begin{itemize}
        \item The answer \answerNA{} means that the paper poses no such risks.
        \item Released models that have a high risk for misuse or dual-use should be released with necessary safeguards to allow for controlled use of the model, for example by requiring that users adhere to usage guidelines or restrictions to access the model or implementing safety filters. 
        \item Datasets that have been scraped from the Internet could pose safety risks. The authors should describe how they avoided releasing unsafe images.
        \item We recognize that providing effective safeguards is challenging, and many papers do not require this, but we encourage authors to take this into account and make a best faith effort.
    \end{itemize}

\item {\bf Licenses for existing assets}
    \item[] Question: Are the creators or original owners of assets (e.g., code, data, models), used in the paper, properly credited and are the license and terms of use explicitly mentioned and properly respected?
    \item[] Answer: \answerYes{} 
    \item[] Justification: All datasets, codebases, and prior models used in this work are properly credited through citations to the original sources.
    \item[] Guidelines:
    \begin{itemize}
        \item The answer \answerNA{} means that the paper does not use existing assets.
        \item The authors should cite the original paper that produced the code package or dataset.
        \item The authors should state which version of the asset is used and, if possible, include a URL.
        \item The name of the license (e.g., CC-BY 4.0) should be included for each asset.
        \item For scraped data from a particular source (e.g., website), the copyright and terms of service of that source should be provided.
        \item If assets are released, the license, copyright information, and terms of use in the package should be provided. For popular datasets, \url{paperswithcode.com/datasets} has curated licenses for some datasets. Their licensing guide can help determine the license of a dataset.
        \item For existing datasets that are re-packaged, both the original license and the license of the derived asset (if it has changed) should be provided.
        \item If this information is not available online, the authors are encouraged to reach out to the asset's creators.
    \end{itemize}

\item {\bf New assets}
    \item[] Question: Are new assets introduced in the paper well documented and is the documentation provided alongside the assets?
    \item[] Answer: \answerYes{} 
    \item[] Justification: We introduce new experimental code and data processing pipelines, which are made available in an anonymized repository to ensure reproducibility.
    \item[] Guidelines:
    \begin{itemize}
        \item The answer \answerNA{} means that the paper does not release new assets.
        \item Researchers should communicate the details of the dataset\slash code\slash model as part of their submissions via structured templates. This includes details about training, license, limitations, etc. 
        \item The paper should discuss whether and how consent was obtained from people whose asset is used.
        \item At submission time, remember to anonymize your assets (if applicable). You can either create an anonymized URL or include an anonymized zip file.
    \end{itemize}

\item {\bf Crowdsourcing and research with human subjects}
    \item[] Question: For crowdsourcing experiments and research with human subjects, does the paper include the full text of instructions given to participants and screenshots, if applicable, as well as details about compensation (if any)? 
    \item[] Answer: \answerNA{} 
    \item[] Justification: The paper does not involve crowdsourcing nor research with human subjects.
    \item[] Guidelines:
    \begin{itemize}
        \item The answer \answerNA{} means that the paper does not involve crowdsourcing nor research with human subjects.
        \item Including this information in the supplemental material is fine, but if the main contribution of the paper involves human subjects, then as much detail as possible should be included in the main paper. 
        \item According to the NeurIPS Code of Ethics, workers involved in data collection, curation, or other labor should be paid at least the minimum wage in the country of the data collector. 
    \end{itemize}

\item {\bf Institutional review board (IRB) approvals or equivalent for research with human subjects}
    \item[] Question: Does the paper describe potential risks incurred by study participants, whether such risks were disclosed to the subjects, and whether Institutional Review Board (IRB) approvals (or an equivalent approval/review based on the requirements of your country or institution) were obtained?
    \item[] Answer: \answerNA{} 
    \item[] Justification: The paper does not involve human subjects, crowdsourcing, or any form of user study.
    \item[] Guidelines:
    \begin{itemize}
        \item The answer \answerNA{} means that the paper does not involve crowdsourcing nor research with human subjects.
        \item Depending on the country in which research is conducted, IRB approval (or equivalent) may be required for any human subjects research. If you obtained IRB approval, you should clearly state this in the paper. 
        \item We recognize that the procedures for this may vary significantly between institutions and locations, and we expect authors to adhere to the NeurIPS Code of Ethics and the guidelines for their institution. 
        \item For initial submissions, do not include any information that would break anonymity (if applicable), such as the institution conducting the review.
    \end{itemize}

\item {\bf Declaration of LLM usage}
    \item[] Question: Does the paper describe the usage of LLMs if it is an important, original, or non-standard component of the core methods in this research? Note that if the LLM is used only for writing, editing, or formatting purposes and does \emph{not} impact the core methodology, scientific rigor, or originality of the research, declaration is not required.
    \item[] Answer: \answerNA{} 
    \item[] Justification: The core method development in this research does not involve LLMs as any important, original, or non-standard components.
    \item[] Guidelines:
    \begin{itemize}
        \item The answer \answerNA{} means that the core method development in this research does not involve LLMs as any important, original, or non-standard components.
        \item Please refer to our LLM policy in the NeurIPS handbook for what should or should not be described.
    \end{itemize}

\end{enumerate}

\end{document}